\documentclass[12pt, a4paper]{article}
\usepackage{amssymb,amsmath,amsthm}
\usepackage[all]{xy}
\usepackage{algpseudocode}
\usepackage{algorithm}
\usepackage{pgfplots}
\pgfplotsset{compat=1.18}
\usepgfplotslibrary{colormaps}
\usepackage{subcaption}
\usepackage{amssymb,amsmath,amsthm}
\usepackage[mathscr]{eucal}%
\usepackage[all]{xy}
\usepackage{lmodern}
\usepackage[T1]{fontenc}
\usepackage[utf8]{inputenc}
\usepackage{tikz-cd}
\usepackage[shortlabels]{enumitem}
\usepackage{relsize}
\usepackage{geometry}
\usepackage{mathtools}
\mathtoolsset{showonlyrefs}
\usepackage{adjustbox}
\usepackage{lscape}
\usepackage{diagbox}
\usepackage{bm}
\usepackage{caption}
\usepackage{float}

\usepackage{stmaryrd}
\usepackage{mathrsfs}  
\usepackage{multirow}
\usepackage[hyperfootnotes=false]{hyperref}
\usetikzlibrary{positioning}
\usepackage{tikz-3dplot}
\usepackage{xifthen}
\usetikzlibrary{quotes,arrows.meta}
\tdplotsetmaincoords{60}{125}
\tdplotsetrotatedcoords{0}{0}{0}
\usepackage{xcolor}
\definecolor{mygreen}{rgb}{0.16,.55,0.0}
\usepackage[nottoc]{tocbibind}

\renewcommand{\div}{\mathrm{div}}

\theoremstyle{plain}
\newtheorem{prop}{Proposition}[section]

\newtheorem{lem}[prop]{Lemma}
\newtheorem{thm}[prop]{Theorem}
\newtheorem{cor}[prop]{Corollary}
\theoremstyle{definition}

\newtheorem{example}[prop]{Example}
\newtheorem{rem}[prop]{Remark}

\theoremstyle{plain}

\theoremstyle{def}

\theoremstyle{plain}

\theoremstyle{def}

\theoremstyle{plain}

\theoremstyle{def}

\def\varddots{\mathinner{\raise7pt\vbox{\kern3pt\hbox{.}}\mkern1mu\smash{\raise4pt\hbox{.}}\mkern1mu\smash{\raise1pt\hbox{.}}}}

\DeclareMathAlphabet{\mathpzc}{OT1}{pzc}{m}{it}

\DeclareMathOperator{\CFM}{CFM}

\DeclareMathOperator{\Id}{Id}

\DeclareMathOperator*{\argmin}{argmin}

\newcommand{\E}{\mathbb{E}}

\newcommand{\NN}{\mathbb{N}}

\renewcommand{\L}{\mathcal{L}}

\newcommand{\N}{\mathcal{N}}
\renewcommand{\u}{\mathfrak{u}}

\newcommand{\dd}{\mathrm{d}}

\newcommand{\PP}{\mathbb{P}}

\newcommand{\T}{\mathcal{T}}

\newcommand{\R}{\mathbb{R}}

\newcommand{\K}{\mathcal{K}}

\newcommand{\B}{\mathcal{B}}

\newcommand{\weakly}{\rightharpoonup}\normalfont

\newcommand{\tT}{\mathrm{T}}

\definecolor{forestgreen}{RGB}{34,139,34}

\renewcommand{\P}{\mathcal{P}}
\theoremstyle{plain}

\title{
Flow Matching: Markov Kernels, \\
Stochastic Processes and Transport Plans
}
\author{Christian Wald\footnotemark[1] 
		\and Gabriele Steidl\footnote{Institute of Mathematics,
	   Technische Universität Berlin,
	   Stra{\ss}e des 17.\ Juni 136, 
	   10623 Berlin, Germany,
	   {\ttfamily\{wald, steidl\}@math.tu-berlin.de},
      \url{http://tu.berlin/imageanalysis}.
	} 
 }

\date{\today }

\begin{document}

\maketitle
\begin{abstract}
Among generative neural models, flow matching techniques stand
out for their simple applicability and good scaling properties.
Here, velocity fields of curves connecting a
simple latent and a target distribution are learned.
Then the corresponding ordinary differential equation can be used to sample from a target distribution, starting in samples from the latent one.
This paper reviews from a mathematical point of view 
different techniques to learn the velocity fields of
absolutely continuous curves in the Wasserstein geometry.
We show how the velocity fields can be characterized and learned via i) transport plans (couplings) between latent and target distributions,
ii) Markov kernels and iii) stochastic processes, where the latter two include the coupling approach, but are in general broader.

Besides this main goal, we show how flow matching can
be used for solving Bayesian inverse problems, where the definition of conditional Wasserstein distances plays a central role.

Finally, we briefly address continuous normalizing flows and score matching techniques, which approach the learning of velocity fields of curves
from other directions.
\end{abstract}
\newpage
\tableofcontents

\section{Introduction}
Generative models with neural network approximations have shown impressive 
results in many applications, in particular in inverse problems and data assimilation. First developments in this direction were
generative adversarial networks by Goodfellow et al. \cite{goodfellow2014} in 2014 and variational autoencoders by Kingma and Welling \cite{KW2013} in 2013. 
In this paper, 
we are interested in recent flow matching techniques
pioneered by Liu et al.  \cite{liu2022rectified,liu2023flow} in 2022/23
 from the point of view of
stochastic processes, by Lipman et al. \cite{lipman2023flow} in 2023
via conditional probability paths and by Albergo et al. \cite{albergobuilding,albergo2023stochastic} in 2022/2023 under the name of stochastic interpolants..
Flow matching is closely related  to
 continuous normalizing flows (CNFs), also called neural ordinary differential equations by Chen et al. \cite{CRBD2018} in 2018, and to 
 score-based diffusion models introduced by  Sohl-Dickstein et al. \cite{pmlr-v37-sohl-dickstein15} and Ermon and Song \cite{SE2019} in 2015 and 2019, respectively.
 The later ones can be considered as an instance of CNFs, but add the  perspective of  stochastic differential equations.
To provide a more complete picture, we will briefly consider these approaches
in  Sections \ref{sec:CNF} and \ref{sec:diffusion}.
Flow matching techniques stand out for their simple applicability and good scaling properties. They can be easily incorporated into other optimization techniques like
plug-and-play approaches \cite{martin2024pnp}
and can be generalized 
for solving Bayesian inverse problems as we will see in Section \ref{sec:fm_inverse_problems}.  Recently, flow matching was generalized towards generator matching in \cite{holderrieth2025generator,JCHWS2025}.

We will approach the topic of flow matching from a mathematical point of view, where we intend to make the paper to a certain degree self-contained. However, for the practical implementation of flow matching, we assume that the reader is familiar
with training neural networks once a tractable loss function is given.

The basic task of generative modeling consists in generating new samples from a probability
distribution $P_{\text{data}}$, where we have only access to  a set of samples, e.g. from the distribution of face images.
This is a quite classical task in stochastics, since it is in general hard to sample from a high-dimensional distribution
even if its density is known.
Few examples, where it is easy to sample from in multiple dimensions, are the Gaussian distribution and their relatives. 
In this paper, we focus on the standard Gaussian distribution as so-called ''latent'' distribution $P_{\text{latent}}$.
Then, the idea is to approximate a ''nice'' curve $t \mapsto \mu_t$, mapping a time $t\in [0,1]$ to a probability measure $\mu_t$,
where 
$$\mu_0 = P_{\text{latent}} \quad \text{ and } \quad \mu_1 = P_{\text{data}}.
$$
Our curves of interest are absolutely continuous curves
in the Wasserstein space, which we introduce in Section \ref{acw_curves}.
Due to the special metric  
of the Wasserstein space, 
such a curve  possesses a 
vector field $v: [0,1] \times \R^d \to \R^d$ such that 
the curve-velocity pair $(\mu_t,v_t)$ fulfills the continuity equation
$$
\partial_t \mu_t +  
\nabla_x \cdot  (\mu_t v_t) = 0.
$$
For the associated vector field $v_t$, under mild additional conditions, there  exists a solution $\phi: [0,1] \times \R^d \to \R^d$
of the ordinary differential equation (ODE)
$$
\partial_t \phi(t,x)=v_t(\phi(t,x)), \quad \phi(0,x)=x,
$$
and the curve $\mu_t$ is just the push-forward of the latent distribution by this solution
$$\mu_t = \phi(t,\cdot)_\sharp \mu_0 \coloneqq \mu_0 \circ \phi^{-1} (t,\cdot).$$
Thus, if the vector field $v_t$ is known, it remains to solve the ODE by some standard solver starting at $t=0$ with a sample of the latent distribution to produce a desired data sample at time $t=1$.
An illustration of the flow of samples from the one-dimensional standard Gaussian distribution
to a Gaussian mixture
is given in Figure \ref{fig:flow_s}.
For a higher-dimensional example, see Figure \ref{fig:cat}. 

\begin{figure}
\begin{center}
    \includegraphics[width = 0.4\textwidth]{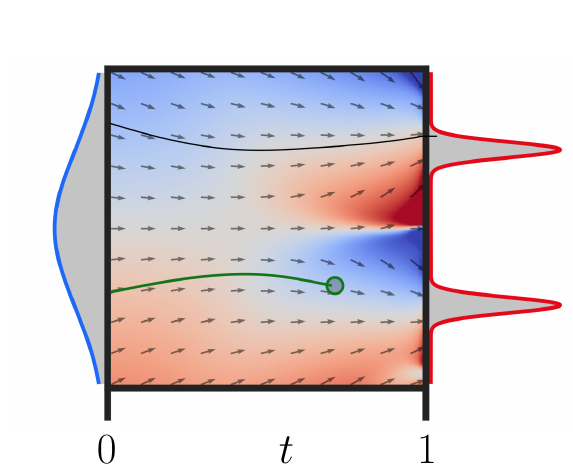}
    \end{center}    
    \caption{Illustration of a curve from the standard Gaussian distribution to a Gaussian mixture in one dimension. The plot shows
    $t \mapsto \phi(t,x^i)$, $i=1,2$ for two different samples $x^i$ (black, green), the vectors $(1, \partial_t \phi)$ 
    and the red-blue color-coded  velocity field $\partial_t \phi$.      Courtesy: Blogpost \cite{sego} } \label{fig:flow_s}
\end{figure}

Before learning the velocity fields via a neural network approximation, we have to  determine appropriate curve-velocity pairs $(\mu_t,v_t)$ fulfilling in particular the continuity equation.
In Sections \ref{sec:planes} - \ref{subsec:conde},
we provide three approaches. 
We will see that the first one follows as a special
case both from the second and third one. Ultimately, we develop the second approach with Markov kernels to  explain  the setting of Lipman et al. in \cite{lipman2023flow}  also for measures without densities.
\begin{itemize}
 \item[1.] \textbf{Curves induced by couplings}: Given a probability measure $\alpha$ on $\R^d \times \R^d$ with marginals $\mu_0$ and $\mu_1$, also called ''coupling'' 
 of $\mu_0$ and $\mu_1$, the curve-velocity pair induced by $\alpha$
 is given by the push-forwards
 $$
\mu_t \coloneqq e_{t,\sharp} \alpha \quad \text{and} \quad v_t \mu_t \coloneqq e_{t,\sharp} [(y-x) \alpha]
$$
with 
$$e_t(x,y)\coloneqq (1-t)x+ty, \quad t \in [0,1].$$ 
Here $v_t \mu_t = e_{t,\sharp} [(y-x) \alpha]$ is a vector valued measures, where you may note that for $\gamma \in \P_2(\R^d),f\in L^1(\R^ d,\gamma)$ the vector valued measure $f\gamma$ is defined by 
$$
\int_{\R^d} \varphi(x)\dd f\gamma(x)=\int_{\R^ d}\varphi(x)f(x)\dd\gamma(x)$$ 
for every bounded measurable function $\varphi:\R^d\to \R$.
Most common couplings are independent couplings
$\alpha = \mu_0 \times \mu_1$ and so-called
''optimal'' transport couplings 
with respect to the Wasserstein distance.
The optimal couplings stand out by the fact that the induced curves are geodesics in the Wasserstein space, which in particular are  minimal length
curves connecting $\mu_0$ and $\mu_1$.
In an alternative formulation, 
we may start with 
random variables
with joint law 
$\alpha = P_{X_0,X_1}$ so that
$\mu_0 = P_{X_0}$ 
and 
$\mu_1 = P_{X_1}$. Then the induced curve by $\alpha$
becomes 
$
\mu_t= P_{X_{t}}
$, 
where
$$X_t\coloneqq e_t(X_0,X_1) = (1-t)X_0+ t X_1.$$
\item[2.] \textbf{Curves via Markov kernels}:
We start with a family of Markov kernels $(\K_t)_{t \in [0,1]}$ with the property that ''conditioned to $y \in \R^d$'',
the curve $t \mapsto \K_t(y,\cdot)$ is  absolutely continuous with associated vector field $v_t^y$.
Then we define a time-dependent sequence of measures on $\R^d \times \R^d$ by glueing the Markov kernel with $\mu_1$ as second marginal, i.e.
\begin{equation} \label{dis}
\alpha_t\coloneqq \K_t(y, \cdot) \times_y \mu_1.
\end{equation}
Denoting by $\nu_t$ its first marginal and using the
disintegration
$\alpha_t = \alpha_t^x \times_x \nu_t$,
we will show that the curve $t \mapsto \nu_t$ is also absolutely continuous with associated velocity field
$$
    v_t (x) \coloneqq \int_{\R^d} v_t^y  \, \dd\alpha_t^x(y).
$$
Note that the curve $\nu_t$ does not necessarily fulfill $\nu_1=\mu_1$ or admit a $\nu_0$ that is easy to sample from. But for reasonable choices it is possible to obtain  that $\nu_0$ is a standard Gaussian distribution and $\nu_1\cong \mu_1$.
An tractable pair $(\nu_t,v_t)$ which cannot be
deduced from a curve-velocity pair induced by a coupling
between $\mu_0$ and $\mu_1$ is given in Example \ref{ex:conv}.

However, for the following special setting, we obtain our induced
curves from couplings $\alpha$ of $\mu_0$ and $\mu_1$:
Let  $\alpha = \alpha^y \times_y \mu_1$ be the disintegration with respect to the second marginal.
Then, $\alpha_t:= (e_t, \pi^2)_\sharp\alpha$ fits into the above setting with
\begin{equation} \label{cond}
\mathcal K_t(y,\cdot) = e_{t,\sharp}( \alpha^y \times \delta_y) \quad \text{and} \quad
v_t^y = \frac {y-x}{1-t}, \quad 
t \in (0,1)
\end{equation}
and $\alpha_t$ has the first marginal $\nu_t = e_{t,\sharp}\alpha$ which is exactly $\mu_t$ from item 1.

 \item[3.] \textbf{Curves via stochastic processes}: Starting with a time differentiable stochastic process $(X_t)_{t \in [0,1]}$, Liu et al.  \cite{liu2022rectified,liu2023flow} 
 determined $(\mu_t,v_t)$ using the conditional expectation
$$
\mu_t \coloneqq P_{X_t} \quad \text{and} \quad v_t(x) 
\coloneqq \E[\partial_t X_t|X_t= x ].
$$
For the special process
$X_t \coloneqq (1-t) X_0 + t X_1$, $t \in [0,1]$, this  results again in a curve-velocity pair
induced by the coupling from the joint distribution 
$\alpha = P_{X_0,X_1}$.
\end{itemize}

Section \ref{sec:fm} deals with the actual flow matching, i.e.,  how we can learn the above vector field  $v_t$ by a neural network $u_t^\theta$.
Clearly, e.g. in the Markov kernel approach, the function
$$
F(\theta) \coloneqq 
\int_0^1 \int_{\R^d \times \R^d} \|u_t^\theta (x)-v_t(x)\|^2 \, \dd \nu_t(x) \dd t
$$
would be a nice loss, but unfortunately we do 
not have access to $v_t$. On the other hand, we know the conditioned velocities $v_t^y$ from \eqref{cond}.
Fortunately, we will see that the vector field $v$
of the curve-velocity pair $(\nu_t,v_t)$
appears as  minimizer of the functional
\begin{align}
 v = \argmin_u 
 \int_0^1 \int_{\R^d \times \R^d} \|u_t(x)-v^y_t(x)\|^2 \, \dd \alpha_t(x,y) \dd t 
\end{align}
or in other words the above loss $F$
coincides up to a constant with the 
loss function
$$
\text{FM}_{\alpha_t}(\theta) \coloneqq 
 \int_0^1 \int_{\R^d \times \R^d} \|u_t^\theta(x)-v^y_t(x)\|^2 \, \dd \alpha_t(x,y) \dd t.
$$
In the original paper \cite{lipman2023flow}, the notation ''conditional flow matching'' (CFM) was used instead of FM which refers to a ''conditional flow path''. We prefer FM, since we will later deal with Bayesian
inverse problems, where the notation of being ''conditional'' is occupied in a different way. 
For the numerical simulation, we  have on the one hand to sample from $\alpha_t=\K_t(y,\cdot) \times_y\mu_1$ and on the other hand  to compute $v_t^y(x)$. Samples $y_i$, $i=1,\ldots,n$ from the data distribution $\mu_1$ are usually available
and it should be easy to sample from the distributions
$\K_t(y_i,\cdot)$, $i=1,\ldots,n$.
In the special case, where $(\nu_t,v_t) = (\mu_t,v_t)$ is induced by a coupling $\alpha$, we have a simple formula for  $v_t^y$ for which the loss function becomes
\begin{equation}\label{eq:fm}\tag{FM} 
\text{FM}_{\alpha}(\theta)= \int_0^1 \int_{\R^d \times \R^d}\|u_t^\theta(e_t(x,y))-(y-x)\|^2
\, \dd \alpha(x,y) \dd t.
\end{equation}
 
Similarly, for the approach via a stochastic process, we get
\begin{align}
	v&=\argmin_{u} \int_0^1 \E \left[\|\partial_t X_t-u_t(X_t)\|^2 \right] \, \dd t.
\end{align}    
Again, for the special stochastic process $X_t=(1-t)X_0+tX_1$ 
leading to a curve-velocity pair 
induced by the plan $\alpha = P_{X_0,X_1}$
this results in \eqref{eq:fm}.
For the independent coupling $\alpha = \mu_0 \times \mu_1$ and the
''optimal'' one, this can be minimized numerically.

Let us start with the necessary notation agreements in the next section.

\section{Preliminaries and Notation} \label{sec:prelim}
Throughout this paper, we equip $\R^d$ with 
the $\sigma$-algebra of Borel sets 
$\mathcal B(\R^d)$, and by ''measurable'' sets/functions we always refer to ''Borel measurable''. Speaking about absolutely continuous measures on $\R^d$, we mean
absolutely continuous with respect to the Lebesgue measure $\mathcal L$
on $\R^d$. By $\mathcal L_A$ we denote the Lebesgue measure restricted to the measurable set $A \subset \R^ d$.
By $e_i$, $i=1,\ldots,d$ we denote the canonical basis elements of $\R^d$.

\subsection{Special Metric Spaces}
Let $C_b(\R^d)$ be the Banach space of bounded, continuous, real-valued functions with norm
\begin{equation} \label{eq:infty-norm}
\|\varphi\|_\infty \coloneqq \sup_{x \in \R^d} |\varphi (x)|,
\end{equation}
$C_0(\R^d)$ its closed subspace of continuous functions vanishing at infinity, and $C_c(\R^d)$ the space of continuous functions with compact support.
Further, $C^l(\R^d)$, $l \ge 1$, denotes the space of $l$ times continuously 
differentiable functions and $C^\infty(\R^d)$ the space of infinity often differentiable functions likewise combined with the above compact support property $C_c^\infty(\R^d)$.
Instead of $\R^d$, functions may be also defined on time intervals $I \subset \R$
or direct products $I \times \R^d$ with the corresponding adaptation of the notation.
We will also consider the vector-valued functions mapping into $\R^m$
and use the notation $C^l(\R^d,\R^m)$ here.

By $\mathcal M(\R^d)$, we denote the Banach space of finite Borel measures on $\R^d$ 
equipped with the \emph{total variation norm}
\begin{equation*}
  \|\mu\|_{\text{TV}} := \sup \Big\{ \sum_{k=1}^\infty
  \lvert\mu(B_k)\rvert : \bigcup_{k=1}^\infty B_k = \R^d \Big\}
\end{equation*}
for any pairwise disjoint Borel sets $B_k$, $k \in \mathbb N$. 
It is the dual space of $C_0(\mathbb R^d)$, in particular, every measure
$\mu \in \mathcal M(\mathbb R^d)$ defines a continuous, linear functional on $C_0(\mathbb R^d)$ via 
$$\varphi \mapsto \langle\mu,\varphi\rangle 
= 
\int_{\mathbb R^d} \varphi(x) \,{\rm d}
\mu( x)
$$ 
and
 \begin{equation*}
  \|\mu\|_{\mathcal M(\mathbb R^d)}
  = \sup_{\lVert\varphi\rVert_{ C_0(\mathbb R^d)} \le 1} |\langle\mu,\varphi\rangle|.
\end{equation*}
A sequence $(\mu_n)_n \subset \mathcal M(\mathbb R^d)$ \emph{converges weak}-* 
to a measure $\mu \in \mathcal M(\mathbb R^d)$,
if
$$
 \lim_{n \to \infty}
\int_{\R^d} \varphi \, \dd \mu_n = \int_{\R^d} \varphi \, \dd \mu \quad \text{for all} \quad \varphi \in  C_0(\mathbb R^d).
$$
We are mainly interested in the subset of probability measures
$$\mathcal P(\R^d) \coloneqq \{ \mu \in  \mathcal M(\R^d): \mu \ge 0, \; \mu(\R^d) = 1 \}.$$
Unfortunately, a sequence of probability measures $(\mu_n)_n$ may not 
converge in the weak-* sense to a probability measure. For example, $\mu_n \coloneqq \delta_n$, $n \in \mathbb N$, converges weak-* towards $\mu \equiv 0$
which is not a probability measure.
Therefore, the concept of narrow convergence\footnote{ 
In measure theory, narrow convergence of measures is also called weak convergence. We do not use the later notation to avoid confusion with the notation of weak convergence in functional analysis.}
was introduced: a sequence $(\mu_n)_n \subset \mathcal M(\mathbb R^d)$ \emph{converges narrowly} 
to a measure $\mu \in \mathcal M(\mathbb R^d)$, written
$\mu_n \weakly \mu$ as $n \to \infty$,  
if
$$
 \lim_{n \to \infty} \int_{\R^d} \varphi \, \dd \mu_n = \int_{\R^d} \varphi \, \dd \mu \quad \text{for all} \quad \varphi \in C_b(\mathbb R^d).
$$
If $\mu_n \in \mathcal P(\mathbb R^d)$
with $\mu_n \weakly \mu$ as $n \to \infty$, then also
$\mu \in \mathcal P(\mathbb R^d)$.
However, note that $\mathcal M (\R^d)$ is not the dual space of $C_b(\mathbb R^d)$,
which is indeed the space $\text{rba} (\R^d)$ of finitely additive set functions. For more information, see \cite[Section 4.4]{PPST2023}.

\begin{rem}[Test functions] 
The interpretation of $\mathcal{M}(\R^d)$ as the dual of $C_0(\R^d)$ implies that $\mu=\nu$ in $\mathcal{M}(\R^d)$ if and only if 
\begin{align}\label{test}
\int_{\R^d} \varphi \, \dd\mu 
=
\int_{ \R^d} \varphi \, \dd\nu \text{ for all } \varphi\in C_0(\R^d).
\end{align}
Since $C_c^\infty(\R^d)$ is dense in $C_0(\R^d)$ with respect to the $\|\cdot\|_{\infty}$ norm, we could also just test against all 
$\varphi \in C_c^{\infty}(\R^d)$. \hfill $\diamond$

\end{rem}

For $\mu\in \P(\R^d)$, let  $L^p(\R^m,\mu)$, $p \in [1,\infty)$, denote the Banach space of (equivalence
classes of) $\mu$-measurable functions $f : \R^d \to \R^m$ with 
$$
\| f\|_{L^p(\R^m,\mu)} \coloneqq
\Big( \int_{\R^d}
\| f\|^p \, \dd \mu\Big)^\frac1p < \infty,
$$
where $\| \cdot \|$ is the Euclidean norm on $\R^m$.
In case of the Lebesgue measure $\mu=\L$, we will skip the $\mu$ in the notation. Furthermore, we write $L^1$ for $L^1(\R,\L)$.

\subsection{Push-Forward Operator}
The \emph{push-forward measure} of 
$\mu \in \mathcal M(\R^d)$ 
by a measurable map 
$T:\R^d \to \R^n$
is defined by
$$
T_\# \mu = \mu \circ T^{-1}
$$
with the preimage $T^{-1}(B)$ of  $B \in \mathcal B(\R^n)$.

Then, a function $f:\R^ n \to \R$ is integrable with respect to $T_\# \mu$ if $f\circ T$ is integrable with respect to $\mu$ and
\begin{equation} \label{push}
    \int_{\R^n} f(y) \, \dd T_\# \mu (y) 
    = \int_{T^{-1} (\R^ n)} f\left(T(x) \right) \, \dd \mu(x).
\end{equation}

If $T: \R^d \to \R^d$ is a $C^1$ diffeomorphism and $\mu$ is absolutely continuous with
density $p_\mu$, then we have by the \emph{transformation theorem}, also known as \emph{change-of-variable formula},
for all measurable, bounded functions $f$ that
 \begin{align} \label{push_density}
\int_{\R^d} f(y) \, \dd T_\# \mu (y) 
    &= \int_{\R^d} f(T(x) ) \, p_\mu(x) \, \dd  x
    = \int_{\R^d} f(y) \, p_\mu\big(T^{-1}(y) \big)| \text{det} \big( \nabla T^{-1} (y) \big) | \, \dd y,
\end{align}
where  $\nabla T$ denotes the Jacobian of $T$. In particular, $T_\# \mu$ has the density
\begin{align} \label{push_density_1}
    p_{T_\# \mu} (y) =  p_\mu \big(T^{-1}(y) \big) | \text{det} \nabla T^{-1} (y)| =  \left(\frac{p_\mu }{|\text{det} \nabla T |}\circ T^{-1}\right)(y).
\end{align}

Recall that the convolution of measures $ \mu, \nu \in \P (\R^d)$ is defined to be the measure $\mu * \nu \in \P(\R^d)$ which fulfills
$$
\int_{\R^d } \varphi(x) \, \dd (\mu *\nu)(x)
=
\int_{\R^d \times \R^d} \varphi(x+y) \, \dd\mu(x) \dd \nu(y)
$$
for all $\varphi \in  C_b(\R^d)$. If $\mu, \nu$ have densities $p_\mu, p_\nu \in L^1$, then $\mu*\nu$
has the density $p_{\mu * \nu}\in L^1$ given by
$$
p_{\mu * \nu} = \int_{\R^d} p_\mu(x) p_{\nu}(\cdot  - x) \, \dd x.
$$
Measures can be characterized via laws of random variables.

\begin{rem}[Random variables]
Let $(\Omega,\Sigma, \mathbb P)$ be a probability space.
Recall that a random variable $X: \Omega \rightarrow \R^d$ is a measurable map $X: \Omega \to \R^d$ and the push-forward measure  
$$P_X \coloneqq X_{\#} \mathbb P = \mathbb P \circ X^{-1}$$ 
is known as law of $X$. Note that different random variables can have the same law. For every measure $\mu\in\P_2(\R^d)$ there exists a random variable with $P_X=\mu$. If $X_0,X_1: \Omega \to \R^d$ are independent random variables, then 
\begin{itemize}
\item[i)]
$(X_0,X_1): \Omega \to \R^d \times \R^d$ has law
$P_{X_0,X_1} = P_{X_0} \times P_{X_1}$, and
\item[ii)]
$Z \coloneqq a_0 X_0 + a_1 X_1 : \Omega \rightarrow \R^d$, $a_0,a_1 \not = 0$ has law
$P_Z = a_{0,\sharp} P_{X_0} *a_{1,\sharp} P_{X_1}$, 
where we abbreviated $a_i\Id$ by $a_i$ in the last formula. 
If $P_{X_i}$ has a density $p_{X_i}$, then
$a_{i,\sharp} P_{X_i}$ has the density $p_{X_i}(a_i^{-1} \cdot)$, 
$i=0,1$,
and $Z$ has the density
$$
p_Z = \int_{\R^d} p_{X_0}(a_0^{-1} x) p_{X_1}\left(a_1^{-1}(\cdot -x) \right) \, \dd x.
\qquad  \qquad\diamond
$$
\end{itemize}
\end{rem}

For comparing measures, in particular $T_\sharp P_{\text{latent}}$ and $P_{\text{data}}$,
divergences between measures can be used. 
A divergence has the main property that it achieves a minimum exactly if both measures coincide. 
One of the most frequently used divergence
is the Kullback-Leibler one.

\begin{rem}[Kullback-Leibler divergence]\label{logl_kl}
The \emph{Kullback-Leibler (KL) divergence} 
$\mathrm{KL}\colon {\mathcal P}(\mathbb R^d) \times {\mathcal P}(\mathbb R^d) \rightarrow \mathbb [0, +\infty]$
of two measures
$\mu,\nu\in {\mathcal P}(\mathbb R^d)$ with existing Radon-Nikodym derivative 
$\frac{\dd \mu}{\dd \nu}$ of $\mu$ with respect to $\nu$ is defined by 
\begin{equation} \label{KLdef}
\mathrm{KL} (\mu,\nu) \coloneqq 
\int_{\mathbb R^d} \log \left(\frac{\dd  \mu}{\dd  \nu} \right) \, \dd  \mu. 
\end{equation}
If the  Radon-Nikodym derivative does not exist, we set $\mathrm{KL} (\mu,\nu) \coloneqq + \infty$.
If $\mu,\nu$ have densities $p_\mu, p_\nu$, where $p_\nu(x) = 0$ implies $p_\mu(x) = 0$, then
\begin{align} \label{KLdef_density}
\mathrm{KL} (\mu,\nu) &\coloneqq 
\int_{\mathbb R^d} \log \left(\frac{p_\mu}{p_\nu} \right) \, p_\mu \, \dd  x\\
&=
\E_{x \sim p_\mu}[\log p_\mu] - \E_{x \sim p_\mu}[\log p_\nu].
\end{align}
The Kullback-Leibler divergence is a so-called Bregman distance related to the Shannon entropy. It is not a distance,
since it is neither symmetric nor fulfills a triangular inequality.
However,  it holds 
$\mathrm{KL} (\mu,\nu) \ge 0$ for all $\mu,\nu\in {\mathcal P}(\mathbb R^d)$ with equality
if and only if the measures $\nu = \mu$ coincide.

If the latent distribution $P_{\text{latent}}$ has a density $p_{\text{latent}}$, and $T:\R^d\to \R^d$ is a $C^1$ diffeomorphism, then by \eqref{push_density_1} also $T_\sharp P_{\text{latent}}$ has a density which we denote by $T_\sharp p_{\text{latent}}$. Thus, if the data distribution admits a density $p_{\text{data}}$, then we get
\begin{align}
\mathrm{KL}(P_{\text{data}}, T_{\sharp} P_{\text{latent}})
=
\underbrace{\E_{x\sim p_{\text{data}}}[\log p_{\text{data}}]}_{\text{constant}}-\E_{x\sim p_{\text{data}}}[\log T_\sharp p_{\text{latent}}].
\end{align}
\hfill $\diamond$
\end{rem}

\subsection{Disintegration and Markov Kernels}
In this paper, we consider projections $\pi^i: \R^d \times \R^d \to \R^d$, $i=1,2$, defined by
$$\pi^1(x,y) \coloneqq x, \quad \pi^2(x,y) \coloneqq y. $$
Then, for a measure $\alpha\in\P(\R^d \times \R^d)$,
we have that
$\pi^1_\sharp \alpha$ and $\pi^2_\sharp \alpha$
are the left and right marginals of $\alpha$, respectively.

For a measure $\alpha \in\P(\R^d \times \R^d)$ with marginal $\pi^1_\sharp \alpha=\mu$, there exists a 
$\mu$-a.e. uniquely defined
family of probability measures $\{\alpha^{x}\}_{x}$,  called 
\emph{disintegration of $\alpha$ with respect to  $\pi^1$}, such that 
the map $x \mapsto \alpha^{x}(B)$ is measurable for every  $B\in \B(\R^d)$, and 
$$
\alpha = \alpha^{x} \times_x \mu
$$
meaning that
\[
\int_{\R^d \times \R^d} f(x,y) \, \dd\alpha(x,y)
=
\int_{\R^d} \int_{\R^d}f(x,y) \, \dd \alpha^{x}(y) \dd \mu(x)
\]
for every measurable, bounded function $f:\R^d\times \R^d \to \mathbb R$.
For an illustration, see Figure \ref{fig:disint}.
Similarly, we define for a measure $\alpha \in\P(\R^d \times \R^d)$ with marginal $\pi^2_\sharp \alpha=\nu$
the \emph{disintegration of $\alpha$ with respect to  $\pi^2$} as
$$
\alpha = \alpha^{y} \times_y \nu.
$$

\begin{figure}
    \centering
    \includegraphics[width=\linewidth]{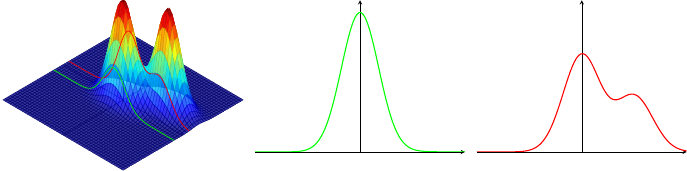}
    \caption{Disintegration of the measure $\alpha \in \mathcal P(\R \times \R)$ (left).
    Measures $\alpha^{-0.3} \in \mathcal P(\R)$ (middle, green) and $\alpha^{0.2}\in \mathcal P(\R)$ (right, red).} \label{fig:disint}
\end{figure}

The notation of disintegration is directly related to  Markov kernels.
A \emph{Markov kernel} is a map $\K:\R ^d \times \B(\R^d)\to \R$ such that
    \begin{itemize}
    \item[i)] $\K(x,\cdot)$ is a probability measure on $\R^d$ for every $x \in \R^d$, and  
    \item[ii)] $\K(\cdot,B)$ is a Borel measurable map for every $B\in\B(\R^d)$.
    \end{itemize}
Hence, given a probability measure $\mu \in \P(\R^d)$, 
we can define a new measure $\alpha \coloneqq \alpha^x \times_x \mu\in \P(\R^d \times \R^d)$ by 
\[
\int_{\R^d\times \R^d}f(x,y) \, \dd\alpha (x,y) \coloneqq\int_{\R^d}\int_{\R^d}f(x,y) \, \dd\K(x,\cdot)(y)\dd\mu( x)
\]
for all measurable, bounded functions $f$.
Identifying $\alpha^{x}(B)$ with $\K(x,B)$, 
we see that conversely,
$\{\alpha^{x}\}_{x}$ is the disintegration of $\alpha$ with respect to $\pi^1$.

\begin{rem}[Random variables]
 Let $X_0,X_1:\Omega \to \R^d$ be random variables with joint distribution $ P_{X_0,X_1}$.
Then the conditional distribution $\{  P_{X_1|X_0=x} \}_x$ provides the disintegration 
of $P_{X_0,X_1}$, i.e.
$$
\int_{\R^d \times \R^d} f(x,y) \, \dd P_{X_0,X_1}(x,y) 
= \int_{\R^d } \int_{\R^d}  f(x,y) \, \dd  P_{X_1|X_0=x} (y) \dd P_{X_0}(x). 
$$
In other words, $\mathcal K = P_{X_1|X_0}$ is a Markov kernel.
If  $P_{X_0,X_1}$ admits a density $p_{X_0,X_1}$ and $P_{X_0}$ a density $p_{X_0}>0$, then $P_{X_1|X_0=x}$ has the density 
$p_{X_1|X_0 = x} = p_{X_0,X_1}(x, \cdot)/p_{X_0}(x)$.
\hfill $\diamond$
\end{rem}

\begin{example}\label{ex:disrete}
 For $\mathcal X \coloneqq \{x_i \in \R^d: i=1,\ldots,n\}$ and
 $\mathcal Y \coloneqq \{y_j \in \R^d, j=1,\ldots,m\}$,
  the discrete probability measure
 $$
 \alpha \coloneqq \sum_{i,j = 1}^{n,m} \alpha_{i,j} \delta_{(x_i,y_j)}, 
 $$
 has first and second marginals 
 \begin{align*}
 \mu &= \pi^1_\sharp \alpha= \sum_{i=1}^n 
\Big( \sum_{j=1}^m \alpha_{i,j} \Big) \delta_{x_i}
= \sum_{i=1}^n \mu_i \delta_{x_i},\\
\nu &= \pi^2_\sharp \alpha= \sum_{j=1}^m 
\Big( \sum_{i=1}^n \alpha_{i,j} \Big) \delta_{y_j}
= \sum_{j=1}^m \nu_j \delta_{y_j},
\end{align*}
see Figure \ref{fig_1}.
The disintegration of $\alpha$ with respect to $\pi^1$ is given by the family of measures
\begin{align} \label{markov_1}
\alpha^{x_i}  &= \mathcal K(x_i, \cdot) = \sum_{j=1}^m \frac{\alpha_{i,j}}{\mu_i} \delta_{y_j} = \sum_{j=1}^m k_{i,j} \delta_{y_j}, 
\\
\alpha^{x_i} (B) &= \mathcal K(x_i, B) = \sum_{y_j \in B} k_{i,j}
\end{align}
for all $B \in \mathcal B(\R^d)$.
The Markov kernel fulfills 
$$
\sum_{j=1}^m k_{i,j} = 1, \quad \sum_{i=1}^n k_{i,j} \mu_i = \nu_j.
$$
In other words, with the vector 
$\boldsymbol{1}_m$ 
consisting of $m$ entries one, and
$\boldsymbol{\mu} \coloneqq (\mu_i)_{i=1}^n$, 
$\boldsymbol{\nu} \coloneqq (\nu_j)_{j=1}^m$, 
as well as
$\boldsymbol{K} \coloneqq \left(k_{i,j} \right)_{i,j=1}^{n,m}$,
we have
$$
\boldsymbol K \boldsymbol{1}_n = \boldsymbol{1}_m, \quad 
\boldsymbol K^\tT \boldsymbol{\mu} = \boldsymbol{\nu}
$$
meaning that $\boldsymbol K^\tT$ is a stochastic matrix and a ''transfer matrix'' from $\mu$ to $\nu$. \hfill $\diamond$
\end{example}

\begin{figure}[ht] 
\begin{center}
\begin{tabular}{c||c|c|c|c|c}
      &$\nu_1$      & \ldots & {\color{red} $\nu_j$} & \ldots & $\nu_m$\\
\hline
$\mu_1$ &$\alpha_{1,1}$& \ldots & {\color{red} $\alpha_{1,j}$}&\ldots & $\alpha_{1,m}$\\
\vdots&            &        & {\color{red}\vdots }     &       & \vdots \\
{\color{blue} $\mu_i$} &{\color{blue} $\alpha_{i,1}$}& {\color{blue} \ldots} &  $\alpha_{i,j}$&{\color{blue} \ldots} & {\color{blue} $\alpha_{i,m}$}\\
\vdots&             &        & {\color{red}\vdots  }    &       & \vdots \\
$\mu_n$ &$\alpha_{n,1}$& \ldots & {\color{red} $\alpha_{n,j}$} &\ldots & $\alpha_{n,m}$
\end{tabular}
\hspace{0.5cm}
\begin{tabular}{c||c|c|c|c|c}
      & $\nu_1/\mu$&    &$\nu_j/\mu$  &  &$\nu_m/\mu$ 
      \\
\hline
$1$ &$\frac{\alpha_{1,1}}{\mu_1}$& \ldots & {\color{red} $\frac{\alpha_{1,j}}{\mu_1}$}&\ldots & $\frac{\alpha_{1,m}}{\mu_1}$\\
\vdots&            &        & {\color{red}\vdots }     &       & \vdots \\
{\color{blue} 1} &{\color{blue} $\frac{\alpha_{i,1}}{\mu_i}$}& {\color{blue} \ldots} &  $\frac{\alpha_{i,j}}{\mu_i}$&{\color{blue} \ldots} & {\color{blue} $\frac{\alpha_{i,m}}{\mu_i}$}\\
\vdots&             &        & {\color{red}\vdots  }    &       & \vdots \\
1 &$\frac{\alpha_{n,1}}{\mu_n}$& \ldots & {\color{red} $\frac{\alpha_{n,j}}{\mu_n}$} &\ldots & $\frac{\alpha_{n,m}}{\mu_n}$
\end{tabular}
\end{center}
\caption{Plan/Coupling of two discrete measures $\mu$ and $\nu$ (left) and 
Markov kernel/disintegration (right) with row and column sums. }\label{fig_1}
\end{figure}

\subsection{Couplings and Wasserstein Distance}
Let 
$$\P_2(\R^d)\coloneqq \{\mu\in\P(\R^d):\int_{\R^d} \|x\|^2 \, \dd\mu<\infty\}$$ 
be the probability measures with finite second moments.
For $\mu,\nu\in \P_2(\R^d)$, we define the set of \emph{plans} or \emph{couplings} with marginals $\mu$ and $\nu$ by 
    \begin{align}
    \Gamma(\mu,\nu)\coloneqq\left\{\alpha\in\P(\R^d \times \R^d):
		\pi^1_\sharp \alpha=\mu, \, \pi^2_\sharp \alpha=\nu\right\}.
    \end{align}
The set $\Gamma(\mu,\nu)$ is a compact subset of $\P_2(\R^d \times \R^d)$ with respect to the narrow topology.
The \emph{Wasserstein(-$2$) distance} is defined by
\begin{align} \label{wasserstein}
W_2(\mu,\nu)\coloneqq \min_{\alpha\in\Gamma(\mu,\nu)}\|x-y\|_{L^2(\R^d \times \R^d,\alpha)}
= 
\min_{\alpha\in\Gamma(\mu,\nu)}\left(\int_{\R^d \times \R^d}\|x-y\|^2 \, \dd\alpha(x,y)\right)^{\frac{1}{2}}.
\end{align}
Indeed, the above minimum is attained, see, e.g. \cite[Theorem 1.7]{santambrogio2015optimal}, and 
the \emph{set of optimal plans} is denoted by $\Gamma_o(\mu,\nu)$.
The space $(\P_2(\R^d),W_2)$ is a complete metric space. Speaking about $\P_2(\R^d)$, we will always equip it with the $W_2$ metric.
Then we have for a sequence $(\mu_n)_n$ of measures in $\P_2(\R^d)$  that 
$$\mu_n \to \mu \quad \text{if and only if} \quad
\mu_n \weakly \mu \; \text{ and }
\lim_{n\to \infty} \int_{\R^d} \|x\|^2 \, \dd \mu_n
= 
\int_{\R^d} \|x\|^2 \, \dd \mu.
$$
In general, the minimizer in \eqref{wasserstein} is not unique.
However, if $\mu$ is absolutely continuous, then uniqueness is ensured by the following theorem.

\begin{thm}[Brenier]\label{thm:brenier}
   Let $\mu \in \mathcal P_2(\R^d)$ be absolutely continuous and $\nu \in \mathcal P_2(\R^d)$. 
	Then there is a unique plan $\alpha \in \Gamma_o(\mu, \nu)$,
	 which is  induced by a unique
	measurable optimal transport map, also called Monge map,  $T\colon \R^d \to \R^d$, i.e.,
	\begin{equation} \label{eq:T}
	 \alpha = ({\rm{Id}}, T)_\# \mu
	\end{equation}
	and
  \begin{equation} \label{eq:monge}
	 W_2^2(\mu, \nu) =
 \min_{T\colon  \R^d \to  \R^d}
  \int_{\R^d}  \|x - T(x)\|_2^2  \, \dd \mu(x) 
  \quad \text{subject \; to} \quad T_\# \mu = \nu.
  \end{equation}
	Further, $T = \nabla \psi$, where $\psi\colon \R^d \to (-\infty,+\infty]$
	is convex, lower semi-con\-tin\-u\-ous (lsc) and $\mu$-a.e.\ differentiable. 
	Conversely, if $\psi$ is convex, lsc and $\mu$-a.e.\ differentiable with $\nabla \psi \in L_2(\mu,\R^d)$,
	then $T \coloneqq \nabla \psi$ is an optimal map from $\mu$ to $\nu \coloneqq   T_\# \mu \in \mathcal P_2(\R^d)$. 
\end{thm}

The transport map between Gaussian distributions can be given analytically.

\begin{example}
For two Gaussians
$\mu = \mathcal N(m_\mu, \Sigma_\mu)$
and
$\nu = \mathcal N(m_\nu, \Sigma_\nu)$,
where $\mathcal N(m, \Sigma)$ has the density
$$
p(x) \coloneqq (2\pi)^{-\frac{d}{2}} \,  (\text{det} \, \Sigma )^{-\frac12} \, \text{e}^{- \frac12 (x-m)^\tT \Sigma^{-1} (x-m)},
$$
the transport map is given by 
\begin{equation} \label{ex:gaussian}
T(x) = m_\nu + A (x-m_\mu), \quad 
A \coloneqq \Sigma_\mu^{-\frac12} \left(  \Sigma_\mu^\frac12 \Sigma_\nu \Sigma_\mu^\frac12\right)^{\frac12} \Sigma_\mu^{-\frac12},
\end{equation}
see e.g. \cite[Remark 2.31]{peyre2019computational}.
\hfill $\diamond$
\end{example}

\section{Absolutely Continuous Curves in $(\P_2(\R^d), W_2)$
}\label{acw_curves}
In the rest of this paper, let $I \coloneqq [a,b]$, $a < b$.
The main player in this paper are curves
$$\mu_t: I \to \mathcal P_2(\R^d), \quad t \mapsto \mu_t.$$
Unfortunately, there is an ambiguity here which is usual in the literature, namely that
the curve itself as well as the value of the curve at time $t \in I$ is denoted by $\mu_t$.
Usually it becomes clear from the context what is meant. 
We are only interested in \emph{narrowly continuous curves}, meaning that for 
$t \to t'$ we have $\mu_t \weakly \mu_{t'}$.
Indeed, to make the definition of integrals 
$$\int_I\int_{\R^d}f(t,x) \, \dd\mu_t\dd t$$ 
meaningful for any measurable, bounded function $f$, we have to consider $\mu_t:I\times \B(\R^d)\to \R$ as a Markov kernel, which is possible by the following theorem whose proof is given in Appendix \ref{app_a}.

\begin{thm} \label{m-markov}
Let $\mu_t: I\to \mathcal P_2(\R^d)$ be a narrowly continuous curve. Then, for every Borel set $B\subseteq \R^d$, we have that $t\mapsto \mu_t(B)$ is measurable, i.e., $\mu_t:I\times\B(\R^d)\to \R$ is a Markov kernel.
\end{thm}

Recall that in a complete metric space $(X,\text{d})$,
a curve $\gamma: I \to X$  is called \emph{absolutely continuous}, if there exists as 
$m \in L^1( I)$ such that 
\begin{equation} \label{eq:abs_cont}
\text{d}(\gamma(s),\gamma(t)) \leq \int_s^t m(r) \, \dd r \quad \text{for all} \quad s,t \in I, \; s \le t.
\end{equation}
Here $L^1([a,b])$ denotes the space of (equivalence classes of) absolutely integrable real-valued function defined on an interval $I$.
 Then the \emph{metric derivative} of $\gamma$   defined by
$$
|\gamma'(t)| \coloneqq \lim_{h \to 0} \frac{\text{d}\left(\gamma(t),\gamma(t+h) \right)}{|h|}
$$
exists a.e. and is the smallest function $m$ in 
\eqref{eq:abs_cont}.  The space of absolutely continuous curves 
is denoted by $\text{AC}^1(X,\dd)$. If we replace in the above definition $m \in L^p(I)$,
$1 \le p < \infty$, we obtain the
spaces $\text{AC}^p(X,\dd)$.

For our special space $\mathcal P_2(\R^d)$ with $\text{d} = W_2$, 
absolute continuity of a curve can be characterized by a continuity equation.
To this end, let $\mu_t: I \to \P_2(\R^d)$ be a narrowly continuous curve
and  
$v:I\times \R^d \to \R^d$ be a measurable vector field. 
For $v(t,\cdot): \R^d\to \R^d$ 
we alternatively write 
$v_t$. 
We say that the curve-velocity pair $(\mu_t,v_t)$ satisfies the \emph{continuity equation}
\begin{align}\label{eq:ce}\tag{CE}
\partial_t \mu_t + 
\nabla_x \cdot (\mu_t v_t)=0
\end{align}
in the sense of distributions, if 
\begin{align}\label{eq:cont_dist}
\int_I\int_{\R^d}\partial_t\varphi +\langle \nabla_x \varphi,v_t\rangle \, \dd\mu_t \, \dd t=0
\end{align}
for all $\varphi\in C_c^{\infty}((a,b) \times \R^d)$. Here $\nabla_x$ is the gradient with respect to $x$ for $\varphi=\varphi(t,x)$.

\begin{rem}
The  term "in the sense of distributions" for the (CE) has the following meaning: Consider the measure $\mu=\mu_t\times_t \dd t$. 
Then $\partial_t\mu$ is 
the distribution 
$$
\partial_t\mu(\varphi) 
= -\mu(\partial_t \varphi)
= -\int_{I} \int_{\R^d} \partial_t \varphi \, \dd\mu, \quad \varphi\in C_c^\infty((a,b)\times \R^d).
$$
Further $\mu v(\varphi)= \mu(v\varphi)$ and 
\begin{align}
\div(\mu v)(\varphi)
&=\sum_{i=1}^d\partial_{i}(\mu v_i)(\varphi)
=-\sum_{i=1}^d(\mu v_i)(\partial_i\varphi)
=-\sum_{i=1}^d\mu(v_i\partial_i \varphi)\\
&=-\mu\left(\langle \nabla_x\varphi,v\rangle\right)
=
- \int_{I} \int_{\R^d} \langle \nabla_x\varphi,v\rangle \, \dd\mu.
\end{align}
Thus, the continuity equation can be seen as an equation of distributions.
\hfill $\diamond$
\end{rem}

Now we have the following fundamental theorem. 

\begin{thm}\cite[Theorem 8.3.1]{AGS2008}\label{thm:abscont_ce}
Let $\mu_t: I\to \P_2(\R^d)$ be a narrowly continuous curve. Then $\mu_t$ is absolutely continuous
if and only if there exists a Borel measurable vector field $v:I\times \R^d \to \R^d$ such that the following two conditions are fulfilled:
\begin{itemize}
    \item[\rm{i)}]  $\|v_t\|_{L^2(\R^d,\mu_t)}\in L^1(I)$, 
\item[\rm{ii)}]  $(\mu_t,v_t)$ fulfills \rm{\eqref{eq:ce}}.
\end{itemize}
In this case, the metric derivative $|\mu_t'|$ exists and we have 
$|\mu'_t|\leq \|v_t\|_{L^2(\R^d, \mu_t)}$ for a.e. $t\in I$. 
\end{thm}

We call a  measurable vector field $v:I \times \R^d \to \R ^d$ 
\emph{associated to an absolutely continuous curve} $\mu_t: I\to \P_2(\R^d)$, if $(\mu_t,v_t)$ fulfill \eqref{eq:ce} and
$\|v_t\|_{L^2(\R^d,\mu_t)}\in L^1(I)$.

Indeed, an absolutely continuous curve may admit different associated velocity fields. For an illustration, see Example \ref{example:not_tangential} and Figure \ref{fig:not_tangential}.
By the following remark, a unique associated vector field is characterized
by having minimal $L^2(\R^d,\mu_t)$-norm for a.e. $t \in I$,  or equivalently, by being in the tangent space of $\mu_t$.

\begin{rem}[Minimal velocity fields of absolutely continuous curves]
\label{rem:v_min}

Let $v_t, \tilde v_t$ be associated vector fields of an absolutely continuous curve $\mu_t$ with minimal norm $\|v_t \|_{L^2(\R^d,\mu_t)} = \|\tilde v_t \|_{L^2(\R^d,\mu_t)} \eqqcolon z_t$ for a.e. $t \in I$ among all associated vector fields of $\mu_t$. By the linear structure of \eqref{eq:ce}, also $\frac{v_t + \tilde v_t}{2}$ is associated to $\mu_t$. Assume that $v_t \ne \tilde v_t$. By the strict convexity of $L^2(\R^d,\mu_t)$, it follows the contradiction $\|\frac{v_t + \tilde v_t}{2}\|_{L^2(\R^d,\mu_t)} < z_t$. Hence, we have the uniqueness $v_t = \tilde v_t$.

For an equivalent description, consider the \emph{regular tangent space} $\mathcal T_{\mu} = T_{\mu}\P_2(\R^d)$  at $\mu \in \P_2(\R^d)$ which is defined by
\begin{align} \label{tan_reg}
     \mathcal T_{\mu}
    &\coloneqq 
      \overline{
      \left\{ 
      \nabla \phi: \phi \in C^\infty_{\mathrm c}(\R^d) 
      \right\}}^{L^2(\R^d, \mu)} \\
    &=
      \overline{
      \left\{ \lambda (T- \Id): (\Id ,T)_{\#} \mu \in \Gamma_o(\mu , T_{\#} \mu), \; \lambda >0
      \right\} }^{L^2(\R^d,\mu)},
\end{align}
see \cite[§~8]{AGS2008}. The second description can be interpreted as locally moving mass in an ''optimal way''.
Note that  $\mathcal  T_{\mu}$ is an
	infinite-dimensional subspace of $L^2(\R^d,\mu)$
	if $\mu$ is absolutely continuous, and it is just $\R^d$ if $\mu = \delta_x$, $x \in \R^d$.
There is another description of $\mathcal T_\mu$, namely, for a vector field $v\in L^2(\R^d,\mu)$, we have
    \begin{align} \label{tang}
    &v\in \mathcal T_\mu\mathcal  \; \Longleftrightarrow \;
        \int_{\R^d} \langle w,v\rangle \, \dd\mu=0 \text{ for all } w\in L^2(\R^d,\mu)\text{ with }\nabla\cdot(w \mu)=0,    \\
    &\; \;  \Longleftrightarrow \|v+w\|_{L^2(\R^d,\mu)}\geq\|v\|_{L^2(\R^d,\mu)} \text{ for all } w\in L^2(\R^d,\mu)\text{ with }\nabla\cdot(w \mu)=0,
    \end{align}
    where the last equation is meant again in the distributional sense $\int_{\R^d} \langle w,\nabla \varphi \rangle \, \dd\mu=0$ for all $\varphi\in C_c^\infty(\R^d)$, see \cite[Lemma 8.4.2]{AGS2008}.

Now let $v_t \in \mathcal T_{\mu_t}$  be a vector field associated to $\mu_t$. For any other vector field $\tilde v_t$ associated to $\mu_t$, it holds $\nabla\cdot((v_t - \tilde v_t) \mu)=0$, and hence by \eqref{tang}, $\|\tilde v_t\|_{L^2(\R^d,\mu)}\geq\|v_t\|_{L^2(\R^d,\mu)}$, meaning that $v_t$ has minimal $L^2(\R^d,\mu_t)$-norm. On the other hand, let $v_t$ be associated to $\mu_t$ having minimal norm. By \cite[Theorem 8.3.1]{AGS2008}, there \emph{exists} a vector field $\tilde v_t \in \mathcal T_{\mu_t}$ being associated to $\mu_t$ and having minimal norm. By the uniqueness shown above, it follows $v_t = \tilde v_t \in \mathcal T_{\mu_t}$.\\ Therefore, we have proven that the associated velocity field $v_t$ of an absolutely continuous curve $\mu_t$ is unique if we require that it has minimal $L^2(\R^d,\mu_t)$-norm, or equivalently, lies in the tangent space $\mathcal T_{\mu_t}$.
\hfill $\diamond$
\end{rem} 

The following theorem connects absolutely continuous curves with flow ODEs.

\begin{thm}\cite[Theorem 8.1.8]{AGS2008}\label{prop:abs_flow}
Let $\mu_t: I \to \P_2(\R^d)$ be an absolutely continuous curve with associated vector field $v_t$ such that for every compact Borel set $B\subset \R^d$ it holds 
\begin{align}
\int_I \sup_{x \in B} \|v_t(x)\|+\mathrm{Lip}(v_t,B) \, \dd t<\infty.
\end{align}
Then there exists a solution $\phi:I\times\R^d\to \R^d$ of the ODE
\begin{align}\label{eq:flow_ode} 
\partial_t \phi(t,x)=v_t(\phi(t,x)), \quad \phi(0,x)=x
\end{align}
and $\phi(t,\cdot)_\sharp\mu_0=\mu_t$.
\end{thm}

 Theorem \ref{prop:abs_flow} indicates why absolutely continuous curves can be used as a tool for sampling. A common practice considers such curves starting in 
$\mu_0\coloneqq\mathcal{N}(0,1)$ towards a target measure $\mu_1 = P_{\text{data}}$ and try to approximate the vector field $v_t$ by a neural network $v_t^\theta$ for trainable parameters $\theta$. To sample from $\mu_1$, we then can just sample from $\mu_0=\mathcal{N}(0,1)$ and solve the ODE \eqref{eq:flow_ode} for these samples.

\begin{rem} \label{rem:conversely}
Also the other direction 
in Theorem \ref{prop:abs_flow}
is true under some conditions. To see this, assume that $\phi$ is measurable and satisfies \eqref{eq:flow_ode} for some locally bounded vector field $v_t$. For $\varphi\in C^\infty_c((a,b) \times \R^d)$, we can compute
\begin{align}
0&=\varphi(b,\phi(b,x))-\varphi(a,\phi(a,x))=\int_{\R^d}\varphi(b,\phi(b,x))-\varphi(a,\phi(a,x))\dd\mu_0\\
&=
\int_{\R^d}\int_a^b\frac{\dd}{\dd t}(\varphi(t,\phi(t,x))) \, \dd t\dd \mu_0 \\
&= 
\int_{\R^d}\int_a^b \langle \nabla_x\varphi \big(t,\phi(t,x) \big),v_t \big(t, \phi(t,x) \big)\rangle + (\partial_t \varphi) \big( t,\phi(t,x) \big) \, \dd t\dd\mu_0\\
&=
\int_a^b\int_{\R^d} \langle \nabla_x\varphi,v_t\rangle + \partial_t \varphi \, \dd [\phi(t,\cdot)_\sharp\mu_0]\dd t
\end{align}
and thus $\mu_t\coloneqq \phi(t,\cdot)_\sharp \mu_0$ satisfies the continuity equation (in a weak sense).
\hfill $\diamond$
\end{rem}

In the following Sections \ref{sec:planes} - \ref{subsec:conde}
we show different methods for the construction of 
curve-velocity pairs fulfilling the conditions of Theorem \ref{thm:abscont_ce}, i.e. providing absolutely continuous curves in the Wasserstein geometry.
Having \eqref{eq:flow_ode} in mind,
this will lead to a  sampling procedure.

\section{Curves Induced by Couplings} \label{sec:planes}
We start with curves induced by couplings $\alpha \in \Gamma(\mu_0,\mu_1)$. There are three important kinds of couplings:
\begin{itemize}
\item optimal couplings $\alpha\in \Gamma_o(\mu_0,\mu_1)$, since they correspond to geodesics in $\P_2(\R^d)$.
\item couplings arising from maps, since their induced curves and vector fields are particularly simple, as we will see later.
\item product couplings $\mu_0\times\mu_1$, since we can easily sample from them by sampling from $\mu_0$ and $\mu_1$ and no extra computation is needed.
\end{itemize}
Using couplings it is also easy to see, that the space $\P_2(\R^d)$ is \emph{path connected}, i.e., any two measures  can be connected by a  continuous curve (with respect to $W_2$). Even more, 
$\P_2(\R^d)$ is a \emph{geodesic space}, meaning that any two measures can be connected by a geodesic. Recall that a curve $\mu_t: [0,1] \to \mathcal P_2(\R^d)$ is called a (constant speed) \emph{geodesic} if
$$
W_2(\mu_s,\mu_t) = (t-s) W_2(\mu_0,\mu_1) \quad \text{for all} \quad 0\le s\le t \le 1.
$$
Let $e_t:\R^d \times \R^d\to \R^d$ be defined by 
$$e_t(x,y)\coloneqq (1-t)x+ty, \quad t \in [0,1].$$
Let $\mu_0, \mu_1 \in \mathcal P_2(\R^d)$ and
$\alpha\in \Gamma(\mu_0,\mu_1)$.
We call a 
\emph{curve $\mu_t: [0,1] \to \mathcal P_2(\R^d)$ 
induced by the coupling} or \emph{plan} $\alpha$, if
\begin{equation} \label{eq:ind_c}
\mu_t \coloneqq e_{t,\sharp}\alpha,
\end{equation}
i.e., by definition of the push-forward measure,
\begin{equation} \label{push_1}
\int_{\R^d}  f \, \dd \mu_t 
= 
\int_{\R^d \times \R^d}  f\big(  e_{t} (x,y) \big) \, \dd \alpha
=
\int_{\R^d \times \R^d}  f\left( (1-t) x + t y \right)  \, \dd \alpha
\end{equation}
for every measurable, bounded function $f: \R^d \to \R$. 

\begin{rem}[Couplings induced by maps]
For couplings induced by maps, i.e., $\alpha = (\Id,T)_\sharp \mu_0$,
the induced curves admit a simple form, namely
\begin{equation} \label{curve_from_map}
\mu_t = (T_t)_\sharp \mu_0 
\quad \text{with} \quad
T_t(x) = (1-t) x + t T(x).
\end{equation}
In general, also optimal plans $\alpha \in \Gamma_o(\mu_0,\mu_1)$
may not  be induced by a  Monge map. 
However, if $\alpha$ is optimal and 
$\mu_t=e_{t,\sharp}\alpha$, then, 
for $s \in (0,1)$,
there exists an optimal map $T^s$ such that $\Gamma_o(\mu_s,\mu_1)=\{(\Id, T^s)_\sharp\mu_t\}$, see \cite[Lemma 7.2.1]{AGS2008}.
Moreover, by \cite[Theorem 7.2.2]{AGS2008}, the curve induced by $T^s$ coincides, up to time parametrizations, with $\mu_t$ for $t\in[s,1]$, $s >0$. Figure \ref{fig:opt_plan} shows  a curve 
from an optimal plan which does not arise from a map and Figure \ref{fig:monge} a curve induced by a plan 
with Monge map. However, in each case, the plan $\Gamma_o(\mu_s,\mu_1)$ is induced by a map for $s \in (0,1)$. \\
This behavior is not necessarily the case for non-optimal plans even if they are induced by a map. 
As an example, see Figure \ref{fig:non_monge}, 
where the paths of two particles cross at a certain time $s \in (0,1)$,
and $\Gamma_0(\mu_s,\mu_1)$ cannot be obtained from a map.
\hfill $\diamond$
\end{rem}

\begin{figure}[h!]
\centering
\begin{subfigure}{0.32\linewidth}
    \includegraphics[width=\linewidth]{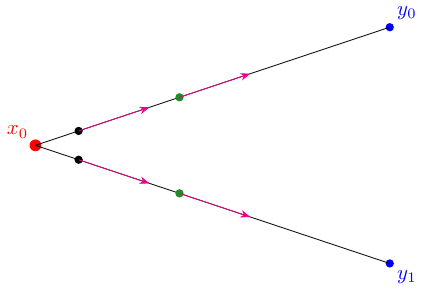}
    \caption{$\mu_t$ for an optimal plan without map.}\label{fig:opt_plan}
\end{subfigure}
\begin{subfigure}{0.32\linewidth}
    \includegraphics[width=\linewidth]{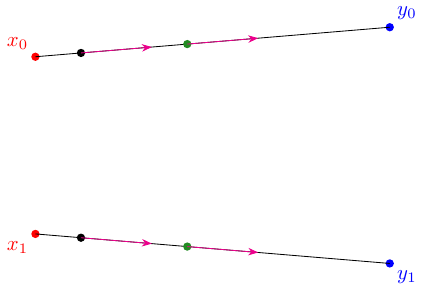}
    \caption{$\mu_t$ for an optimal plan with Monge map $T$.}\label{fig:monge}
\end{subfigure}
\begin{subfigure}{0.32\linewidth}
    \includegraphics[width=\linewidth]{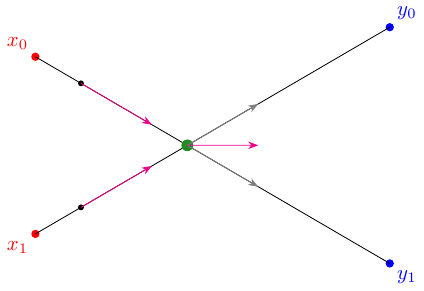}
    \caption{$\mu_t$ for a non-optimal plan with map $T$.}
    \label{fig:non_monge}
\end{subfigure}
\caption{Curve induced by $\alpha=(\Id,T)_\sharp\mu_0$ from 
$\mu_0 = \delta_{x_0}$, resp.
$\mu_0 = \tfrac12 (\delta_{x_0} + \delta_{x_1})$ to $\mu_1 = \tfrac12 (\delta_{y_0} + \delta_{y_1})$.
In (c), at the crossing time $s$ of the path,
there does not exist a map $T_s$ that induces an element in
$\Gamma_o(\mu_s,\mu_1)$. Red arrows: vector fields computed via  \eqref{eq:speed_1}}.
\end{figure}

\begin{rem}[Random variables] 
Curves induced by plans can also be interpreted via random variables: Let $X_0,X_1:\Omega\to\R^d$ be random variables with laws $P_{X_0} = \mu_0$ resp. $P_{X_1} = \mu_1$. 
Then $\alpha = P_{X_0,X_1} \in \Gamma(\mu_0,\mu_1)$ and the induced curve in \eqref{eq:ind_c} reads as
$$
\mu_t= P_{X_{t}}, \quad \text{where} \quad X_t\coloneqq (1-t)X_0+ t X_1.
$$
In particular, if $X_0$ and $X_1$ are independent, 
then $P_{X_0,X_1}= \mu_0 \times \mu_1$ and 
$$\mu_t= P_{X_t}=\left((1-t)_\sharp P_{X_0}\right)*\left(t_\sharp P_{X_1}\right).$$ 
Conversely, for $\alpha\in\Gamma(\mu_0,\mu_1)$, there is a random variable $Z:\Omega\to \R^d\times \R^d$ with law $P_Z=\alpha$, we can choose 
$X_0 \coloneqq \pi^1\circ Z$, 
$X_1 =\pi^2\circ Z$ 
and obtain 
$\mu_0=P_{X_0},\, \mu_1=P_{X_1}$ and $P_{X_0,X_1}=\alpha$.   
\hfill $\diamond$
\end{rem}

First of all, curves induced by plans are narrowly continuous, as the following lemma shows.
\begin{lem}\label{thm:n_cont}
    Let $\mu_0,\mu_1\in \P_2(\R^d)$ and $\alpha\in \Gamma(\mu_0,\mu_1)$. 
		Then $\mu_t \coloneqq e_{t,\sharp}\alpha$  is narrowly continuous.
\end{lem}

\begin{proof}
Let $f\in C_b(\R^d)$. Then $f\circ e_t\in C_b(\R^d\times\R^d)$ is measurable and bounded by
$| (f\circ e_t) (x,y)|\leq \|f\|_{\infty}$ for all $(x,y) \in \R^d\times\R^d$
and $f\circ e_{t'}\to f\circ e_t$ pointwise for $t'\to t$. Thus, by the dominated convergence theorem, we have 
\begin{align}
\lim_{t'\to t}\int_{\R^d} f \, \dd\mu_{t'}
& =\lim_{t'\to t}\int_{\R^d} f\circ e_{t'}\, \dd\alpha = \int_{\R^d} \lim_{t'\to t}f\circ e_{t'}\, \dd\alpha\\
&=\int_{\R^d} f\circ e_t \, \dd\alpha
=\int_{\R^d} f\dd\mu_t
\end{align}
and hence the claim.
\end{proof}

Moreover, curves induced by optimal plans are geodesics, see \cite[Chapter 7.2]{AGS2008}.

\begin{prop} \label{prop:geodesic}
   Let $\mu_0,\mu_1\in \P_2(\R^d)$ and $\alpha\in \Gamma_o(\mu_0,\mu_1)$. 
		Then $\mu_t \coloneqq e_{t,\sharp}\alpha$ 
		 is a geodesic between $\mu_0$ and $\mu_1$ and every 
         geodesic connecting $\mu_0$ and $\mu_1$ is of this form.
\end{prop}

Next, we want to show that a curve induced by a plan is also absolutely continuous. For this, we have to find a velocity field
fulfilling i) and ii) in Theorem \ref{thm:abscont_ce}
so that $(\mu_t,v_t)$ satisfy (CE).
To this end, we associate to a plan $\alpha$ the velocity field  
$v: [0,1] \times \R^d \to \R^d$ as follows:
Consider
$\underline {{\boldsymbol \alpha}}\coloneqq \mathcal{L}_{[0,1]}\times \alpha$ and 
$$e:[0,1]\times \R^d\times \R^d\to [0,1]\times\R^d, ~(t,x,y)\mapsto (t,e_t(x,y)).$$ 
Then, by \cite[Section 17]{ambrosio2021lectures}, there exists a \emph{unique vector field} $v\in L^2(e_\sharp \underline {{\boldsymbol \alpha}},\R^d)$ 
such that 
\begin{align}\label{eq:v_char}
v (e_\sharp \underline {{\boldsymbol \alpha}}) = e_\sharp[(y-x)\underline {{\boldsymbol \alpha}}].
\end{align}
Since every $e_\sharp{\boldsymbol{\alpha}}$-measurable function has a Borel measurable representative, see \cite[Proposition 2.1.11]{bogachev2007measure}, we can choose $v$ Borel measurable.
We call a measurable \emph{vector field} $v: [0,1] \times \R^d \to \R^d$
\emph{induced by a plan} $\alpha$, if it fulfills \eqref{eq:v_char}.

\begin{lem}\label{eq:velo}
Let $\mu_t = e_{t,\sharp} \alpha$ be the curve induced by the plan $\alpha$.
Then \eqref{eq:v_char} is equivalent to
\begin{equation} \label{eq:ind_v}
v_t\mu_t=e_{t,\sharp}[(y-x)\alpha] \quad \text{for a.e. } t\in [0,1].
\end{equation}
\end{lem}

\begin{proof}
For $f \in  C_b([0,1] \times \R^d)$, the left- hand side of \eqref{eq:v_char} means 
\begin{align}
\int_0^1 \int_{\R^d } f v \, \dd e_\sharp\underline {{\boldsymbol \alpha}}
&= \int_0^1\int_{\R^d \times \R^d} f(t,e_t(x,y))v(t,e_t(x,y)) \, \dd\underline {{\boldsymbol \alpha}}\\
&= \int_0^1\int_{\R^d} f v \, \dd(e_{t,\sharp} \alpha)  \dd t
\\
&=\int_0^1\int_{\R^d } f(t,x) v(t,x) \, \dd\mu_t \dd t
\end{align}
and the right-hand side
\begin{align}
\int_0^1 \int_{\R^d}f \, \dd e_\sharp[(y-x)\underline {{\boldsymbol \alpha}}]
&=
\int_0^1 \int_{\R^d \times \R^d} f(t,e_t(x,y))(y-x) \, \dd\underline {{\boldsymbol \alpha}}\\
&=
\int_0^1\int_{\R^d} f\, \dd e_{t,\sharp}[(y-x)\dd\alpha] \dd t.
\end{align}
Thus, \eqref{eq:ind_v} implies \eqref{eq:v_char}.

Conversely, assume that \eqref{eq:v_char} holds true. Then we have
$$
\int_0^1\int_{\R^d } f(t,x) v(t,x) \, \dd\mu_t \dd t 
= 
\int_0^1\int_{\R^d} f\, \dd e_{t,\sharp}[(y-x)\dd\alpha] \dd t.
$$
Let $\{g_n\}_{n\in \NN}\subset C_b(\R^d)$ be a dense subset. Then we obtain for any $h\in C_b([0,1])$ and all $n \in \NN$ that
\begin{align}
\int_0^1 h\left(\int_{\R^d}g_n v_t \, \dd\mu_t-\int_{\R^d}g_n \, \dd e_{t,\sharp}[(y-x)\alpha]\right)\dd t = 0 
\end{align}
and consequently 
\begin{align}\label{ax}
\int_{\R^d}g_n v_t \, \dd\mu_t=\int_{\R^d}g_n \, \dd e_{t,\sharp}[(y-x)\alpha] \quad \text{a.e. } t\in[0,1].
\end{align}
Since the number of test functions $g_n$ is countable, there exists a zero set $\mathcal I \subset [0,1]$ such that
\eqref{ax} holds true on $[0,1] \setminus \mathcal I$ for all $n \in \NN$.
 Since $\{g_n\}_{n\in\NN}$ is dense in $C_b(\R^d)$ we obtain that $v_t\mu_t=e_{t,\sharp}[(y-x)\alpha]$ for a.e. $t \in [0,1]$. More precisely, the last equation is an equation of vector-valued finite signed measures and thus it is enough to test equality at a dense subset of $C_b(\R^d)$.
\end{proof}

Using the above lemmas, we obtain that curve-velocity pairs induced by couplings have the following favorable properties.   

\begin{thm}\label{prop:alpha_curve}
Let $\mu_0,\mu_1 \in \mathcal P_2(\R^d)$ and  $\alpha\in \Gamma(\mu_0,\mu_1)$. 
Let $(\mu_t,v_t)$ be a curve-velocity pair induced by $\alpha$, 
i.e.
\begin{equation} \label{meanth}
\mu_t = e_{t,\sharp} \alpha, 
\quad \text{and} \quad
v_t\mu_t=e_{t,\sharp}[(y-x)\alpha] \quad \text{for a.e. } t\in [0,1].
\end{equation}
Then the following holds true:
\begin{itemize}
    \item[\rm{i)}] 
   $(\mu_t,v_t)$ satisfy \eqref{eq:ce} and
    $$  
    \|v_t\|_{L^2(\R^d,\mu_t)}\leq \|y-x\|_{L^2(\R^d \times \R^d,\alpha)} \quad \text{ for a.e. } t \in [0,1].
    $$
    In particular, we have $\|v_t\|_{L^2(\R^d,\mu_t)}\in L^1(I)$
    and $\mu_t$ is an absolutely continuous curve.
     \item[\rm{ii)}] If $\alpha\in \Gamma_o(\mu_0,\mu_1)$,  then $W_2(\mu_t,\mu_{t+h})=hW_2(\mu_0,\mu_1)$ and 
     \begin{align}\label{eq:v_leq_n}
     |\mu'_t|=
     W_2(\mu_0,\mu_1)
     =\|v_t\|_{L^2(\R^d,\mu_t)}
     \quad  \text{for a.e. } t \in[0,1].
     \end{align}
\end{itemize}
\end{thm}

\begin{proof}
 Assertion i)  follows as in \cite[Theorem 17.2, Lemma 17.3]{ambrosio2021lectures}, see also \cite[proof of Proposition 6]{chemseddine2024conditional}.
 
 For Assertion ii), let $\alpha$ be an optimal plan. Since $\mu_t$ is a  geodesic, we obtain 
 $W_2(\mu_t,\mu_{t+h})=h W_2(\mu_0,\mu_1)$, which immediately implies the first equality \eqref{eq:v_leq_n}. 
 By i) and since $\alpha$ is optimal, we know that $\|v_t\|_{L^2(\R^d,\mu_t}\leq W_2(\mu_0,\mu_1)$ for a.e. $t\in[0,1]$. Finally, by Theorem \ref{thm:abscont_ce}, we have $\|v_t\|_{L^2(\R^d,\mu_t)}\geq |\mu_t'|=W_2(\mu_0,\mu_1)$ for a.e. $t\in[0,1]$. 
\end{proof}

By the above theorem, we immediately see that the Wasserstein distance between two measures can be described by the velocity field of any geodesic connecting them (curve energy).

\begin{cor}[Benamou-Brenier]\label{prop:benamou-brenier}
The Wasserstein distance is given by 
\begin{align}
W_2(\mu,\nu)
=\min_{(\mu_t,v_t)}
\left(
\int_0^1 \|v_t\|^2_{L^2(\R^d,\mu_t)} \, \dd t
\right)^\frac12
\end{align}
where the minimum is taken over all pairs $(\mu_t,v_t)$, where $v_t$ is a measurable vector field, $\mu_t$ a narrowly continuous curve
with $\mu_0=\mu,\mu_1=\nu$ and $(\mu_t,v_t)$ satisfy \rm{\eqref{eq:ce}}.
\end{cor}

By the following corollary, curve-velocity pairs induced by optimal plans
have the favorable property of minimal vector fields.
This ensures ''short'' curves in the related ODE. 

\begin{cor}\label{prop:optimal_tangential}
Let $\mu_0,\mu_1 \in \mathcal P_2(\R^d)$ and  $\alpha\in \Gamma_o(\mu_0,\mu_1)$. Let $(\mu_t,v_t)$ be curve-velocity pair induced by $\alpha$.
Then, $v_t \in \mathcal T_{\mu_t}$ for a.e. $t\in[0,1]$.
\end{cor}

\begin{proof} By Theorem \ref{prop:alpha_curve}$ii)$ we know that 
$\|v_t\|_{ L}^2(\R^d,\mu_t)= |\mu'_t|$ for a.e. $t\in[0,1]$. Thus \cite[Proposition 8.4.5]{AGS2008} implies $v_t\in\T_{\mu_t}$ for a.e. $t\in[0,1]$.
\end{proof}
 If $\alpha$ is not only optimal, but also induced by a Monge map $T:\R^d\to \R^d$, it is often the case that the trajectories are straight, in the sense that the solution of $\partial_t\phi_t(x)=v_t(\phi_t(x)),\, \phi_0(x)=x$, is $T_t$. This is in particular the case, if $\mu_0,\mu_1$ are empirical measures with the same number of points. For a proof and further cases, e.g. if both measures admit densities with a certain regularity, see \cite[Proposition 16]{chemseddine2024conditional}.

The above minimality property, i.e. laying in the tangent space, is quite special for optimal plans.
Indeed, by the following example, this is not the case  for arbitrary induced velocity fields, see also \cite[Example 3.5]{liu2022rectified}. 

\begin{figure}[t]
\begin{subfigure}{0.245\linewidth}
    \includegraphics[width=\linewidth]{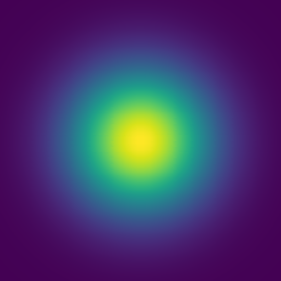}
    \caption{$\mu_0$}
\end{subfigure}
\begin{subfigure}{0.245\linewidth}
\includegraphics[width=\linewidth]{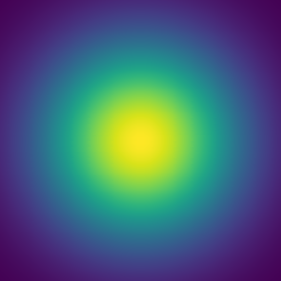}
\caption{$\mu_1$}
\end{subfigure}
\begin{subfigure}{0.245\linewidth}
\includegraphics[width=\linewidth]{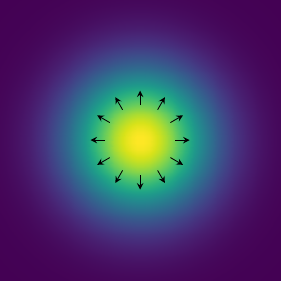}
\caption{$v_0\in \T_{\mu_0}$}
\end{subfigure}
\begin{subfigure}{0.245\linewidth}
\includegraphics[width=\linewidth]{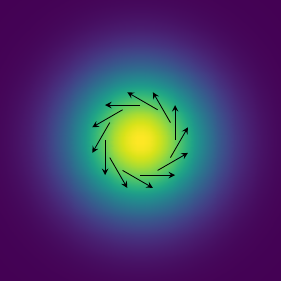}
\caption{$v_0\notin \T_{\mu_0}$}
\end{subfigure}
\caption{Illustration to Example \ref{example:not_tangential}. Both vector fields generate the same curves $\mu_t$ but in $\rm{(d)}$ mass is rotated unnecessarily, which makes the trajectories of single particles longer than for the vector field in $\rm{(c)}$.} \label{fig:not_tangential}
\end{figure}

\begin{example}\label{example:not_tangential}
We consider a  radial density
$p_0(x) \coloneqq f(\|x\|^2)$ for some  $f\in C^\infty(\R)$
with corresponding measure $\mu_0 \coloneqq p_0 \, \dd x$,
and the linear maps $T_t: \R^2 \to \R^2$, $t \in [0,1]$
defined by
$$
T_t(x) \coloneqq \left(
\begin{array}{rr}
1&-t\\t&1
\end{array}
\right) x = 
x + t w(x) , \quad w(x) \coloneqq \left(
\begin{array}{rr}
0&-1\\1&0
\end{array}
\right)
x,
$$
which fulfill
$$
T_t^{-1}(x) = \frac{1}{1+t^2}\left(
\begin{array}{rr}
1&t\\-t&1
\end{array}
\right) x = \frac{1}{1+t^2}
\big(x - t w(x)\big), \quad \text{det} \big( \nabla T_t(x) \big) = \frac{1}{1+t^2}.
$$
We are interested in the plan induced by $T=T_1$, i.e.,
$$
\alpha \coloneqq (\Id,T)_\sharp \mu_0 \in \Gamma(\mu_0 ,T_\sharp \mu_0).
$$
We will show that the induced vector field 
$v_t\mu_t=e_{t,\sharp}((y-x)\alpha)$
is not minimal.
By the change of variable  formula \eqref{push_density_1},
the induced curve 
$$\mu_t\coloneqq e_{t,\sharp}\alpha =T_{t,\sharp} \mu_0,$$
has the density 
\begin{align}
p_t(x)=\frac{(\rho\circ T_t^{-1})(x)}{\det(\nabla T_t(x))}
&=(1+t^2)  f\left(\tfrac{1}{1+t^2}\|x\|^2\right)
\end{align}
which is radial again. 
In particular, the gradient $\nabla p_t(x)$  is a multiple of $x$ and therefore we get by definition of $w$ that  $\langle \nabla p_t(x),w(x)\rangle=0$. 
Further, we have  $\nabla \cdot w=0$ and consequently
$$
\nabla\cdot(p_t w)=\langle \nabla p_t ,w\rangle + p_t \nabla \cdot w =0
$$
Hence, by \eqref{tang}, 
we know that every $u_t\in\T_{\mu_t}$ must satisfy 
$\int_{\R^d} \langle w,u_t\rangle \, \dd\mu_t =0$
for a.e. $t \in [0,1]$.
Unfortunately, the induced velocity field
fulfills 
$\int_{\R^d} \langle w,v_t\rangle \, \dd\mu_t>0$ for an non zero subset of $[0,1]$ by the following reasons:
\begin{align}
\int_{\R^d} \langle w,v_t\rangle\dd\mu_t 
&= 
\int_{\R^d \times \R ^d} \langle w(e_t(x,y)),y-x\rangle \, \dd\alpha \\
&= 
\int_{\R^d \times \R ^d} \langle w(e_t(x,y)),y-x\rangle \, \dd[(\Id,T)_\sharp \mu_0]
\\
&=\int_{\R^d} \langle w((1-t)x+tT(x)),T(x)-x\rangle \, \dd \mu_0\\
&=\int_{\R^d} \langle w(x+tw(x)),w(x)\rangle \, \dd\mu_0.
\end{align}
 The last expression is continuous in $t$ and equal to $\|w\|^2_{L^2(\R^2,\rho)}$ for $t=0$. Since $\|w\|^2_{L^2(\R^2,\rho)}>0$, there exists an open non empty interval, such that  $\int_{\R^d} \langle w,v_t\rangle \, \dd\mu_t>0$. Thus, we have $v_t\notin \mathcal T_{\mu_t}$.
  \hfill $\diamond$
\end{example}

\begin{example}
    Let $\mu_0=\frac 12\delta_{x_0}+\frac12\delta_{x_1}$, $\mu_1=\frac13\delta_{y_0}+\frac23\delta_{y_1}$ and $\alpha=\mu_0\times \mu_1$. Then 
    \begin{align}
    \mu_t &= \frac16\delta_{e_t(x_0,y_0)}+ \frac13\delta_{e_t(x_0,y_1)}+\frac16\delta_{e_t(x_1,y_0)}+ \frac13\delta_{e_t(x_1,y_1)}\\
    \alpha_t &= \frac16\delta_{e_t(x_0,y_0),y_0}+ \frac13\delta_{e_t(x_0,y_1),y_1}+\frac16\delta_{e_t(x_1,y_0),y_0}+ \frac13\delta_{e_t(x_1,y_1),y_1}.
    \end{align}
    Assume furthermore that for some $s\in(0,1)$ we have that $e_{s}(x_0,y_1)=e_{s}(x_1,y_0)\eqqcolon \hat{x}_{s}$, but $e_s(x_0,y_0)\neq \hat{x}_s\neq e_s(x_1,y_1)$, see Figure \ref{fig:indep}. We then have that 
    \begin{align}
    \mu_s &= \frac16\delta_{e_s(x_0,y_0)}+ \frac12\delta_{\hat{x}_{s}}+ \frac13\delta_{e_s(x_1,y_1)}\\
    \alpha_s &= \frac16\delta_{e_s(x_0,y_0),y_0}+ \frac13\delta_{\hat{x}_{s},y_1}+\frac16\delta_{\hat{x}_{s},y_0}+ \frac13\delta_{e_s(x_1,y_1),y_1},
    \end{align}
    which implies $\alpha_s^{\hat{x}_{s}}=\frac13\delta_{y_0}+\frac23\delta_{y_1}$. Hence using Proposition \ref{prop:vector_dis} we obtain
    \begin{align}
        v_{s}(\hat{x}_{s})=\frac13\frac{y_0-\hat{x}_{s}}{1-s}+\frac23\frac{y_1-\hat{x}_{s}}{1-s}.
    \end{align}
    Note that since at $\mu_s$ the mass in $\hat{x}_s$ has to be split, $\mu_t$ cannot be described by a pushforward of a solution of an ODE and thus the assumptions of Theorem \ref{prop:abs_flow} cannot be fulfilled. \hfill $\diamond$

\begin{figure}[h!]
    \centering
    \includegraphics[width=0.4\linewidth]{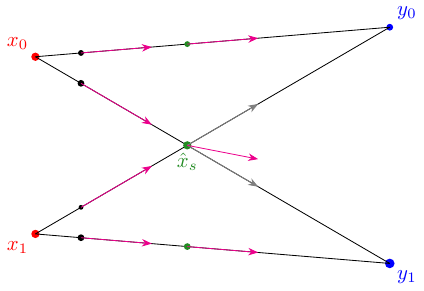}
    \caption{Curve and vector field associated to $\alpha=\mu_0\times \mu_1$, where 
    $\mu_0 = \tfrac12 \delta_{x_0} + \tfrac12 \delta_{x_1}$ 
    and 
        $\mu_1 = \tfrac13 \delta_{y_0} + \tfrac23 \delta_{y_1}$. Vectors are scaled by $0.2$ for better visibility.}
    \label{fig:indep}
\end{figure}
\end{example}
We have already seen that vector fields $v_t$ associated to optimal plans are minimal ones, meaning that $v_t\in \mathcal T_{\mu_t}$.
This is in general not true for an independent coupling 
$\alpha = \mu_0 \times \mu_1$ with an arbitrary $\mu_0$,
see Figure \ref{fig:indep}.
Fortunately,  the  independent coupling with a Gaussian marginal $\mu_0$ has this minimality property as the following example shows,  see also \cite[Section 5.1]{liu2022rectified}.

\begin{prop}\label{prop:flow_score}
Let $\mu_0\sim \N(0,I_d)$, $\mu_1 \in \mathcal P_2(\R^d)$ and $\alpha=\mu_0\times \mu_1$. Let $\mu_t \coloneqq e_{t,\sharp}\alpha$
be the induced curve and $v_t$  the induced velocity field $v_t$  in \eqref{eq:ind_v}. Then $\mu_t$ admits  
the strictly positive density 
\begin{equation} \label{eq:pt}
p_t(x) \coloneqq (1-t)^{-d}(2\pi)^{-\frac{d}{2}}\int_{\R^d} {\rm{e}}^{-\frac{\|x-ty\|^2}{2(1-t)^2}}\dd\mu_1(y), \quad t \in [0,1),
\end{equation}
and 
\begin{equation} \label{eq:vt}
v_t = \nabla f_t 
 \quad \text{with} \quad 
 f_t \coloneqq \frac{1-t}{t}\log p_t+ \frac{1}{2t}\|\cdot\|^2, \quad  t \in (0,1),
\end{equation}
Further, it holds $v_t\in \mathcal T_{\mu_t}$.
\end{prop}

The expression $\nabla \log p_t$ will reappear as so-called ''score''
in Section \ref{sec:diffusion}.

\begin{proof}
1. Relation \eqref{eq:pt} follows by directly computing for $g\in C_c^\infty(\R^d)$,
\begin{align}
\int_{\R^d} g \, \dd\mu_t
&= 
\int_{\R^d \times \R^d} g\left(e_t(x,y) \right) \, \dd \alpha(x,y)\\
&=\frac{1}{(2\pi)^{\frac{d}{2}}}\int_{\R^d \times \R^d} g\left( (1-t)x + ty) \right) \text{e}^{-\frac{\|x\|^2}{2}} \, \dd x \dd\mu_1(y)
\end{align}
and substituting $x$ by $\frac{z-ty}{1-t}$ then
\begin{align}
\int_{\R^d} g \, \dd\mu_t
&= \frac{1}{(1-t)^d(2\pi)^{\frac{d}{2}}}
\int_{\R^d} g(z) \text{e}^{-\frac{\|z-ty\|^2}{2(1-t)^2}} \, \dd\mu_1(y)\dd z.
\end{align}
2. Concerning relation \eqref{eq:vt}, we verify
\begin{align}
\int_{\R^d} g \, \nabla f_t \, \dd\mu_t 
&= 
\frac{1-t}{t}\int_{\R^d} g(x)\frac{\nabla p_t(x)}{p_t(x)}p_t(x) \, \dd x + \frac{1}{t}\int_{\R^d} g(x) \, x \, p_t(x) \, \dd x\\
&= 
\frac{1}{t (1-t)^{d +1}(2\pi)^{\frac{d}{2}}}
\int_{\R^d \times \R^d} g(x)(ty-x)
\text{e}^{-\frac{\|x-ty\|^2}{2(1-t)^2}}\dd\mu_1(y) \, \dd x \\
& \quad
+ \frac{1}{t(1-t)^d(2\pi)^{\frac{d}{2}}}
\int_{\R^d\times \R^d} g(x) x \text{e}^{-\frac{\|x-ty\|^2}{2(1-t)^2}}\, \dd\mu_1(y)\dd x\\
&
= \frac{1}{(1-t)^{d+1}(2\pi)^{\frac{d}{2}}}\int_{\R^d\times \R^d} g(x)(y-x) \text{e}^{-\frac{\|x-ty\|^2}{2(1-t)^2}}\, \dd\mu_1(y) \, \dd x.
\end{align}
On the other hand, we have for the induced velocity field that
\begin{align}
\int_{\R^d}  g \, v_t \, \dd \mu_t &=
\int_{\R^d}  g \, \dd e_{t,\sharp}(y-x)\alpha
=
\frac{1}{(2\pi)^{\frac{d}{2}}}\int_{ \R^d \times \R^d} g(e_t(x,y))(y-x) \text{e}^{-\frac{\|x\|^2}{2} }\, \dd\mu_1(y)\dd x\\
&=
\frac{1}{(1-t)^{d+1}(2\pi)^{\frac{d}{2}}}
\int_{\R^d \times \R^d} g(z)(y-z)\text{e}^{-\frac{\|z-ty\|^2}{2(1-t)^2}} \, \dd\mu_1(y)\dd z,
\end{align}
which yields the assertion. Note that, in particular $\nabla f_t\in L^2(\R^d,\mu_t)$.
\\
3. It remains to prove the minimality property of $v_t$.
Since $\mu_t\in \P_2(\R^d)$, we have  
$\nabla \|x\|^2\in \mathcal T_{\mu_t}$. 
Thus, we are left to show that $\nabla\log p_t\in \mathcal T_{\mu_t}$.
\\
First, we have $p_t\in C^{\infty}(\R^d)$, $t\in [0,1)$ and since $p_t>0$ also $\log p_t \in C^\infty(\R^d)$. 
Since $\nabla \|x\|^2, \nabla f_t \in L^2(\R^d,\mu_t)$, we conclude that
also $\nabla \log p_t \in L^2(\R^d,\mu_t)$.
Furthermore, it is easy to see that 
$\lim_{r \to \infty}\max\{\log p_t(x):\|x\|\geq r\}\to -\infty$. 
Consequently, we obtain for functions $h_n\in C_c^\infty(\R)$ 
that 
$\phi_n\coloneqq h_n\circ\log p_t\in C_c^\infty(\R^d)$.
Now consider especially functions with
$h_n(s)=s$ for $ -n \le s \le \|\log p_t\|_\infty$ and 
$|h_n'|\leq 1$. 
Then we obtain with
$C_n \coloneqq\{x\in \R^d: \log p_t(x) \geq -n\}$ 
that $\phi_n(x)=\log p_t(x)$ for $x \in C_n$ and thus
\begin{align}
\int_{\mathbb R^d} \|\nabla \phi_n-\nabla \log p_t\|^2 \, \dd\mu_t 
&= \int_{\R^d \setminus C_n}\|\nabla \phi_n-\nabla \log p_t\|^2 \, \dd\mu_t
\\
&
\leq 4\int_{\R^d \setminus C_n}\|\nabla \log p_t\|^2 \, \dd\mu_t.
\end{align}
Since $\nabla \log p_t\in L^2(\R^d,\mu_t)$, the latter converges to zero.
Hence, by \eqref{tan_reg}, this yields  $\nabla \log p_t \in \mathcal T_{\mu_t}$.
\end{proof}
\begin{example} \label{ex:v_explode}
We consider $\mu_0\sim \N(0,1)$, $\mu_1 \coloneqq \delta_0$ and $\alpha=\mu_0\times \mu_1$. Then,  by \eqref{eq:pt}, the  curve $\mu_t \coloneqq e_{t,\sharp}\alpha$ admits the strictly positive density 
\begin{equation}
p_t(x) = (1-t)^{-1}(2\pi)^{-\frac{1}{2}} {\rm{e}}^{-\frac{x^2}{2(1-t)^2}}, \quad t \in [0,1), ~ x \in \R.
\end{equation}
Following \eqref{eq:vt}, we calculate
\begin{align}
    f_t(x) 
    &\coloneqq \frac{1-t}{t} \log p_t(x)+ \frac{1}{2t}x^2 \\
    &= \frac{1-t}{t} 
    \big( 
    \log\left( (1-t)^{-1}(2\pi)^{-\frac12} \right) - \frac{x^2}{2(1-t)^2}  
    \big) + \frac{1}{2t}x^2
    \\
    &= \frac{1-t}{t} \log\left((1-t)^{-1}(2\pi)^{-\frac12}\right) -\frac{x^2}{2t(1-t)}
     + \frac{1}{2t}x^2
\end{align}
and the induced velocity field 
\begin{align}
  v_t(x) &= \nabla f_t = \frac{1}{t}x -\frac{x}{t(1-t)} \\
   &= \frac{x}{t-1}.
\end{align}
Hence, for any $x \in \R \setminus \{0\}$, the absolute value $|v_t(x)|$ explodes for $t \to 1$. Note that despite the exploding velocity field, the velocity of a trajectory remains constant over time. More precisely, the map $T_t(x)=(1-t)x$ is the solution of the flow ODE $\partial_t \phi(t,x)=v_t(\phi(t,x)),\phi(0,x)=x$ and thus the speed of a trajectory starting at $x$ is constant in time and equal to $v_t(T_t(x))=-x$.
\end{example}

\section{Curves via Markov kernels} \label{sec:mk}
Unfortunately, the approach of Lipman et al. \cite{lipman2023flow},
see Example \ref{ex:conv}, is not covered by the previous section.
Indeed, we need a more general viewpoint via Markov kernels
which is given in the following Subsection \ref{subsec:mk1}.
We will see in Subsection \ref{subsec:mk2} that curve-velocity pairs induced by couplings fit into this general  setting.

\subsection{General Construction} \label{subsec:mk1}

The main observation is the following theorem, which describes how to construct a curve velocity pair $(\nu_t,v_t)$ when we are given a family of Markov kernels $\K_t(y,\cdot)$ and a measure $\mu_1$. Note that in this construction it is not necessarily the case that $\nu_0$ is tractable and $\nu_1=\mu_1$. But we will see later that for suitable chosen $\K_t(y,\cdot)$, we can achieve $\nu_0=\N(0,1)$ and $\nu_1\cong \mu_1$.  

\begin{thm}\label{prop:lip_flow}
Let $(\K_t)_{t \in [0,1]}$ be a family of Markov kernels 
$\K_t:\R^d\times \B(\R^d)\to \R$.
For a given measure $\mu_1 \in \mathcal P_2(\R^d)$, we consider 
$$\alpha_t\coloneqq \K_t(y, \cdot) \times_y \mu_1$$
and introduce the measure
$$
\nu_t\coloneqq \pi^1_\sharp\alpha_t \quad \text{so that} 
\quad \alpha_t = \alpha_t^x \times_x \nu_t.
$$
Assume that $\K_t(y,\cdot)$ is an absolutely continuous curve in $\P_2(\R^d)$
for $\mu_1$ - a.e. $y \in \R^d$ with a 
vector field $v_t^y$  such that ($\K_t(y,\cdot),v_t^y)$ fulfills \eqref{eq:ce} for a.e. $y$.
Further, let  $(t,x,y)\mapsto v_t^y$ be measurable and 
$\|v_t^y \|_{L^2(\R^d \times \R^d,\alpha_t)} \in L^1([0,1])$.
Then  $\alpha_t$ and $\nu_t$ are narrowly continuous.
The curve $\nu_t$ is absolutely continuous and $(\nu_t,v_t)$ satisfies  \eqref{eq:ce}, where
    \begin{equation} \label{eq:speed_2}
    v_t (x) \coloneqq \int_{\R^d} v_t^y  \, \dd\alpha_t^x(y).
    \end{equation}
 \end{thm}

\begin{proof}
1. For  $f\in C_b(\R^d\times\R^d)$ with
$\|f\|_\infty\leq C$ we have  $\left|\int_{\R^d} f \, \dd \K_t(y,\cdot)\right|\leq C$. 
Thus, using the dominated convergence theorem and the fact that $\K_t(y,\cdot)$ is absolutely continuous for $\mu_1$-a.e. $y \in \R^d$, we conclude 
\begin{align}
    \lim_{t\to t'}\int_{\R^d \times \R^d} f \, \dd \alpha_t
    &=
    \int_{\R^d} \lim_{t\to t'}\int_{\R^d} f \, \dd\K_t(y,\cdot)\dd \mu_1(y)
    =
    \int_{\R^d} \int_{\R^d} f \, \dd K_{t'}(y,\cdot) \, \dd\mu_1(y)\\
    &= 
    \int_{\R^d \times \R^d} f  \, \dd \alpha_{t'}  
    \quad \text{for all}  \quad
    f \in C_b(\R^d \times \R^d),
\end{align}
so that $\alpha_t$ is narrowly continuous 
and the same holds true for $\nu_t=\pi^1_{\sharp}\alpha_t$.
\\
2. For any $ \varphi \in  C_c^\infty([0,1] \times \R^d)$, we obtain 
\begin{align}\label{eq:cond_int_abs}
&\int_0^1 \int_{\R^d} \Big\langle\nabla_x\varphi(t,x),\int_{\R^d} v_t^y(x) \, \dd\alpha_t^x(y) \Big\rangle \, \dd \nu_t(x)\dd t\\
&=
\int_0^1 \int_{\R^d \times \R^d} \langle \nabla_x\varphi(t,x),v_t^y(x)\rangle \, \dd\alpha_t(x,y)
\, \dd t\\
&=
\int_0^1\int_{\R^d} \int_{\R^d} \langle \nabla_x\varphi(t,x),v_t^y(x)\rangle \, \dd\K_t(y,\cdot)(x) \, \dd\mu_1(y)\, \dd t\\
&=
\int_{\R^d}\int_0^1 \int_{\R^d} \langle \nabla_x\varphi(t,x),v_t^y(x)\rangle \, \dd\K_t(y,\cdot)(x) \, \dd t  \, \dd\mu_1(y)\\
&=  -
\int_{\R^d}\int_0^1 \int_{\R^d}  \partial_t \varphi(t,x) \, \dd\K_t(y,\cdot)(x) \, \dd t  \, \dd\mu_1(y)\\
&=-
\int_0^1 \int_{\R^d \times \R^d} \partial_t \varphi(t,x) \, \dd\alpha_t \dd t
=
\int_0^1\int_{\R^d \times \R^d} \partial_t \varphi(t,x) \, \dd (\pi^1_\sharp\alpha_t) \dd t\\
&=
-
\int_0^1\int_{\R^d} \partial_t\varphi(t,x) \, \dd \nu_t(x)\dd t,
\end{align}
so that $(\nu_t,v_t)$ fulfills \eqref{eq:ce}.
\\
3. Finally, we conclude
\begin{align}
\int_0^1 \|v_t\|_{L^2(\R^d,\nu_t)} \, \dd t
&= 
\int_0^1 \left(\int_{\R^d} \|v_t\|^2 \, \dd \nu_t\right)^{\frac 12} \dd t
\leq \int_0^1\left(\int_{\R^d}\int_{\R^d} \|v_t^y(x)\|^2 \, \dd\alpha_t^x(y) \dd\nu_t(x)\right)^{\frac 12}\dd t
\\
&=\int_0^1 \left(\int_{\R^d \times \R^d} \|v_t^y(x)\|^2 \, \dd\alpha_t\right)^{\frac 12}\dd t<\infty,
\end{align}
which finishes the proof, where the measurability of $v$ is left to the reader. 
\end{proof}

\begin{rem} In general, every curve $\mu_t\in\P_2(\R^d)$ can be written as in Theorem \ref{prop:lip_flow}. Namely, $\K_t(y,\cdot)\coloneqq \mu_t(\cdot)$ is a Markov kernel and for $\alpha_t\coloneqq \K_t(y,\cdot)\times_y\mu_1$ it holds that $\nu_t\coloneqq\pi^1_\sharp\alpha_t=\mu_t$. Of course, we don't win anything using this construction since $v_t^y=v_t$ and thus we do not obtain a more tractable vector field. \hfill $\diamond$
\end{rem}

\begin{rem}[Relation to Lipman et al. \cite{lipman2023flow}]\label{rem:orig}
The authors in \cite{lipman2023flow} argue only with densities. Then the notation 
of the so-called ''conditional probability path'' translates as
$$\K_t(y,\cdot)  \longleftrightarrow p_t(x|y).$$
Further, we have
$$
\alpha_t =\K_t(y, \cdot) \times_y \mu_1 \longleftrightarrow p_t(x,y), \quad
\alpha_t^x \longleftrightarrow p_t(y|x),\quad 
\pi^2_\sharp \alpha_t = \mu_1 \longleftrightarrow  q(y)
$$
and
$$
\int_{\R^d} \varphi \, \dd\nu_t(x)
=
\int_{\R^d \times \R^d} \varphi(x) \, \dd\K_t(y,\cdot)\dd\mu_1(y)
\longleftrightarrow
p_t(x) = \int_{\R^d} p_t(y|x) q(y) \, \dd y
$$
so that $\nu_t   \longleftrightarrow  p_t$.
By the Bayesian law,  we know that $p_t(y|x)= \frac{p_t(x|y)}{p_t(x)}q(y)$, so that we finally
obtain the correspondence
$$v_t(x)=\int_{\R^d} v_t^y(x)\, \dd\alpha_t^x(y) \longleftrightarrow 
v_t=\int_{\R^d} v_t(x|y)\frac{p_t(x|y)}{p_t(x)}q(y) \, \dd y.\qquad\diamond$$
\end{rem}

We recap the example in \cite[Example II]{lipman2023flow} with our approach. 

\begin{example}  \label{ex:conv}
For fixed $r \in (0,1)$, we consider the family of Markov kernels
$$
\K_t(y,\cdot) \coloneqq \N(ty,(1-rt)^2I_d), \quad t \in [0,1].
$$
By \eqref{ex:gaussian}, there is an optimal transport map $T: \R^d \to \R^d$
between 
the Gaussians $\K_0(y,\cdot) = \N(0,I_d) =: \mu_0$ and $\K_1(y,\cdot) = \N(y, (1-r)^2 I_d)$  
given by 
  $$
  T(x) = y + (1-r)x
  $$
  with corresponding plan $\gamma = (\text{Id},T)_\sharp \mu_0$
  and 
 $$
  T_t(x)= e_{t,\sharp} (x,T(x)) =
   (1-tr)x + ty, \quad
  T_t^{-1} (z) = \frac{z-ty}{1-tr}.
    $$ 
Let $\mu_1 \in \mathcal P_2(\R^d)$.    
Further, it can be easily verified that   
$$
\K_t(y,\cdot) = e_{t,\sharp} \left((\Id,T)_\sharp \mu_0 \right) = (T_t)_\sharp \mu_0.
$$
By Corollary \ref {cor:cm}, we obtain that $(\K_t(y,\cdot),v^y_t )$ with the velocity field
  \begin{align}
      v_t^{y}(x) =T(T_t^{-1}(x))-T_t^{-1}(x) =\frac{y-rx}{1-tr},
  \end{align}
fulfills \eqref{eq:ce}.
Now, let 
$\mu_1\in\P_2(\R^d)$ 
and consider the family of plans
$$
\alpha_t \coloneqq \K_t(y, \cdot) \times_y \mu_1 
\quad \text{and} \quad
\nu_t\coloneqq\pi^1_\sharp\alpha_t.
$$
It is easy to check that $\nu_0 = \N(0,I_d)$.
Then we know by Proposition \ref{prop:lip_flow} that
the curve $\nu_t$ is absolutely continuous and fulfills \eqref{eq:ce} with the velocity field $v_t$ in \eqref{eq:speed_2}. 
However, if we want to sample from the target density $\mu_1$,
this can only be done approximately by following the path of $\nu_t$,
since 
\[
\nu_1= \pi^1_\sharp\left[\N(y,(1-r)^2 I_d) \times_y \mu_1(y)\right]=\N(0,(1-r)^2I_d)*\mu_1 \neq \mu_1.
\]
By \cite[Lemma 7.1.10]{AGS2008} we have that $W_2(\nu_1,\mu_1)\leq (1-r)\sqrt{\int \|x\|^2\dd\mu_1}$ and thus $\nu_1$ converges to $\mu_1$ in $\P_2(\R^d)$ as $r\to 1$.
The authors of \cite{lipman2023flow} then learn the velocity field  $v_t$ of $\nu_t$ as explained in  Example \ref{ex:lip_ex_loss}.
\hfill $\diamond$
\end{example}

A  curve that can be described by Markov kernels, but is not induced by any coupling is given in the following example.

\begin{example}
	Let $\mu_t$ be a curve induced by  $\alpha \in \Gamma(\mu_0,\mu_1)$. Then, by Proposition \ref{prop:alpha_curve} and Corollary  \ref{prop:benamou-brenier} of  Benamou-Brenier,  we obtain  
	\[
		W_2(\mu_t,\mu_{t+h})^2\leq h^2 \|x-y\|_{L^2(\R^d\times\R^d,\alpha)}^2\leq h^2\left(\|x\|_{L^2(\R^d,\mu_0)}+\|x\|_{L^2(\R^d,\mu_1)}\right)^2,
\]
which is a bound independent of the plan.
Thus, if $\mu_t$ is not constant, we can find a monotone function $f\in C^\infty([0,1])$ with $f(0)=0,f(1)=1$, such that $\gamma_t\coloneqq \mu_{f(t)}$ cannot be induced by a plan by simply speeding up $\mu_t$ such that the above inequality for $W_2(\gamma_t,\gamma_{t+h})^2$ is not satisfied. However, we can construct it via a Markov kernel.  
Consider $\bar{\alpha}_t=\alpha_{f(t)}$, 
where as usual $\alpha_t=(e_t,\pi^2)_\sharp \alpha$ and $\alpha_t^y\times_y \mu_1=\alpha_t$. 
Then   
$\bar{\alpha}_t^y\coloneqq \alpha_{f(t)}^y$ 
is absolutely continuous. Furthermore,
for $\bar{v}_t^y\coloneqq f'(t)v^y_{f(t)}$ we obtain
that $(\bar{\alpha}_t^y,\bar{v}^y_t)$ 
fulfills the continuity equation. 
The curve $\bar{\mu}_t$ 
associated to the Markov kernel $\bar{\alpha}_t^y$ 
and $\mu_1$ is $\gamma_t$, which cannot be induced by a plan. \hfill $\diamond$
\end{example}

\subsection{Curves induced by Couplings via Disintegration}\label{subsec:mk2}
Next, let us see how curve-velocity pairs 
induced by couplings $\alpha \in \Gamma(\mu_0,\mu_1)$
fit into the general setting of the previous subsection.
Let
$(e_t, \pi^2):\R^d\times \R^d\to \R^d\times \R^d$ be given by 
$(x,y)\mapsto (e_t(x,y),y)$
and define a family of couplings
\begin{equation} \label{push_a_t}
\alpha_t:= (e_t, \pi^2)_\sharp\alpha, \quad t \in [0,1].
\end{equation}
Then $\alpha_t$ has the marginals
\begin{equation}\label{eq:marginals}
 \pi^1_\sharp \alpha_t = e_{t,\sharp}\alpha = \mu_t \quad \text{and} \quad
 \pi^2_\sharp \alpha_t = \mu_1,    
\end{equation}
and corresponding disintegrations
\begin{equation} \label{desint_alpha}
  \alpha_t =  \alpha_t^x \times_x \mu_t \quad \text{and} \quad
  \alpha_t = \alpha_t^y \times_y \mu_1 = \mathcal K_t(y,\cdot) \times_y \mu_1.
\end{equation}
By the following proposition,  $\alpha_t^y$ is a curve
induced by the independent coupling  of $\alpha^y$  and  $\delta_y$. 

\begin{prop} \label{prop:desint}
The  disintegration $\alpha_t^y$  in \eqref{desint_alpha}
fulfills
\begin{equation}\label{xx}
\alpha_t^y = e_{t,\sharp} (\alpha^y \times \delta_y) \quad \text{for } \mu_1-\text{a.e. } y 
\end{equation}
and
\begin{equation}\label{yy}
v_t^y(x) \coloneqq \frac{y-x}{1-t} , \quad t \in (0,1)
\end{equation}
is an induced velocity field of $\alpha^y \times \delta_y$.
In particular, $\alpha_t^y$ is absolutely continuous  for $\mu_1$-a.e. $y \in \R^d$,
$v_t^y \in \mathcal T_{\alpha_t^y}$
and
$(\alpha_t^y,v_t^y)$ fulfills \eqref{eq:ce}.
\end{prop}

\begin{proof}
For any measurable, bounded function $f:\R^d \to \R^d$, we obtain
\begin{align}
\int_{\R^d\times\R^d}f(x,y) \, \dd\alpha_t&=\int_{\R^d\times\R^d}f(e_t(x,y),y) \, \dd\alpha = \int_{\R^d\times\R^d}f(e_t(x,y),y)\dd\alpha^y(x)\dd\mu_1(y)\\
&=
\int_{\R^d\times\R^d}f(e_t(x,z),y) \, \dd(\alpha^y\times\delta_y)(x,z)\dd\mu_1(y)\\
&=\int_{\R^d\times\R^d}f(x,y)\dd e_{t,\sharp}[\alpha^y\times\delta_y](x) \, \dd\mu_1(y),
\end{align}
which implies  \eqref{xx}. 
Furthermore, we get for $v_t^y$ in \eqref{yy} that
\begin{align}
\int_{\R^d} f(x)v_t^y(x) \, \dd \alpha_t^y
&=
\int_{\R^d}f(x)\frac{y-x}{1-t} \, \dd e_{t,\sharp}(\alpha^y\times_y\delta_y)=\int_{\R^d}f(e_t(x,y))(y-x) \, \dd \alpha^y(x)\\
&=\int_{\R^d\times\R^d} f(e_t(x,z))(z-x) \, \dd (\alpha^y\times\delta_y)\\
&=\int_{\R^d\times\R^d}f \, \dd e_{t,\sharp}[(z-x)(\alpha^y\times\delta_y)].
\end{align}
Hence, by \eqref{eq:ind_v},  the velocity field $v_t^y$ is associated to $\alpha^y\times \delta_y$. 
Since 
$\Gamma(\alpha^y,\delta_y) = \Gamma_0(\alpha^y,\delta_y) = \{\alpha^y \times \delta_y\}$, we obtain by Theorem \ref{prop:alpha_curve} 
that $\alpha^y_t$
is absolutely continuous. By Corollary \ref{prop:optimal_tangential},
we know that $v_t^y \in \mathcal T_{\alpha_t^y}$.
\end{proof}

By the next proposition, we can rewrite the $\alpha$-induced velocity field in \eqref{eq:ind_v} using the disintegration $\alpha_t^x$.

\begin{prop}\label{prop:vector_dis}
For $\mu_0,\mu_1\in \P_2(\R^d)$, let $\alpha\in \Gamma(\mu_0,\mu_1)$
and
$\mu_t \coloneqq e_{t,\sharp}\alpha$.
Then
$v_t:\R^d\to\R^d$ defined by
\begin{equation} \label{eq:speed_1}
v_t(x) \coloneqq \int_{\R^d} \frac{y-x}{1-t} 
\, \dd \alpha^x_{t}(y), \quad t \in (0,1).
\end{equation}
is the induced velocity field by $\alpha$.
\end{prop}

\begin{proof}
The assertion follows by straightforward computation: 
\begin{align}
\int_{\R^d } f(x)v_t(x) \, \dd\mu_t
&= \int_{\R^d }\int_{\R^d } f(x) \frac{y-x}{1-t} \, \dd\alpha_t^{x}(y) \dd\mu_t(x)\\
&=\int_{\R^d \times \R^d} f(x)\frac{y-x}{1-t} \, \dd\alpha_t(x,y)\\
&=\int_{\R^d } f((1-t)x+ty)(y-x) \, \dd\alpha (x,y)\\
&= \int_{\R^d \times \R^d} f\dd e_{t,\sharp}\left[(y-x)\alpha\right].
\end{align}

\end{proof}
Note that it is not a priori clear, that there are representatives of $v_t$ such that $v_t$ seen as a function $[0,1]\times \R^d\to \R^d$ is measurable. We show in Proposition \ref{rem:v_t_alpha_x_mes} why this is the case. 

We summarize the last two propositions in the following theorem.

\begin{thm}
    Let $\mu_0,\mu_1\in\P_2(\R^d)$ and $\alpha\in\Gamma(\mu_0,\mu_1)$. Then the pair
$$
\mathcal K_t(y, \cdot) = e_{t,\sharp} (\alpha^y \times_y \delta_y), \quad v_t^y = \frac{y-x}{1-t}
$$ 
fulfills the conditions of Theorem \ref{prop:lip_flow}
and  in this case
$$
\mu_t \coloneqq e_{t,\sharp}\alpha = \nu_t \coloneqq \pi^1_\sharp (\mathcal K_t(y, \cdot) \times_y \mu_1), \quad
v_t(x) = \int_{\R^d} v_t^y  \, \dd \alpha^x_t(y)
$$
is a curve-velocity pair induced by $\alpha$.
\end{thm}

Here is the relation to random variables.

\begin{rem}[Random variables]
Consider two random variables $X_0,X_1$ with law $\mu_0,\mu_1\in\P_2(\R^d)$, $X_t=(1-t)X_0+tX_1$ and $\alpha=P_{X_0,X_1}$. Then $\alpha_t=P_{X_t,X_1}$ and $\alpha_t^x=P_{X_1|X_t=x}$, where $P_{X_1|X_t=x}$ is defined via $P_{X_t,X_1}=P_{X_1|X_t=x}\times_x P_{X_t}$. Thus \eqref{eq:speed_1} reads as 
\[
v_t(x)=\int_{\R^d}\frac{y-x}{1-t}\dd P_{X_1|X_t=x}(y).
\]
\hfill $\diamond$
\end{rem}

If the coupling $\alpha \in \Gamma(\mu_0,\mu_1)$ comes from a map $T$, we can characterize the induced velocity field  using the map as in the 
following corollary.

\begin{cor}\label{cor:cm}
Let $\mu_0,\mu_1\in \P_2(\R^d)$ and 
let $T:\R^d \to \R^d$ be an invertible measurable map such that
$T_\sharp \mu_0=\mu_1$. 
Consider the coupling $\alpha \coloneqq (\Id,T)_\sharp \mu_0$ and
the curve \eqref{curve_from_map} induced by $\alpha$.
Assume that $T_t(x)\coloneqq(1-t)x+tT(x) $ is invertible for $t\in[0,1]$
Then it holds that
$\alpha_t = (T_t, T)_\sharp\mu_0$
with  disintegrations
$$\alpha_t^x=\delta_{T(T_t^{-1}(x))} 
\quad \text{and} \quad
\alpha_t^y= \delta_{T_t(T^{-1}(y))} 
$$
and the $\alpha$-velocity field in \eqref{eq:speed_1} reads as
$
v_t(x) =
T(T_t^{-1}(x))-T_t^{-1}(x)
$. 
In other words, we have a linear velocity field
$
v_t(T_t(x)) =
T(x)-x
$.
\end{cor}

\begin{proof}
Then we see that
\begin{align}
\int_{\R^d \times \R^d} f(x,y) \, \dd \alpha_t
&= \int_{\R^d \times \R^d} f \left( e_t(x,y),y \right) \, \dd\alpha\\
&= \int_{\R^d}  f\left(e_t\left( (x,T(x) \right),T(x) \right) \,  \dd \mu_0\\
&= \int_{\R^d}  f\left( T_t(x),T(x) \right) \,  \dd \mu_0, \label{aha}
\end{align}
which yields $\alpha_t = (T_t, T)_\sharp\mu_0$ 
and by \eqref{desint_alpha} 
also $\mu_t = (T_t)_\sharp \mu_0$.
Further, this implies
\begin{align}
\int_{\R^d \times \R^d} f(x,y) \, \dd \alpha_t
&=
\int_{\R^d } f \left( T_t(x) ,T\left(T_t^{-1} T_t(x) \right)\right) \, \dd \mu_0\\
&=
\int_{\R^d \times \R^d} f \left(e_t(x,y),T(T_t^{-1}(e_t(x,y))) \right) \, \dd \alpha\\
&=
\int_{\R^d \times \R^d}  f \left(x,T(T_t^{-1}(x)) \right) \, \dd \alpha_t\\
&=
\int_{\R^d}  f \left(x,T(T_t^{-1}(x)) \right) \, \dd \pi^1_\sharp \alpha_t,
\end{align}
so that $\alpha_t^x=\delta_{T(T_t^{-1}(x))}$.
On the other hand, we obtain by \eqref{aha} that
\begin{align}
\int_{\R^d \times \R^d} f(x,y) \, \dd \alpha_t 
&=
\int_{\R^d } f((1-t)y+tT(y),T(y)) \, \dd  \mu_0\\
&=
\int_{\R^d} f((1-t)T^{-1}(T(y))+tT(y),T(y)) \, \dd  \mu_0\\
&=
\int_{\R^d \times \R^d} f((1-t)T^{-1}(y)+ty,y) \, \dd  \alpha\\
&=
\int_{\R^d \times \R^d} f((1-t)T^{-1}(y)+ty,y) \, \dd  \alpha_t\\
&=
\int_{\R^d \times \R^d} f((1-t)T^{-1}(y)+ty,y) \, \dd \pi^2_\sharp \alpha_t,
\end{align}
meaning that $\alpha_t^y= \delta_{T_t(T^{-1}(y))}$. 
Finally, we get by \eqref{eq:speed_1} that 
\begin{align}\label{eq:monge_v}
v_t(x)&=
\int_{\R^d} \frac{y-x}{1-t} \, \dd \delta_{T(T_t^{-1}(x))}
=\frac{T(T_t^{-1}(x))-x}{1-t}
=
T(T_t^{-1}(x))-T_t^{-1}(x).
\end{align}
\end{proof}

\begin{example}
In Example \ref{ex:conv},
we have seen that the curve $\nu_t$ is unfortunately not induced by the coupling 
$\N(0, I_d) \times \mu_1$. 
This would be only the case for $r=0$. But then  
$\mathcal K_1 (y, \cdot) = \delta_y$ has no longer a density.
However, also for $r \in (0,1)$, the curve $\nu_t$
is induced by another plan, namely
$$
\tilde \alpha \coloneqq \tilde e_\sharp(\N(0, I_d) \times \mu_1) \quad \text{with} \quad
\tilde e(x,y)\coloneqq(x,(1-r)x+y),
$$
since
\begin{align}\label{eq:nut_coupling}
\int_{\R^d}f(x) \, \dd\nu_t(x)&=\int_{\R^d}f(x) \, \dd\pi^1_\sharp(\K_t(y,\cdot)\times_y\mu_1(y))\\
&=\int_{\R^d\times\R^d}f(x) \, \dd\N(ty,(1-rt)^2)\dd\mu_1(y)\\
&=\int_{\R^d\times\R^d} f((1-rt)x+ty) \, \dd\N(0,1)(y)\dd\mu_1(y)\\
&=\int_{\R^d}f(x) \, \dd (e_{t}\circ \tilde e)_\sharp(\N(0,1)\times\mu_1)(x),
\end{align}
i.e., $\nu_t = e_{t,\sharp} (\tilde \alpha)$.
By Example \ref{ex:lip_ex_loss}
the vector field induced by $\tilde \alpha$ coincides with the vector field constructed in \eqref{eq:speed_2}.

However, considering the family 
	$\tilde \alpha _t = (e_t, \pi^2)_\sharp \tilde \alpha$, we have that $\pi^2_\sharp \tilde{\alpha}_t=\N(0,(1-r)^2\Id)*\mu_1$, which is not the construction of $\pi^2_\sharp\alpha_t=\mu_1$ from Example \ref{ex:conv}. In particular, the family of Markov kernels inducing a curve is not unique. \hfill $\diamond$
\end{example}

\section{Curves via Stochastic Processes}\label{subsec:conde}
This section resembles the results from the papers of  Liu et al.  \cite{liu2022rectified,liu2023flow} with our notation.
As already mentioned in the introduction, these authors use a stochastic process $(X_t)_t$ to determine an absolutely continuous curve $\mu_t$ and an associated 
vector field $v_t$ based on the conditional expectation
$$
\mu_t \coloneqq P_{X_t} \quad \text{and} \quad v_t (x)
\coloneqq \E[\partial_t X_t|X_t= x ],
$$
see Theorem \ref{thm.rv}. For the special process
$X_t \coloneqq (1-t) X_0 + t X_1$, this will again result in a curve-velocity pair
induced by a coupling, namely $\alpha = P_{X_0,X_1}$, see Corollary \ref{cor:plna-rv}.
We start by recalling the notation of conditional expectation in Subsection \ref{subsec:ce1} and use this in Subsection
\ref{subsec:ce2} to deduce appropriate curve-velocity pairs.

\subsection{Conditional Expectation}\label{subsec:ce1}
Throughout this section, let $(\Omega,\Sigma, \PP)$ be a probability space and $X,Y:\Omega\to \R^d$ be square integrable random variables, i.e.,
$\int_\Omega \|X\|^2 \, \dd \PP < \infty$. 
Furthermore, let 
$$\sigma(Y) \coloneqq \{Y^{-1}(A):A\in B(\R^d)\}$$ 
be the $\Sigma$-algebra generated by $Y$. 
The \emph{conditional expectation} $\E[X|Y]:\Omega\to \R^d$ is a square integrable random variable
on $(\Omega,\sigma(Y),\PP)$ characterized by 
\begin{equation} \label{condi}
\E_\PP[f\, \E[X|Y]]=\E_\PP[f X]
\end{equation}
for all $\sigma(Y)$-measurable, bounded functions $f:\Omega\to \R$.

Conditional expectations $\E[X|Y]$ of random vectors can also be interpreted as Borel measurable functions $\psi:\R^d\to\R^d$ via the Doob--Dynkin Lemma, see \cite[Lemma 1.14]{kallenberg1997foundations}.
\begin{prop}
[Doob--Dynkin Lemma]
Let $(\Omega, \Sigma,\PP)$ be a measure space and let $C$ and $D$ be metric spaces endowed with the Borel $\Sigma$-algebra. Let  $f:\Omega\to C,\, g:\Omega \to D$ be measurable functions. Then the following are equivalent:
\begin{itemize}
    \item[\rm{i)}] $f$ is $\sigma(g)$-measurable.
    \item[\rm{ii)}] There exists  Borel measurable function $\psi:D\to C$ such that $f=\psi\circ g$.
\end{itemize}
If $C=\R^d$, $D=\R^m$ and $f\in L^2(\Omega,\R^d,\PP)$, then we have  $\psi \in L^2(\R^m, \R^d,g_\sharp\PP)$.
\end{prop}

By the Doob-Dynkin Lemma, there exists a function $\psi\in L^2(\R^d,Y_\sharp \PP)$ such that  
$$\E[X|Y]=\psi \circ Y.$$ We also write $\E[X|Y=y]$ for $\psi(y)$. 
Strictly speaking, this notation only makes sense for $y\in Y(\Omega)$, 
but every measurable subset of $Y(\Omega)^C$ is a $Y_\sharp\PP$ zero set. 
By definition it holds that $\E[X|Y=\cdot]\circ Y=\E[X|Y]$.

We can express $\E[X|Y=y]$ also by disintegration of measures.

\begin{prop}\label{rem:cond_disintegration} 
Consider the disintegration
$$
\alpha \coloneqq P_{X,Y} = \alpha^y \times_y P_Y.
$$
Then it holds $\E[X|Y]=\psi \circ Y$ with
\begin{align}\label{eq:cond_exp_dis}
\psi(y) \coloneqq  \int_{\R^d} x \, \dd\alpha^y(x).
\end{align}
\end{prop}

\begin{proof}
We  check that $\E_\PP[f\psi\circ Y]=\E_\PP[fX]$ for every 
$\sigma(Y)$-measurable bounded function $f:\Omega \to \R$. 
By the Doob-Dynkin Lemma there exists a Borel measurable function $g:\R^d\to\R^d$ such that $f=g\circ Y$. Thus we can compute, changing the integral order by Fubini, 
\begin{align}
\E_\PP[f\psi\circ Y]
&=
\int_\Omega (g\circ Y)(\omega)\int_{\R^d} x \, \dd\alpha^{Y(\omega)}(x)\dd\PP(\omega)\\
&= \int_{\R^d \times \R^d} g(y)x \, \dd\alpha^y(x)  \dd (Y_\sharp \PP)(y)\\
&=\int_{\R^d \times \R^d} g(y)x \, \dd\alpha^y(x)\dd P_Y(y)
=\int_{\R^d \times \R^d} g(y)x \, \dd P_{X,Y}(x,y)\\
&=
\int_\Omega f(\omega) X(\omega) \, \dd\PP(\omega)
=\E_\PP[fX].
\end{align}
\end{proof}

\subsection{Curve via Conditional Expectation}\label{subsec:ce2}
A family of random variables $X_t:\Omega\to \R^d$, $t\in[0,1]$, is called a
\emph{stochastic process}.
We say that $X_t$ is \emph{continuously differentiable in} $t_0\in [0,1]$, 
if, for a.e. $\omega\in \Omega$, the map $t\mapsto X_t(\omega)$ is continuously differentiable in $t_0$, where we mean the continuous extension to the boundary if $t_0 \in \{0,1\}$.  
We denote the derivative with respect to $t$ by $\partial_t X_t$.

In \cite[Theorem 3.3]{liu2023flow} it was proven that for well-behaved stochastic processes $(X_t)_t$ the conditional expectation can be used to obtain a vector field $v_t$ such that $(X_{t,\sharp}\PP,v_t)$ fulfills the continuity equation. For convenience, we include the proof.

\begin{thm}\label{thm.rv}
Let  $X_t\in L^2(\Omega, \R^d,\PP)$ be continuously differentiable in every $t\in [0,1]$. Let 
\begin{equation}\label{eq:rv}
\mu_t\coloneqq X_{t,\sharp}\PP = P_{X_t}, 
\quad  \text{and} \quad
v_t \coloneqq \E[\partial_t X_t|X_t = \cdot].
\end{equation}
Assume  that $\mu_t$ is a narrowly continuous curve, 
$
v_t \in L^2(\R^d,\mu_t)
$
for 
$t \in [0,1]$
and
$\|v_t\|_{L^2(\R^d,\mu_t)}\in L^1([0,1])$. 
Then $\mu_t$ is an absolutely continuous curve in $\P_2(\R^d)$ and  $(\mu_t,v_t)$ fulfills the continuity equation.
\end{thm}

\begin{proof}
For $X_t\in L^2(\Omega,\R^d,\PP)$, we know that $\mu_t=X_{t,\sharp}\PP\in\P_2(\R^d)$. 
Since $\mu_t$ is narrowly continuous and $\|v_t\|_{L^2(\mu_t,\R^d)}\in L^1([0,1])$, it suffices by Theorem \ref{thm:abscont_ce} to show that $(\mu_t,v_t)$ satisfy the continuity equation \eqref{eq:ce}.
For $\varphi\in C_c^\infty((0,1)\times\R^d)$, we have 
\begin{align}
	\frac{\dd}{\dd t}(\varphi_t \circ X_t)
    =
    (\partial_t \varphi_t )\circ X_t + \langle \nabla_x\varphi_t \circ X_t,\partial_t X_t\rangle, 
\end{align}
so that
\begin{align}
	0&= 
    \int_{\Omega} (\varphi_1\circ X_1 - \varphi_0 \circ X_0) \, \dd \PP
    =\int_{\Omega}\int_0^1\frac{\dd}{\dd t}(\varphi_t \circ X_t) \, \dd t\dd \PP\\
&= 
\int_0^1 \int_{\Omega}(\partial_t \varphi_t)\circ X_t + \langle\nabla_x\varphi_t \circ X_t,\partial_t X_t\rangle \, \dd\PP\dd t\\
&= 
\int_0^1 \int_{\R^d}\partial_t \varphi_t \, \dd\mu_t\dd t +\int_0^1\E[\langle \nabla_x\varphi_t\circ X_t ,\partial_t X_t\rangle] \, \dd t\\
&
=\int_0^1 \int_{\R^d}\partial_t \varphi_t \, \dd\mu_t\dd t + \int_0^1\E[\langle \nabla_x\varphi_t\circ X_t,\E[\partial_tX_t|X_t]\rangle] \, \dd t\\
&=
\int_0^1 \int_{\R^d}\partial_t \varphi_t \, \dd\mu_t\dd t + \int_0^1\E[\langle \nabla_x\varphi_t\circ X_t,\E[\partial_t X_t|X_t=x]\circ X_t\rangle]\, \dd t\\
&=
\int_0^1 \int_{\R^n}\partial_t \varphi _t \, \dd\mu_t \, \dd t + \int_0^1\int_{\R^d}\langle \nabla_x\varphi_t,\E[\partial_t X_t|X_t=x]\rangle \, \dd\mu_t\dd t.
\end{align}
It remains to verify the measurability of $v$ which can be done similarly as in the proof of Proposition \ref{eq:velo}.  Set $X:[0,1]\times\Omega\to[0,1]\times\R^d$ with  $(t,\omega) \mapsto (t,X_t(\omega))$. 
Then $X$ is measurable, since it can be written as limit of measurable functions by approximating it in $t$ by step functions. 
Similarly, we have that  $d_t X: [0,1]\times\Omega \to \R^d$ given by $(t,\omega)\mapsto \partial_t X_t(\omega)$ is measurable. 
Thus, we can define a measurable function $v:[0,1]\times\R^d\to \R^d$ by
$(t,x) \mapsto \E[d_t X|X=(t,x)]$.
Then it holds
$v(t,\cdot)=\E[\partial_tX_t|X_t=\cdot]$ for a.e. $t\in[0,1]$ by the following reason: let $\{g_n\}_{n\in\NN}\subset C_b(\R^d)$ be a dense subset and $h\in C_b([0,1])$. Then we have
\begin{align}
\E_{(t,\omega)\sim\L_{[0,1]}\times\PP}[ \left( (hg_n)\circ X \right) \, 
\left( v \circ X \right)]
&=
\int_0^1 h(t)\E[\left( g_n\circ X_t \right) \, \left( v(t,\cdot)\circ X_t\right)]  \, \dd t
\end{align}
and
\begin{align}
\E_{(t,\omega)\sim\L_{[0,1]}\times\PP}[\left( (hg_n)\circ X \right) \, d_t X]
&=
\int_0^1 h(t)\E[\left( g_n\circ X_t \right) \,\partial_t X_t] \, \dd t
\end{align}
and the left hand sides are equal by definition of $v$. 
This means that
$$\E[ \left(g_n\circ X_t \right) \, \left(v(t,\cdot)\circ X_t \right)]
=
\E[ \left( g_n\circ X_t \right) \, \partial_t X_t]
$$
for a.e. $t\in[0,1]$ and all $n\in\NN$. 
Since $\{g_n\}_{n\in\NN}$ is dense in $C_b(\R^d)$ and $C_b(\R^d)$ is dense in $L^2(\R^d,X_{t,\sharp}\PP)$, which contains the bounded measurable functions, we can conclude that  
$v(t,\cdot)=\E[\partial_t X_t|X_t=\cdot]$ for a.e. $t\in[0,1]$.
\end{proof}

The following special case of Theorem \ref{prop:alpha_curve}
gives a relation to curve-velocity pairs induced by couplings.

\begin{cor} \label{cor:plna-rv}
For the special stochastic process 
$X_t=(1-t)X_0+tX_1$, $t \in [0,1]$
and the plan 
$\alpha = P_{X_0,X_1}$,
the curve-velocity field \eqref{eq:rv}
coincides with those induced by $\alpha$ in \eqref{meanth}.
\end{cor}

\begin{proof}
Let $(\mu_t,v_t)$ be given by \eqref{eq:rv}.
By \eqref{eq:cond_exp_dis}, we can rewrite 
$$v_t(x) = \E[\partial_t X_t|X_t=x]
=
\int_{\R^d}  z \, \dd P_{\partial_t X_t|X_t=x}.
$$
Since $X_t=(1-t)X_0+tX_1$, $t \in [0,1]$, we obtain $\partial_t X_t=X_1-X_0$ 
and further, for any $f \in  C_b(\R^d)$, that
\begin{align}
\int_{\R^d} f v_t \, \dd\mu_t
&=
\int_{\R^d} f(x) \int_{\R^d} z \, \dd P_{\partial_t X_t|X_t=x} \dd P_{X_t}
=
\int_{\R^d \times \R^d} f(x)z \, \dd P_{\partial_t X_t,X_t}
\\
&= 
\int_{\R^d \times \R^d} f(x)z \, \dd P_{X_1-X_0,X_t}
=
\int_{\R^d \times \R^d}  f(e_t(x,y))(y-x) \, \dd P_{X_0,X_1}
\\
&=\int_{\R^d} 
f\, \dd e_{t,\sharp}\left((y-x)\alpha\right).
\end{align}
\end{proof}

\begin{rem}\label{rem:flow_score}
    Let $X_t=(1-t)X_0 + t X_1$ for independent random variables $X_0,X_1$, where $X_0\sim \N(0,I_d)$. Then in \cite[Appendix A1]{zhang2024flow} it is shown that the formula for the vector field associated with $\mu_t$,
    $$v_t = \frac{1-t}{t}\nabla_x\log p_t(x)+ \frac{x}{t}, \quad  t \in (0,1)
$$
from Proposition \ref{prop:flow_score} can also be derived via the velocity field from Theorem \ref{thm.rv} and Tweedie's formula. Recall that Tweedie's formula \cite[Lemma 3.2]{daras2024consistent} for $X_t$ reads as
\begin{align}
\nabla_x  \log p_t(x)=\frac{t\E[X_1|X_t=x]-x}{(1-t)^2},
\end{align}
and thus $\E[X_1|X_t=x]=\frac{(1-t)^2}{t}\nabla_x\log p_t(x)+\frac{x}{t}$. Using $X_0=\frac{X_t-tX_1}{1-t}$ for $t\in(0,1)$ we obtain
\begin{align}
v_t(x)&=\E[\partial_t X_t|X_t=x]=\E[X_1-X_0|X_t=x]\\
&=\E\left[\frac{X_1}{1-t}-\frac{X_t}{1-t}\Big|X_t=x\right]\\
&=\frac{1-t}{t}\nabla_x\log p_t(x)+\frac{x}{t(1-t)}-\frac{x}{1-t}\\
&=\frac{1-t}{t}\nabla_x\log p_t(x)+\frac{x}{t},
\end{align}
where we used Tweedie's formula and the fact that $\E[X_t|X_t]=X_t$. \hfill $\diamond$
\end{rem}

\section{Flow Matching} \label{sec:fm}

In this section, we will see how we can approximate velocity fields of certain absolutely continuous curves $\mu_t: [0,1] \to \P(\R^d)$
by neural networks. 
If the velocity field is known and the starting measure 
$\mu_0$ is such that we can easily sample from, like the standard Gaussian one, then we can use the  ODE \eqref{eq:flow_ode}
to produce samples from $\mu_t$. 

We consider the velocity fields $v_t$ associated to the absolutely 
continuous curves $\nu_t$ from Proposition \ref{prop:lip_flow}. 
By the next proposition, these velocity fields can be obtained as minimizers of a certain loss function. 

\begin{prop} \label{thm:FM_lip}
Let $\mu_1 \in \P_2(\R^d)$, and let
$\K_t(y,\cdot): [0,1] \to \P_2(\R^d)$ be absolutely continuous curves 
with associated velocity fields $v_t^y$ for $\mu_1$-a.e. $y \in \R^d$.
Further, assume that  $(t,x,y)\mapsto v_t^y$ is measurable and 
$\|v_t^y \|_{L^2(\R^d \times \R^d,\alpha_t)} \in L^2([0,1])$.
For $\alpha_t\coloneqq \K_t(y, \cdot) \times_y \mu_1$,
we consider the absolutely continuous curve
$
\nu_t\coloneqq \pi^1_\sharp\alpha_t.
$
Then its associated velocity field $v: [0,1] \times \R^d \to \R^d$
defined by \eqref{eq:speed_2} fulfills
\begin{align}
    v=
    \argmin_{u\in L^2(\R^d,\nu_t\times_t\L_{(0,1)})}\E_{(t,x,y)\sim \alpha_t\times_t\L_{(0,1)}}\left[\|u_t(x)-v^y_t(x)\|^2\right]. 
\end{align}
\end{prop}

\begin{proof}
Starting with
\begin{align}
    \|u_t(x)-v_t^y(x)\|^2 = \|u_t(x)\|^2 -2\big\langle u_t(x),v^y_t(x)\big\rangle + \|v^y_t(x)\|^2,
\end{align}
we deal with each part individually.
It holds
\begin{align}
\E_{(t,x,y)\sim \alpha_t\times_t\L_{(0,1)}}\left[\|u_t(x)\|^2\right]
&=
\E_{t\sim\L_{(0,1)}}\E_{ x\sim\nu_t}\left[\E_{y\sim \alpha_t^x}\left[\|u_t(x)\|^2\right]\right]\\
&= \E_{(t,x)\sim \nu_t\times_t \L_{(0,1)}}\left[\|u_t(x)\|^2\right].
\end{align}
For the second term,  we obtain with Fubini's theorem 
\begin{align}
\E_{(t,x,y)\sim\alpha_t\times_t\L_{(0,1)}}\big[\big\langle u_t(x),v_t^y(x)\big\rangle\big]
&= 
\E_{(t,x)\sim\nu_t\times_t\L_{(0,1)}}\big[\int_{\R^d} \big\langle u_t(x),v_t^y(x)\big\rangle \, \dd \alpha_t^{x}(y)\big]\\
&= 
\E_{(t,x)\sim\nu_t\times_t\L_{(0,1)}}
\big[\big\langle u_t(x),\int_{\R^d} v_t^y(x) \, \dd \alpha_t^{x}(y)\big\rangle\big]
\\
&=\E_{(t,x)\sim\nu_t\times_t\L_{(0,1)}}\big[\big\langle u_t(x),v_t^y(x)\big\rangle\big].
\end{align}
Putting everything together, we get
\begin{align}
&\E_{(t,x,y)\sim \alpha_t\times_t\L_{(0,1)}}\left[\|u_t(x)-v_t^y(x)\|^2\right]=\E_{(t,x)\sim\nu_t\times_t\L_{(0,1)}}\left[\|u_t(x)-v_t(x)\|^2\right]\\
&-\E_{(t,x)\sim\nu_t\times_t\L_{(0,1)}}\left[\|v_t(x)\|^2\right]+\E_{(t,x,y)\sim\alpha_t\times_t\L_{(0,1)}}\left[\|v^y_t(x)\|^2\right].
\end{align}
Since the last two terms do not depend on $u$ and
the minimizer of the first term is $v$, we obtain the assertion.
\end{proof}

In order to learn $v$, the  proposition suggests to use the loss
\begin{align}\label{eq:markov_loss_fin}
\text{FM}_{\alpha_t}(\theta) \coloneqq \E_{(t,x,y)\sim \alpha_t\times_t\L_{(0,1)}}[\|u^\theta_t(x)-v^y_t(x)\|^2].
\end{align}
The name ''FM'' comes from ''flow matching''.
In the original paper \cite{lipman2023flow}, the notation ''conditional flow matching'' was used which refers to a ''conditional flow path'', see Remark \ref{rem:orig}. 
In practice, we work with the empirical expectation.
To this end, we will need to be able to sample from $\alpha_t=\K_t(y,\cdot) \times_y\mu_1$ and to compute $v_t^y(x)$. A typical scenario is 
that samples $y_i$, $i=1,\ldots,n$ from the data distribution $\mu_1$ are available
and that it is easy to sample from the
 Markov kernel $\K_t(y_i,\cdot)$, $i=1,\ldots,n$. 

 \begin{example}\label{ex:lip_ex_loss}
A setting, where $\text{FM}_{\alpha_t}$ is tractable is Example \ref{ex:conv}.
Here $\mu_0 = \N(0,I_d)$ and
$$\K_t(y,\cdot)=\N(ty,(1-rt)^2I_d) \quad \text{and} \quad v_t^y=\frac{y-rx}{1-tr},$$ and we obtain 
\begin{align}
\text{FM}_{\alpha_t}(\theta)
&=
\E_{(t,x,y)\sim \alpha_t\times_t\L_{(0,1)}}\big[\|u_t^\theta(x)-v_t^y(x)\|^2\big]\\
&=
\E_{t\sim[0,1]}\E_{y\sim\mu_1}\E_{x\sim \N(ty,(1-rt)^2 I_d)}\big[\|u_t^\theta(x)-\frac{y-rx}{1-tr}\|^2 \big]\\
&=
\E_{t\sim[0,1]}\E_{y\sim\mu_1}\E_{x\sim \N(0, I_d)}\big[\|u_t^{ \theta}((1-tr)x+ty)-(y-rx)\|^2 \big]\\
&=
\E_{t\sim[0,1],  (x,y)\sim \mu_0\times \mu_1} \big[\|u_t^{ \theta}((1-tr)x+ty)-(y-rx)\|^2\big].
\end{align}
\end{example}

For $\alpha_t$ coming from a coupling $\alpha\in \Gamma(\mu_0,\mu_1)$
as in \eqref{push_a_t} with disintegration kernel 
$\K_t(y,\cdot)$, the associated velocity field $v_t^y$ has the simple form \eqref{yy}.
In this important case, the loss function \eqref{eq:markov_loss_fin} can be simplified.

\begin{prop}\label{prop:v_as_min}
Let $\alpha\in \Gamma(\mu_0,\mu_1)$ and $\mu_t=e_{t,\sharp}\alpha$. Then the velocity field $v: [0,1] \times \R^d \to \R^d$ induced by 
$\alpha$ fulfills
 \begin{align}\label{v_min_plan}
 v = \argmin_{u\in L^2(\R^d,\mu_t\times_t\L_{(0,1)})}\E_{(t,x,y)\sim \L_{(0,1)}\times \alpha}\left[\|u_t\left(e_t(x,y) \right)-(y-x)\|^2\right] .
 \end{align}
\end{prop}

\begin{proof}
By Proposition \ref{prop:desint}, we know that 
$\K_t(y, \cdot)$ given via  
$$\alpha_t \coloneqq (e_t,\pi^2)_\sharp\alpha =\K_t(y,\cdot) \times_y\mu_1$$ 
is an absolutely continuous curve for 
$\mu_1$-a.e. $y \in \R^d$
with associated simple velocity field $v_t^y(x)\coloneqq \frac{y-x}{1-t}$, $t\in[0,1)$. Further, $\mu_t = \nu_t \coloneqq \pi^1_\sharp\alpha_t$ and the velocity field induced by  $\alpha$ in \eqref {eq:speed_1} coincides with those in \eqref{eq:speed_2}. 
Thus, by Proposition \ref{thm:FM_lip}, we can calculate
\begin{align}
    v&=\argmin_{u\in L^2(\R^d,\mu_t\times_t\L_{(0,1)})}\E_{(t,x,y)\sim \alpha_t\times_t\L_{(0,1)}}\left[\|u_t(x)-v_t^y(x)\|^2\right]\\
    &=\argmin_{u\in L^2(\R^d,\mu_t\times_t\L_{(0,1)})}\E_{(t,x,y)\sim \alpha_t\times_t\L_{(0,1)}}\left[\left\|u_t(x)-\frac{y-x}{1-t}\right\|^2\right]\\
    &=\argmin_{u\in L^2(\R^d,\mu_t\times_t\L_{(0,1)})}\E_{(t,x,y)\sim \alpha\times\L_{(0,1)}}\left[\|u_t(e_t(x,y))-(y-x)\|^2\right].
\end{align}
\end{proof}

Thus, for curves and velocity fields induced by plans $\alpha\in\Gamma(\mu_0,\mu_1)$ we can use the simplified flow matching objective 
    \begin{align}\label{eq:markov_loss}
\text{FM}_{\alpha}(\theta)=\E_{t\sim[0,1],(x,y)\sim \alpha}[\|u_t^\theta(e_t(x,y))-(y-x)\|^2].
\end{align}
Computing the empirical expectation  requires only samples from the coupling 
$\alpha\in \Gamma(\mu_0,\mu_1)$. 
Two canonical choices of couplings, which we can usually sample from, are
\begin{itemize}
\item the independent coupling 
$\alpha = \mu_0 \times \mu_1$, which requires only samples of both $\mu_0$ and $\mu_1$, and
\item an optimal coupling $\alpha \in \Gamma_o(\mu_0,\mu_1)$, 
which is usually approximated by applying an optimal transport solver on empirical measures approximating $\mu_i$, $i=0,1$.
\end{itemize}

Finally, for a curve arising from a stochastic process 
also the description of the velocity field 
by conditional expectation in \eqref{eq:rv},
$$v_t \coloneqq \E[\partial_t X_t|X_t = \cdot]$$ 
can be interpreted as a minimization problem. 
Namely, it is well-known, see, e.g., \cite[10.1.5 (5)]{bogachev2007measure}, 
that 
\[\E[X|Y=\cdot]=\argmin_{f\in L^2(\R^d,Y_\sharp \PP)}\E_{\PP}[\|X-f(Y)\|^2].\]
Thus, we obtain 
\begin{align}
	v&=\argmin_{u\in L^2(\R^d,\mu_t\times_t \L_{(0,1)})}\E_{\L_{(0,1)}\times\PP}[\|\partial_t X_t-u_t(X_t)\|^2].
\end{align}    
For the special stochastic process $X_t=(1-t)X_0+tX_1$ 
in Corollary \ref{cor:plna-rv},
leading to curves induced by plans, this becomes
\begin{align}
	v
	 &=\argmin_{u\in L^2(\R^d,\mu_t\times_t \L_{(0,1)})}\E_{
\L_{(0,1)}\times\PP}[\|X_1-X_0-u_t(X_t)\|^2]\\
&=\argmin_{u\in L^2(\R^d,\mu_t\times_t \L_{(0,1)})}\E_{(t,x,y)\sim \L_{(0,1)}\times P_{X_0,X_1}}[\|y-x-u_t(e_t(x,y))\|^2]
\end{align}
and we are back at Proposition \ref{v_min_plan}.

\section{Numerical Examples for Flow Matching}
\subsection{Flow Matching Algorithms}
Sampling from a product plan $\alpha=\mu_0\times\mu_1$ can be reduced to sampling from the individual measures $z_i\sim \mu_0,\, x_i\sim \mu_1$ and simply building the product $(z_i,x_i)\sim \alpha$. This is the reason that training a flow matching model from a product plan is particularly easy, as can be seen in Algorithm \ref{alg:flow_product}. 

Given two measures $\mu_0,\mu_1\in\P_2(\R^d)$, sampling from an optimal transport plan $\alpha\in \Gamma_o(\mu_0,\mu_1)$ is usually not possible. Often both measures are not known analytically but only an approximation by samples, so that an optimal plan cannot be obtained. But even the plan for the approximation is often not computable since usually the number of samples is so high that solving the OT problem is computationally too expensive. In this case a popular method is minibatch OT flow matching \cite{tong2023improving}. The idea is to partition the samples $x_1,\ldots,x_{N_t}\sim \mu_1$ into minibatches (smaller subsets), on which solving the OT problem is computational feasible. Details are described in Algorithm \ref{alg:minibatch_ot}.

Once a flow vector field $u^\theta$ is learned, sampling is then done by sampling $z_i\sim \mu_0$ from the source distribution and using an ODE solver for $f'(t)=u_t^\theta(f(t)),\, f(0)=z_i$ to compute $f(1)$.
\begin{algorithm}[H]
\begin{algorithmic}
\State $N_b$: batch size, $N_e$: number of epochs, $N_t$ number of samples from the target distribution
\State $x_1,\ldots,x_{N_t}$ samples from the target distribution $\mu_1$
\State \textbf{Network} $u^\theta:[0,1]\times \R^d\to \R^d$ depending on a parameter $\theta\in \R^m$. 
\For{$e=1,...,N_e$}
\For{$s=1,\ldots,\frac{N_t}{N_b}$}
\State Draw $(z_i)\sim \mu_0$ for $1\leq i\leq N_b$
\State Draw $t_i\sim \L_{(0,1)}$ for $1\leq i\leq N_b$
\State Compute $x_{t_i}^i=(1-t_i)z_i+t_ix_{i+(s-1)N_b}$
\State Compute $\L(\theta)=\frac{1}{N_b}\sum_{i=1}^{N_b}\left\|u^\theta(t,x_{t_i}^i)-(x_{i+(s-1)N_b}-z_i)\right\|^2$ and $\nabla_\theta \L(\theta)$
\State Update $\theta$ using $\nabla_\theta\L(\theta)$
\EndFor
\EndFor
\State \textbf{Return:} $\theta$
\end{algorithmic}
\caption{Learning $u^\theta$ from the plan $\alpha=\mu_0\times\mu_1\in \Gamma(\mu_0,\mu_1)$}
\label{alg:flow_product}
\end{algorithm}

\begin{algorithm}[H]
\begin{algorithmic}
\State $N_b$: batch size, $N_{bOT}$ OT batch size, $N_e$: number of epochs, $N_t$ number of samples from the target distribution
\State $x_1,\ldots,x_{N_t}$ samples from the target distribution $\mu_1$
\State \textbf{Network} $u^\theta:[0,1]\times \R^d\to \R^d$ depending on a parameter $\theta\in \R^m$. 
\For{$e=1,...,N_e$}
\For{$o=1,\ldots,\frac{N_t}{N_{bOT}}$}
\State Draw $z_i\sim \mu_0$ for $1\leq i\leq N_{bOT}$
\State Compute optimal transport map for $$W_2\left(\frac{1}{N_{bOT}}\sum_{i=1}^{N_{bOT}}\delta_{z_i},\frac{1}{N_{bOT}}\sum_{i=1}^{N_{bOT}}\delta_{x_{i+(o-1)N_{bOT}}}\right)$$
\State described by a function $T_o:\{1,\ldots,N_{bOT}\}\to \{1+(o-1)N_{bOT},\ldots,oN_{bOT}\}$
\For{$s=1,\ldots,\frac{N_{bOT}}{N_b}$}
\State Draw $t_i\sim \L_{(0,1)}$ for $1\leq i\leq N_b$
\State Compute $x_{t_i}^i=(1-t_i)z_{i}+t_ix_{T_o(i)}$
\State Compute $\L(\theta)=\frac{1}{N_b}\sum_{i=1}^{N_b}\left\|u^\theta(t_i,x_{t_i}^i)-(x_{T_o(i)}-z_{i})\right\|^2$ and $\nabla_\theta \L(\theta)$
\State Update $\theta$ using $\nabla_\theta\L(\theta)$
\EndFor
\EndFor
\EndFor
\State \textbf{Return:} $\theta$
\end{algorithmic}
\caption{Learning $u^\theta$ from minibatch OT flow matching for $\mu_0$ the source distribution and $\mu_1$ the target distribution.}
\label{alg:minibatch_ot}
\end{algorithm}

\subsection{Numerical Examples}
For our first example we chose standard $2d$ Gaussian distribution as source distribution and a Gaussian mixture model with eight equality weighted modes as target distribution, see\ref{fig:gaussian_mix}. In the middle, we can see the trajectories created by a vector field learned with minibatch OT flow matching, using a simple Euler method for solving the flow ODE. Here the trajectories are mostly straight, making sampling easier since the sampling quality does not depend as much on time discretization or the ODE solving method. On the left, the trajectories for a vector field learned via flow matching for a product plan is displayed, where again an Euler method was used to solve the flow ODE. Here the trajectories are much less straight and sample quality depends very much on the ODE solving quality. On the right we see the result for Neural ODEs, described in Section \ref{sec:CNF}, which has trajectories differing from both of the flow matching approaches.

\begin{figure}[ht]
\centering
\begin{subfigure}{0.32\linewidth}
    \includegraphics[width=\linewidth]{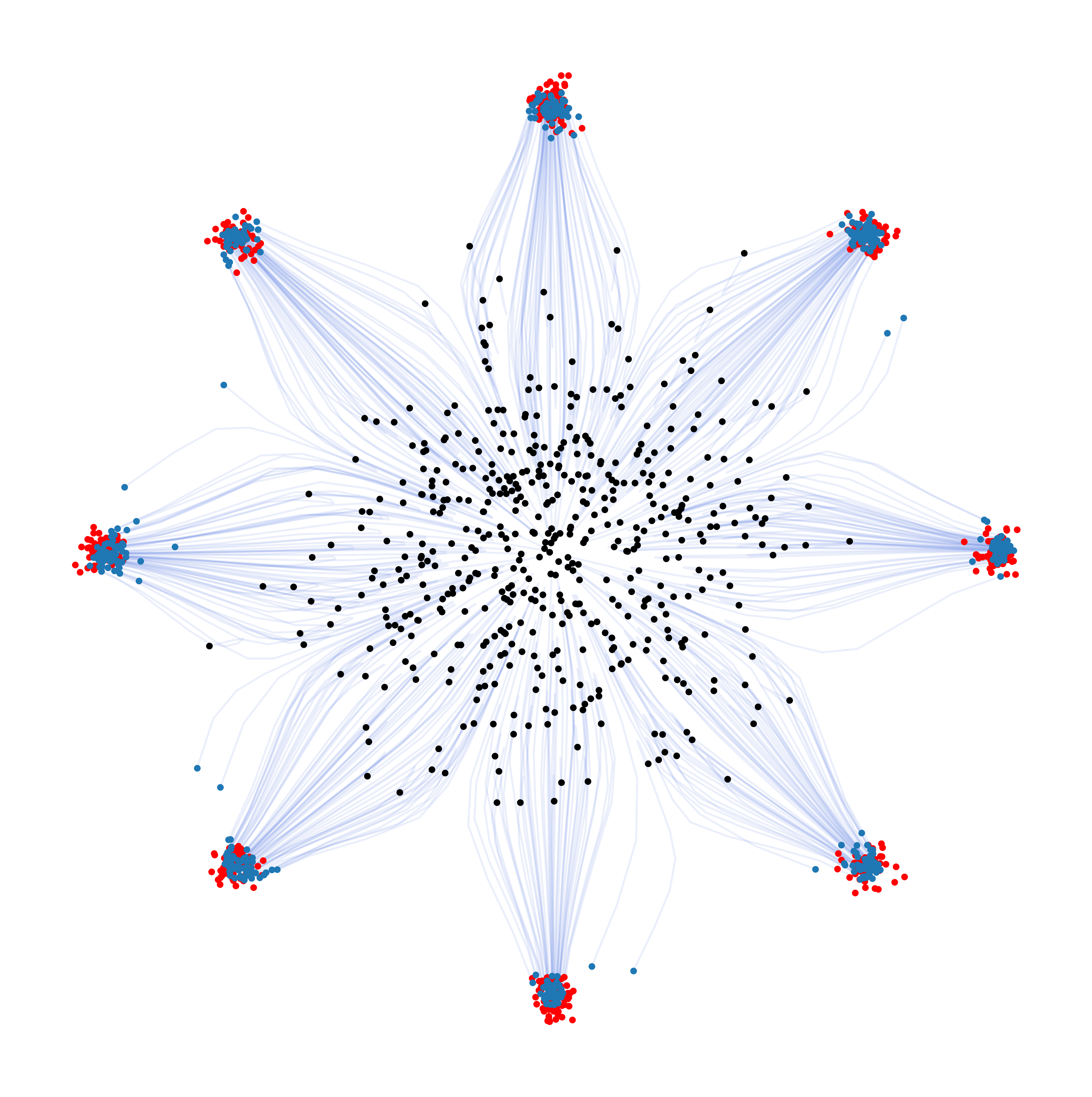}
\end{subfigure}
\begin{subfigure}{0.32\linewidth}
    \includegraphics[width=\linewidth]{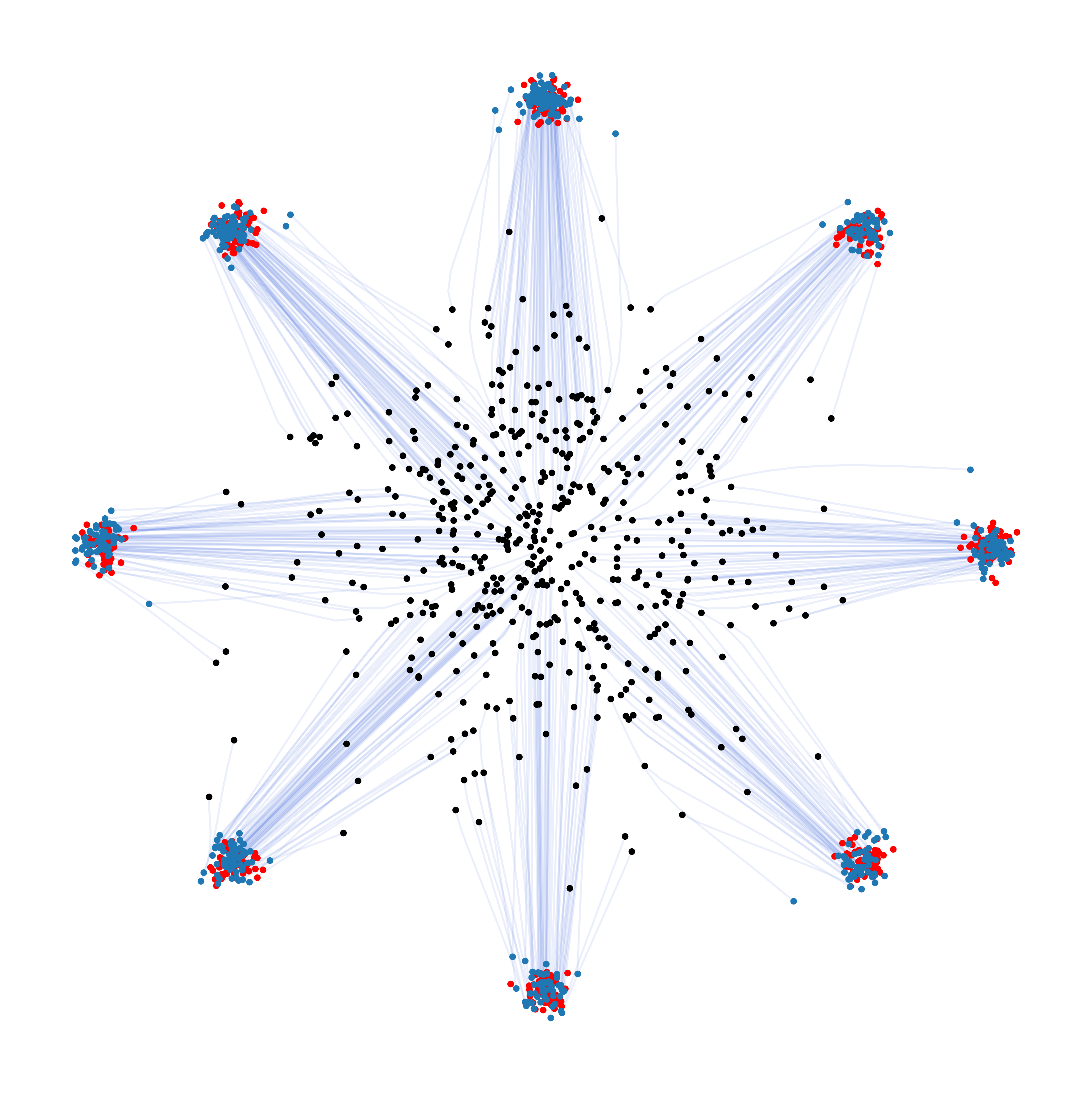}
\end{subfigure}
\begin{subfigure}{0.32\linewidth}
    \includegraphics[width=\linewidth]{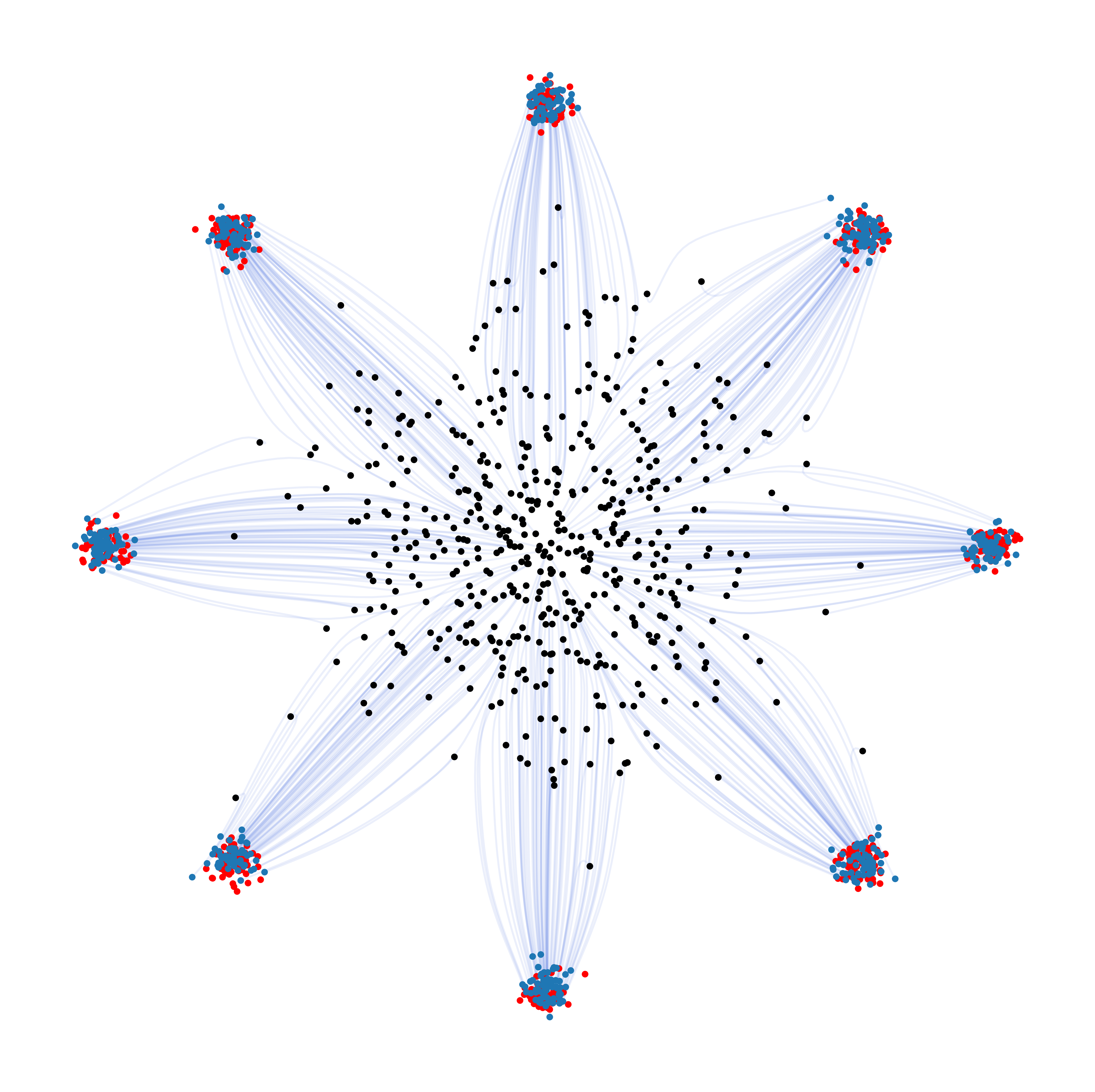}
\end{subfigure}

    \caption{Trajectories of points from a vector field $u^\theta$ obtained via flow matching from a coupling $\alpha\in\Gamma(\mu_0,\mu_1)$. We chose $\mu_0=\N(0,1)$ and as target distribution $\mu_1$ a Gaussian mixture model with $8$ modes. Black: points $\{z_i\}_{i=1}^{300}$ drawn from the source distribution $\mu_0=\N(0,1)$. Blue: points sampled via the vector field from $\{z_i\}_{i=1}^{300}$. Red: points sampled from the target distribution. Blue lines: trajectories of the vector field. The left image shows the results for $u^\theta$ trained using the independent coupling and the middle image shows the result when trained with minibatch OT flow matching. The right image shows the result from training a continuous normalizing flow as in Section \ref{sec:CNF}.}
    \label{fig:gaussian_mix}
\end{figure}

Finally, we show a high dimensional application of flow matching in Figure \ref{fig:cat}. Here we used the flow vector field from \cite{martin2024pnp}, which was already trained via minibatch OT flow matching on images of cat faces. For a comparison of the performance for the flow matching with different plans see \cite[Table 2]{tong2023improving}. 

\begin{figure}[ht]
\includegraphics[width=\linewidth]{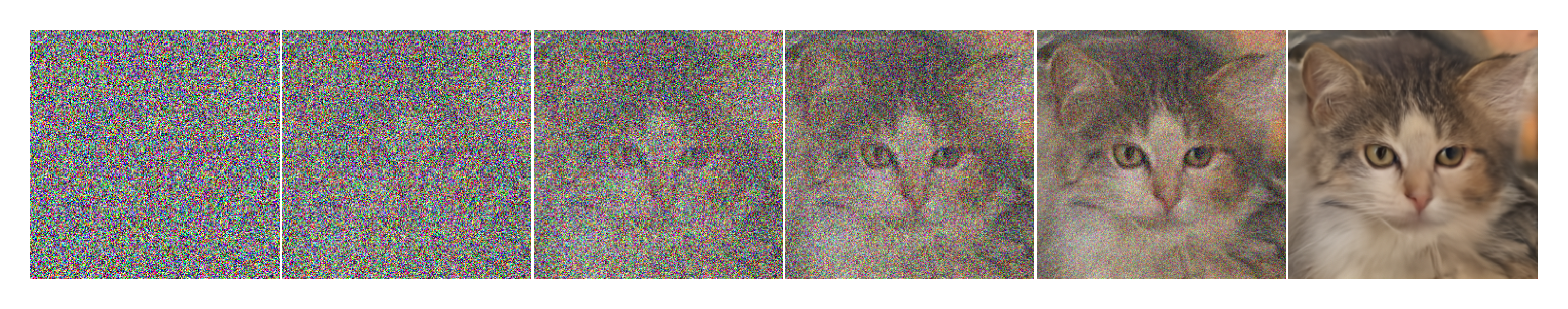}
\caption{Single trajectory for a flow matching model trained on cat images. For the trained model, see \cite{martin2024pnp}.
Note that for the standard Gaussian distribution we can sample in each direction
separately, so that the left starting sample contains the one-dimensional coordinates
in each of the $d= 256^2$ directions for the red-green-blue channels.}
\label{fig:cat}
\end{figure}
\section{Flow Matching in Bayesian Inverse Problems} \label{sec:fm_inverse_problems}

In this section, we are interested in Bayesian inverse problems 
\begin{align}
W=F(X)+\Xi
\end{align}
for random variables $W:\Omega\to\R^m$, $X:\Omega\to\R^d$, a forward operator $F:\R^d\to\R^m$ and a random variable $\Xi:\Omega\to\R^m$
which is responsible for noisy outputs. 
The aim  is to approximate the posterior distribution 
$P_{X|W=w}$. 
We propose to learn a map 
$$
g:\R^m\times\R^d \to \R^m\times \R^d,\quad (w,x) \mapsto (w,g_w(x))
$$ 
such that
$$g_\sharp(P_W\times \N(0,1))=P_{W,X}.$$  
This implies $g_{w,\sharp} \N(0,1) = P_{X|W=w}$, since we have for any $f \in  C_b(\R^m\times \R^d)$ that
\begin{align}
\int\limits_{\R^m \times \R^d} f \, \dd P_{W,X}
& = 
\int\limits_{\R^m \times \R^d} f \, \dd g_\sharp(P_W\times\N(0,1))
=
\int\limits_{\R^m} \int\limits_{\R^d} f(w,x) \, \dd (g_{w,\sharp}\N(0,1))\dd P_W.
\end{align}
In particular, we can sample from $P_{X|W=w}$ by sampling from $\N(0,1)$ and applying $g_w$. 
Again, we want to describe the generator $g$  via a curve in a space of probability measures. However, now we need a curve, where the $W$-marginal does not change, i.e., 
a curve in the space 
$$\P_{P_W} 
\coloneqq 
\{\mu\in\P_2(\R^m\times\R^d):\pi^1_\sharp\mu=P_W\}.$$ 
In the next subsection, we will equip the space 
$\P_{P_W}$ with a metric turning it into a complete metric space and determine couplings such that the corresponding curve stays in $\P_{P_W}$. 
Then we will use these curves for flow matching in order to learn conditional generating networks that approximate the posterior distribution $P_{X|W=w}$ for $w$ sampled from $P_W$. This section is based on our paper
\cite{chemseddine2024conditional}, see also  \cite{kerrigan2024dynamic}. 

\subsection{Conditional Wasserstein Distance} \label{sec:condW}
Conditional Wasserstein distances appear already in \cite[Section 2.4]{kloeckner2021extensions},
\cite{peszek2023heterogeneous}
as well as in the PhD thesis of Gigli \cite[Chapter 4]{gig}, who studied a metric on the so-called ''geometric tangent space'' of the Wasserstein-$2$ space. 
Recently, conditional Wasserstein distances were studied in the realm of machine learning in \cite{barboni2024understanding} as well as in functional optimal transport in \cite{hosseini2024conditional}. 
We will focus on its application in Bayesian flow matching as presented in our paper \cite{chemseddine2024conditional}, where also the
proofs of the following theorems and propositions are provided.

To fix the notation, let $\u=(w,x)\in \R^m\times \R^d$ 
and denote the projections by
$$\pi^1(\u) \coloneqq w, \quad \pi^2(\u) \coloneqq x.$$
For an arbitrary,  fixed reference measure $\eta\in\P_2(\R^m)$,  we consider the space 
\begin{align}
\P_\eta(\R^m\times \R^d):=\{\mu\in \P_2(\R^m\times\R^d):\pi^1_\sharp\mu=\eta\}.
\end{align}
In order to define a conditional Wasserstein distance on 
$\P_\eta$, we introduce, for $\mu,\nu\in \P_\eta(\R^m\times\R^d)$, 
the set of 4-plans
\begin{equation}\label{y-coupling}
  \Gamma_\eta =  \Gamma_\eta(\mu,\nu) \coloneqq \Big\{ \alpha \in \Gamma(\mu,\nu): \pi^{1,1}_{\sharp}\alpha = \Delta_{\sharp}\eta \Big\},
\end{equation}
 where 
 $$\pi^{1,1}(\u_1,\u_2) \coloneqq(\pi^1(\u_1),\pi^1(\u_2))
 = 
 \left( \pi^1 \left((w_1,x_1) \right) ,\pi^1\left( (w_2,x_2) \right) \right) =
 (w_1,w_2)
 $$ 
 and 
 $\Delta:\R^m \to \R^m\times \R^m$, $ w \mapsto (w,w)$ is the diagonal map. 
 Note that 
 $$\Delta^{-1} (w_1,w_2) = \emptyset \text{ if } w_1 \not = w_2
 \quad \text{and} \quad
 \Delta^{-1} (w_1,w_2) = w \text{ if }w_1 = w_2 = w.
 $$
 Now we define the \emph{conditional Wasserstein(-$2$) distance} by
        \begin{align}   \label{def:cond_wass}
        W_{2,\eta}(\mu,\nu) \coloneqq  \Big( 
        \min_{\alpha \in \Gamma_\eta(\mu,\nu)} \int_{(\R^m\times \R^d)^2}\|\u_1-\u_2\|^2 \, \dd \alpha 
        \Big)^{\frac 1 2}.
        \end{align}   
Indeed, the above minimum is obtained and the set of optimal couplings is denoted by $\Gamma_{\eta,o}$.       
The following theorem collects important properties of the conditional Wasserstein distance.

\begin{thm}\label{prop:cond_wass}
The conditional Wasserstein space 
$\left(\P_\eta(\R^m \times \R^d), W_{2,\eta} \right)$ is a complete metric space.
The conditional Wasserstein distance has the following properties:
\begin{itemize}
     \item[\rm{i)}] 
     For $\mu,\nu \in \P_\eta(\R^m \times \R^d)$,
     it holds
     $$ W_{2,\eta}^2(\mu,\nu)=\E_{w\sim\eta}\big[ W_2^2(\mu^w,\nu^w)\big],$$ 
     where $\mu^w,\nu^w$ denote the disintegrations of $\mu,\nu$ with respect to $\eta$, respectively.    
\item[\rm{ii)}] 
There exist optimal plans 
$\alpha_w \in \Gamma_o(\mu^w,\nu^w)$, $w \in \R^m$
such that 
 \begin{equation} \label{def_alpha}
 \alpha \coloneqq \underbrace{(\delta_w\times \alpha_w) }_{\in \P_2(\R^m \times \R^d \times \R^d)} \times_w\eta \in \Gamma_{\eta,o}(\mu,\nu).
 \end{equation}
 \end{itemize}
\end{thm}

Note that ii) means, for any $f \in C_b(\R^d)$, that
\begin{align}
\int\limits_{(\R^m \times \R^d)^2} f \, \dd \alpha
&= 
\int\limits_{\R^m} \int\limits_{\R^m \times \R^d \times \R^d} f(w_1,x_1,w_2,x_2) \,
\dd \delta_{w_{1}} (w_2) \dd \alpha_{w_1}(x_1,x_2)
\dd \eta(w_1)\\
&=\int\limits_{\R^m} \int\limits_{\R^d \times \R^d} f(w,x_1,w,x_2) \,
\dd \alpha_{w} (x_1,x_2)
\dd \eta(w).
\end{align}

\begin{figure}
\begin{minipage}{0.3\textwidth} 
        \centering
\scalebox{0.7}{
 \begin{tikzpicture}[>={Stealth[length=8pt,angle'=28,round]}]
    \fill [blue] (-0.1,-0.1) rectangle (0.1,0.1);
    \node[below left=14pt of {(0,0)}, outer sep=2pt,fill=white] {$\delta_{0,0}$};
    \fill [red] (0,5) circle [radius=4pt];
    \node[above left=14pt of {(0,5)}, outer sep=2pt,fill=white] {$\delta_{0,n}$};
    \fill [red] (2,0) circle [radius=4pt];
    \node[below right=14pt of {(2,0)}, outer sep=2pt,fill=white] {$\delta_{1,0}$};
    \fill [blue] (2-0.1,5-0.1) rectangle (2+0.1,5+0.1);
    \node[above right=14pt of {(2,5)}, outer sep=2pt,fill=white] {$\delta_{1,n}$};
    \draw [<-,green] (0.2,0) --(1.8,0);
    \draw [->,green] (0.2,5) -- (1.8,5);

    \draw [<-, brown] (0,0.2) -- (0,5);
    \draw [->, brown] (2,0) -- (2,4.8);
\end{tikzpicture}}
\end{minipage}
\begin{minipage}{0.7\textwidth}
\caption{Consider $\eta=\frac12\delta_{0}+\frac12\delta_1$, ${\color{red}\mu_n}=\frac12\delta_{0,n}+\frac12\delta_{1,0}$ and ${\color{blue}\nu_n}=\frac12\delta_{0,0}+\frac12\delta_{1,n}$. Then $W_2(\mu_n,\nu_n)=1$ is the length of {\color{green} $\longrightarrow$} and {\color{green} $\longrightarrow$} indicates the optimal transport map. Furthermore $W_{2,\eta}(\mu_n,\nu_n)=n$ is the length of {\color{brown} $\longrightarrow$} and {\color{brown} $\longrightarrow$} indicates the optimal transport for $W_{2,\eta}$.}\label{fig:counterexample}
\end{minipage}
\end{figure}
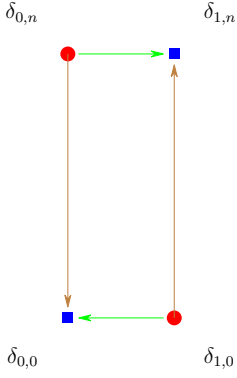

\begin{rem}
Proposition \ref{prop:cond_wass}$ii)$ indicates another use of the conditional Wasserstein distance. Namely if we have two measures $\mu_0,\mu_1\in\P_\eta$ and we have that $W_2(\mu_0,\mu_1)$ is small that does not necessarily mean that the posteriors are close. We can see in Figure \ref{fig:counterexample} that there exist $\eta\in \P_2(\R),\mu_n,\mu_n\in \P_{\eta}(\R\times\R)$ such that $W_2(\mu_n,\nu_n)=1$ and $\E_{w\sim \eta}[W_2^2(\mu_{0,w},\mu_{1,w})]=W_{2,\eta}(\mu_n,\nu_n)=n$. This means that $W_2$ is not a good measure in order to access if a generator $g$ is approximating the posteriors well. In contrast if we use $W_{2,\eta}$ we know that the posteriors are, at least in expectation, well approximated.
\end{rem}

As shown in \cite[Theorem 3,4]{kerrigan2024dynamic} and \cite{barboni2024understanding}, there is a counterpart of Theorem \ref{thm:abscont_ce} for the conditional Wasserstein distance.

\begin{thm}\label{thm:abscont_ce_c}
Let $\mu_t:I\to \P_{\eta}(\R^m\times \R^d)$ be a narrowly continuous curve. Then $\mu_t$ is absolutely continuous in with respect to $W_{2,\eta}$
in the sense of \eqref{eq:abs_cont},
if and only if there exists a measurable vector field $v:I\times \R^m\times\R^d \to \R^m\times\R^d$
such that the following three conditions are fulfilled:
\begin{itemize}     
  \item[\rm{i)}] $(v_t)_{j} = 0$ as element in $L^2(\R,\mu_t)$ for a.e. $t\in[0,1]$ and
    for all $j\leq m$,
\item[\rm{ii)}]  $\|v_t\|_{L^2(\R^m \times \R^d,\mu_t)}\in L^1(I)$, 
\item[\rm{iii)}]  $(\mu_t,v_t)$ fulfills \rm{\eqref{eq:ce}}.
\end{itemize}
In this case, the metric derivative $|\mu_t'|$ exists and we have 
$|\mu'_t|\leq \|v_t\|_{L^2(\mu_t,\R^m \times \R^d)}$ for a.e. $t\in I$. 
\end{thm}

The following theorem ensures that the curve $\mu_t$ associated to a plan $\alpha\in \Gamma_{\eta}(\mu_0,\mu_1)$ fulfills $\mu_t\in \P_\eta(\R^m \times \R^d)$ for all $t\in[0,1]$ and the corresponding vector field does not move mass in the $w$ component.

\begin{thm}\label{prop:continuity}
Let $\mu_0,\mu_1\in \P_{\eta}(\R^m\times \R^d)$ and let $\alpha\in \Gamma_{\eta}(\mu_0,\mu_1)$. Let $(\mu_t,v_t)$ be the curve and measurable velocity field induced by $\alpha$, 
i.e.
\begin{equation} \label{meanth_c}
\mu_t = e_{t,\sharp} \alpha, 
\quad \text{and} \quad
v_t\mu_t=e_{t,\sharp}[(\u_2-\u_1)\alpha] \quad \text{for a.e. } t\in [0,1].
\end{equation} 
Then the following holds true:
\begin{itemize}
    \item[\rm{i)}] $\mu_t\in \P_{\eta}(\R^m\times\R^d)$ for all $t\in[0,1]$,
     \item[\rm{ii)}]  $(v_t)_{j} = 0$ as element in $L^2(\R,\mu_t)$ for a.e. $t\in[0,1]$ and
    for all $j\leq m$,
    \item[\rm{iii)}] 
   $(\mu_t,v_t)$ satisfy the continuity equation \eqref{eq:ce} and
    $$  
    \|v_t\|_{L^2(\R^m \times \R^d,\mu_t)}\leq \|y-x\|_{L^2\left((\R^m \times \R^d)^2,\alpha \right)} \quad \text{ for a.e. } t \in [0,1].
    $$
    In particular, we have $\|v_t\|_{L^2(\R^m \times \R^d,\mu_t)}\in L^1(I)$
    and $\mu_t$ is an absolutely continuous curve with respect to $W_{2,\eta}$.
    
\end{itemize}
\end{thm}

\begin{proof}
To proof $\rm{i)}$ and $\rm{ii)}$  follow from \cite[Lemma 5, Proposition 6]{chemseddine2024conditional}.  The first statements of $\rm{iii)}$ follow from Proposition \ref{prop:alpha_curve} and finally the absolutely continuity with respect to $W_{2,\eta}$ is implied by Theorem \ref{thm:abscont_ce_c}.
\end{proof}

\begin{rem}
Note that for an absolutely curve $\mu_t\in\P_{\eta}$ and a vector field $v_t\in L^2(\mu_t,\R^m\times\R^d)$ it is not necessarily true that $v_j=0$ for every $j\leq m$ since these components could encode $\eta$ preserving maps. However, Proposition \ref{prop:continuity} ensures that for a curve associated to a plan $\alpha\in \Gamma_\eta(\mu_0,\mu_1)$ this is the case and in particular for a map $\phi$ which is a solution of the ODE $\partial_t\phi_t=v_t(\phi_t);\phi_0=\Id$ we can conclude that $\pi^1(\phi_t(\u))=\pi^1(\u)$.
\end{rem}

The counterpart of Proposition \ref{prop:geodesic} reads as follows, see
\cite[Lemma 5]{chemseddine2024conditional}).

\begin{prop}\label{lem:geo_dis}
Let $\mu_0,\mu_1\in \P_{\eta}(\R^m\times \R^d)$ and
let $\alpha\in \Gamma_{\eta,o}(\mu_0,\mu_1)$.
Then the curve
$
\mu_t \coloneqq (e_t)_{\sharp}\alpha
$
is a geodesic in $(\P_{\eta}(\R^m \times \R^d), W_{2,\eta})$. In particular, $(P_\eta(\R^m \times \R^d),W_{2,\eta})$ is a geodesic space.
\end{prop}

\begin{figure}
\begin{minipage}{0.4\textwidth}
\begin{subfigure}{0.45\textwidth}
\scalebox{0.6}{
    \begin{tikzpicture}[>={Stealth[length=8pt,angle'=28,round]}]
    \fill [blue] (-0.1,-0.1) rectangle (0.1,0.1);
    \node[below left=14pt of {(0.2,0.2)}, outer sep=2pt,fill=white] {$\delta_{0,0}$};
    \fill [red] (0,5) circle [radius=4pt];
    \node[above left=14pt of {(0.2,4.8)}, outer sep=2pt,fill=white] {$\delta_{0,5}$};
    \fill [green] (1,5) circle [radius=4pt];
    \fill [red] (2,0) circle [radius=4pt];
    \fill [green] (1,0) circle [radius=4pt];
    \node[below right=14pt of {(1.8,0.2)}, outer sep=2pt,fill=white] {$\delta_{1,0}$};
    \fill [blue] (2-0.1,5-0.1) rectangle (2+0.1,5+0.1);
    \node[above right=14pt of {(1.8,4.8)}, outer sep=2pt,fill=white] {$\delta_{1,5}$};
    \draw [->, brown] (1.8,0) -- (0.2,0);
    \draw [->, brown] (0.2,5) -- (1.8,5);
    
\end{tikzpicture}
} 
\caption{ }
\end{subfigure}
\begin{subfigure}{0.45\textwidth}
\scalebox{0.6}{
    \begin{tikzpicture}[>={Stealth[length=8pt,angle'=28,round]}]
    \fill [blue] (-0.1,-0.1) rectangle (0.1,0.1);
    \node[below left=14pt of {(0.2,0.2)}, outer sep=2pt,fill=white] {$\delta_{0,0}$};
    \fill [red] (0,5) circle [radius=4pt];
    \fill [green] (0,2.5) circle [radius=4pt];
    \node[above left=14pt of {(0.2,4.8)}, outer sep=2pt,fill=white] {$\delta_{0,5}$};
    \fill [red] (2,0) circle [radius=4pt];
    \fill [green] (2,2.5) circle [radius=4pt];
    \node[below right=14pt of {(1.8,0.2)}, outer sep=2pt,fill=white] {$\delta_{1,0}$};
    \fill [blue] (2-0.1,5-0.1) rectangle (2+0.1,5+0.1);
    \node[above right=14pt of {(1.8,4.8)}, outer sep=2pt,fill=white] {$\delta_{1,5}$};
    \draw [<-, brown] (2,4.8) -- (2,0.2);
    \draw [->, brown] (0,4.8) -- (0,0.2);

\end{tikzpicture}
}
\caption{ }
\end{subfigure}
\end{minipage}
\begin{minipage}{0.6\textwidth}
\caption{Consider $\eta=\frac12\delta_{0}+\frac12\delta_1$, ${\color{red}\mu_0}=\frac12\delta_{0,5}+\frac12\delta_{1,0}$ and ${\color{blue}\mu_1}=\frac12\delta_{0,0}+\frac12\delta_{1,5}$. \rm{(a)} Geodesic with respect to $W_2$, green: {\color{green}$\mu_{\frac12}$}. \rm{(b)} Geodesic with respect to $W_{2,\eta}$, green: {\color{green}$\mu_{\frac12}$}. }
\label{fig:cond_geodesic}

\end{minipage}
\end{figure}
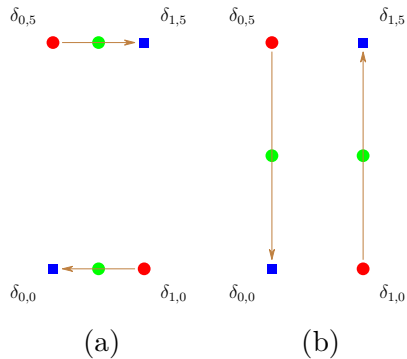
An illustration of the difference between geodesics with respect to $W_2$ and geodesics with respect to $W_{2,\eta}$ is given in Figure \ref{fig:cond_geodesic}.
\subsection{Almost Conditional Couplings}
One drawback of the space $\P_\eta(\R^m \times \R^d)$ is that we can in general not approximate $\mu\in\P_{\eta}(\R^m \times \R^d)$ by an empirical measure, if $\eta$ is not empirical. 
In other words, an empirical approximation 
$\mu_n = \frac{1}{n} \sum_{i=1}^n \delta_{(w_i,x_i)}$
of $\mu\in \P_{\eta}(\R^m \times \R^d)$ with respect to $W_2$ is usually not an element of $\P_{\eta}(\R^m \times \R^d)$. Even if we approximate 
two measures $\mu,\nu\in \P_{\eta}(\R^m \times \R^d)$ 
with respect to $W_2$ by empirical measures 
$\mu_n,\nu_n$ with $\pi^1_\sharp \mu_n=\pi^1_\sharp\nu_n$, it is not necessarily true that there is a sequence $\alpha_n\in \Gamma_{\pi^1_\sharp \mu_n,o}(\mu_n,\nu_n)$ such that $\alpha_n$ converges to an  $\alpha \in \Gamma_{\eta,o}(\mu,\nu)$ as \cite[Example 9]{chemseddine2024conditional} shows.

One way to overcome the above drawback is to relax the hard constraint 
$\pi^{1,1}\alpha=\Delta_{\sharp}\eta$ in the definition of the conditional Wasserstein distance.
To this end, we define
$$d_\beta(\u_1,\u_2) \coloneqq \|x_1-x_2\|^2+\beta\|w_1-w_2\|^2, \quad \beta > 0.$$
Then 
\begin{align}
W_{\beta}(\mu,\nu) \coloneqq \min_{\alpha\in\Gamma(\mu,\nu)}
\Big( \int\limits_{(\R^m \times \R^d)^2} d_\beta(\u_1,\u_2) \, \dd\alpha \Big)^\frac12
\end{align}
defines a metric  on $\P_2(\R^m \times \R^d)$.
It turns out that for $0 \ll \beta $ the optimal couplings $\alpha$ in this metric are ''close'' to fulfilling $\pi^{1,1}\alpha = \Delta \eta$ in the following sense, see \cite[Proposition 10]{chemseddine2024conditional}. 

\begin{prop}\label{prop_beta}
Let $\mu_0,\mu_1\in \P_{\eta}(\R^{m} \times \R^d)$ and let $(\alpha_\beta)_\beta$ be a sequence of optimal transport plans from $\mu_0$ to $\mu_1$ with respect to $W_{\beta}$.
Then we have 
\[
\int\limits_{\R^m\times \R^m} \|w_1-w_2\|^2 \, \dd{\pi^{1,1}}_{\sharp}\alpha_{\beta} 
=
\int\limits_{(\R^m\times \R^d)^2} \|w_1-w_2\|^2 \, \dd \alpha_\beta
\rightarrow 0
\quad \text{as} \quad
\beta\to\infty.
\]
\end{prop}

The distance $W_{\beta}$ successfully addresses the issue of approximating measures by empirical measures as the following proposition from \cite[Proposition 12]{chemseddine2024conditional} shows.

\begin{prop}\label{prop:ex_fix}
Let $\mu,\nu\in P_{\eta}(\R^m\times \R^d)$ and let $\mu_n,\nu_n$ be empirical measures which  converge narrowly to $\mu,\nu$. 
Then, for a sequence $\beta_k \to\infty$, there exists an increasing subsequence $n_k$ and  optimal plans $\alpha_{n_k} \in \Gamma(\mu_{n_k},\nu_{n_k})$ for $W_{\beta_k}(\mu_{n_k},\nu_{n_k})$ such that $\alpha_{n_k}$ converges narrowly to an optimal plan $\alpha\in\Gamma_{\eta,o}(\mu,\nu)$.
\end{prop}

\begin{rem} For $\beta>0$ let $a_\beta:\R^m\times\R^d\to \R^m\times \R^d $ be defined by $a_\beta(w,x)=(\sqrt{\beta} w,x)$. Then an easy calculation shows that $W_\beta(\mu,\nu)=W_2(a_{\beta,\sharp}\mu,a_{\beta,\sharp}\nu)$ and there is a one-to-one correspondence between optimal plans for $W_\beta(\mu,\nu)$ and $W_2(a_{\beta,\sharp}\mu,a_{\beta,\sharp}\nu)$ given by $\alpha \mapsto (a_\beta,a_\beta)_\sharp\alpha$. Thus $W_\beta$ and optimal plans can be computed by standard $W_2$ solvers.
\hfill $\diamond$
\end{rem}

\subsection{Bayesian Flow Matching}
Since $\Gamma_\eta(\mu_0,\mu_1)\subseteq\Gamma(\mu_0,\mu_1)$ we can learn the vector field $v_t$ corresponding to $\alpha\in \Gamma_\eta(\mu_0,\mu_1)$ with the extended flow matching objective \eqref{eq:markov_loss}, i.e.,
\begin{align}
\CFM(\theta):=\E_{t\sim \mathcal L[0,1],(w,x_0,w,x_1)\sim\alpha}\left[\|v^\theta(t,e_t((w,x_0,w,x_1))-(0,x_1-x_0)\|^2\right].
\end{align}
Since $\pi^{1,1}_\sharp \alpha=\Delta_\sharp \eta$,
we only have to consider samples $(w,x_0,w,x_1)$ from $ \alpha$.
Further, since  $v_j=0$ for $j\leq m$,
we can parametrize $v_t^\theta$ as function with values in $\R^d$ instead of $\R^m\times \R^d$. 
In this way, we have for a solution of 
$$\partial_t \phi_t(x)=v_t(\phi_t(x)), \quad \phi_0 (x) =x$$ automatically that 
$\pi^1_\sharp (\phi_{t,\sharp}\mu_0)=\eta$.

\begin{rem}\label{rem:cond_plans}
In order to get absolutely continuous curves from plans, we have used i) direct product of measures
or ii) optimal transport plans in the previous sections.
Let us see how similar settings look in the Bayesian case.
\\
i) Unfortunately, for $\mu_0,\mu_1 \in \mathcal P_\eta(\R^m \times \R^d)$ it does not hold that $\mu_0 \times \mu_1 \in \Gamma_\eta(\mu_0,\mu_1)$.
A remedy is to use $\mu_0 = \eta \times \nu \in \P_2(\R^m\times \R^d)$ with some $\nu \in \P_d(\R^d)$ and the plan
$$
\alpha \coloneqq \tilde \Delta_\sharp(\nu\times \mu_1)\in \Gamma_\eta(\mu_0,\mu_1),
$$
where
$$
\tilde \Delta:\R^d\times \R^m\times \R^d\to (\R^m\times \R^d)^2, \quad \tilde \Delta (x_0,w_1,x_1) \coloneqq (w_1,x_0,w_1,x_1).
$$
In order to sample from $\alpha$ we only need to be able to sample from $\nu$ and $\mu_1$. 
\\
ii) We can use Proposition \ref{prop:cond_wass} ii) and solve the optimal transport problem  for fixed samples $w$ from $\eta$,  to obtain an optimal coupling $\alpha\in \Gamma_{\eta,o}(\mu_0,\mu_1)$. In order to ensure that the $\alpha_w\in \Gamma_o(\mu_0^w,\mu_1^w)$ can be used to build $\alpha\in \Gamma_{\eta,o}(\mu_0,\mu_1)$,  some measurability conditions must  be ensured. For examples where these conditions are fulfilled, see e.g. \cite[Corollary 1.2]{gonzalez2024linearization} and \cite[Proposition 8]{chemseddine2024conditional}.\\
Of course we can also choose $\beta\gg0$ and compute an optimal transport plan $\alpha$ for $W_{\beta}(\mu_0,\mu_1)$. We then use the loss 
\begin{align}
\CFM(\theta):=\E_{t\sim \mathcal L[0,1],(w_0,x_0,w_1,x_1)\sim\alpha}\left[\|v^\theta(t,e_t((w_0,x_0,w_1,x_1))-(0,x_1-x_0)\|^2\right].
\end{align}
Although this is not exactly the loss for the plan $\alpha$ and  we only have approximately $w_0\cong w_1$, in practice this leads to good results, see alsonext section. Furthermore, we use minibatch OT described in Algorithm \ref{alg:minibatch_ot} to approximately minimize $\CFM$. We will call this strategy minibatch OT Bayesian flow matching with respect to $W_{\beta}$.
\end{rem}

\subsection{Numerical Examples of Bayesian Flow Matching}
In order to show the feasibility of Bayesian flow matching, we provide two examples. 

\paragraph{1. Conditional image generation.}
We used the dataset Cifar10 \cite{krizhevsky2009learning},    which consists of 10 classes of color images of size $32\times32\times3$. Here $\eta=\P_2(\R^ {10})$ is defined by $\eta=\frac{1}{10}\sum_{i=1}^{10}\delta_{e_i}$, where $e_i$ is the standard basis of $\R^10$. Then samples from $\mu_1$ are of the form $(e_{l(x_i)},x_i)$, where $l(x_i)$ is the class of $x_i$. The source measure is $\mu_0\coloneqq \eta\times\N(0,1)$. 
For conditional image generation, we used trained flow matching models from \cite{chemseddine2024conditional}. 
For training with respect to $W_{2,\beta}$ 
the minibatch OT training algorithm \ref{alg:minibatch_ot} was used for Figure \ref{fig:bayesian_flow}[(a),(b)]. 
For the plan $\alpha\in\Gamma_\eta(\mu_0,\mu_1)$ in Figure \ref{fig:bayesian_flow}[(c)], we applied the coupling from Remark \ref{rem:cond_plans} and for training Algorithm \ref{alg:flow_product}. 

Figure \ref{fig:bayesian_flow} shows generated images, where in row $i$ the samples are generated from $(e_i,z)$ where $10$ different samples $z\sim \N(0,1)$ are drawn. For $\beta=1$, Figure \ref{fig:bayesian_flow}[(a)] illustrates that the generation is not class conditional, i.e. in every row there are images of multiple classes. 
For $\beta=100$ and for the conditional plan from Remark \ref{rem:cond_plans},
we can see in Figure \ref{fig:bayesian_flow} that the generation is class conditional, meaning that every row only contains samples from one class.

\begin{figure}[ht]
\centering
\begin{subfigure}{0.325\linewidth}
    \includegraphics[width=\linewidth]{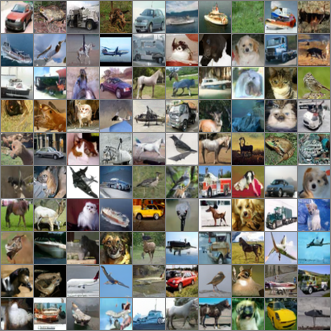}
    \caption{$\alpha$ for $W_{2,1}(\mu_0,\mu_1)$}
\end{subfigure}
\begin{subfigure}{0.325\linewidth}
\includegraphics[width=\linewidth]{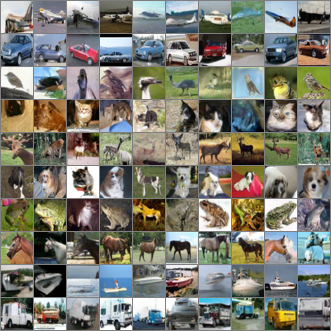}
\caption{$\alpha$ for $W_{2,100}(\mu_0,\mu_1)$}
\end{subfigure}
\begin{subfigure}{0.325\linewidth}
\includegraphics[width=\linewidth]{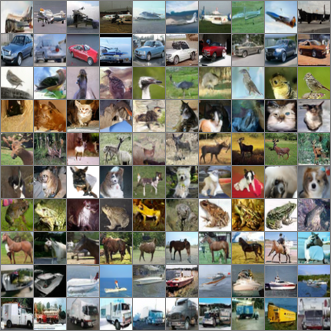}
\caption{$\alpha\in \Gamma_{\eta}(\mu_0,\mu_1)$}
\end{subfigure}
\caption{(Bayesian) Flow matching on Cifar10.} 
\label{fig:bayesian_flow}
\end{figure}

\paragraph{Simple Bayesian inverse problem.}
The following example can be found in \cite[Section 8.2]{chemseddine2024conditional}. 
Let $X \in \R^5$ be a random variable for which the law $P_X$ is a Gaussian mixture model with 10 modes with means uniformly chosen in $[-1,1]^5$ and standard derivations $0.1$. 
We consider the Bayesian inverse problem 
$$
Y=AX+\Xi \quad \text{with}  \quad A \coloneqq \big(\frac{0.1}{i}\delta_{i,j} \big)_{i,j=1}^5.
$$
The noise variable $\Xi$, which is independent from $X$, is distributed as $\N(0,0.1)$. Then by \cite[Lemma 11]{HHS22} the posterior distribution $P_{X|Y=y}$ is a Gaussian mixture model which can be analytically computed. This allows us to sample from $P_{X|Y=y}$. \\
\emph{Training Bayesian flow matching with $\alpha$ from Remark \ref{rem:cond_plans}.} Here we sample from $\alpha\in \Gamma(P_Y\times \N(0,I_5),P_{Y,X})$ by sampling $x_i\sim P_X$, computing $y_i\coloneqq f(x_i)+\xi_i\sim P_Y$ for $\xi_i\sim \N(0,0.1\Id)$ and $z_i\sim \N(0,\Id)$. We then use $(y_i,z_i,y_i,x_i)$ in order to approximate $\alpha$ and to compute $\CFM(\theta)$. Batching in $i$ and $t$ is done as in Algorithm \ref{alg:flow_product}.\\
\emph{Training minibatch OT Bayesian flow matching with respect to $W_{2,100}$.} Here we use Algorithm \ref{alg:minibatch_ot} for $\mu_0=P_Y\times \N(0,\Id)$ and $\mu_1=P_{Y,X}$. In order to be able to sample from $\mu_0$ for a minibatch $(y_i,x_i)_{i=1}^{N_{bOT}}\sim P_{Y,X}$, where $N_{bOT}$ is the minibatch size, we sample $z_i\sim \N(0,\Id)$ and use $(y_i,z_i)_{i=1}^{N_{bOT}}\sim \mu_0$.

The results are depicted in Figure \ref{fig:inverse_problem}. To create this figure, we draw $x\sim P_X$ and computed an observation $y=Ax+\xi\sim P_Y$. Then we draw samples $x^{y,n},\, 1\leq n\leq 1000$ from $P_{X|Y=y}$. The orange colored histograms on the diagonal are the one dimensional histograms of the projection onto the $i-th$ component $\langle e_i,x^{y,n}\rangle$. For $i\neq j$ the orange colored two dimensional histograms of $(\langle e_i,x^{y,n}\rangle, \langle e_j,x^{y,n}\rangle)$ are shown in position $(i,j)$. The blue histograms are the output obtained by the flow matching procedures, i.e. instead of $x^{y,n}$, we use $\phi_1(y,z^n)$ for $z^n\sim \N(0,I_5)$, $1\leq n\leq 1000$, where $\phi_t$ solves $\partial_t \phi_t = u^\theta_t(\phi_t),\, \phi_0(y,z)=(y,z)$.

\begin{figure}[h!]
    \centering
    \begin{subfigure}{0.49\linewidth}
            \includegraphics[width=\linewidth]{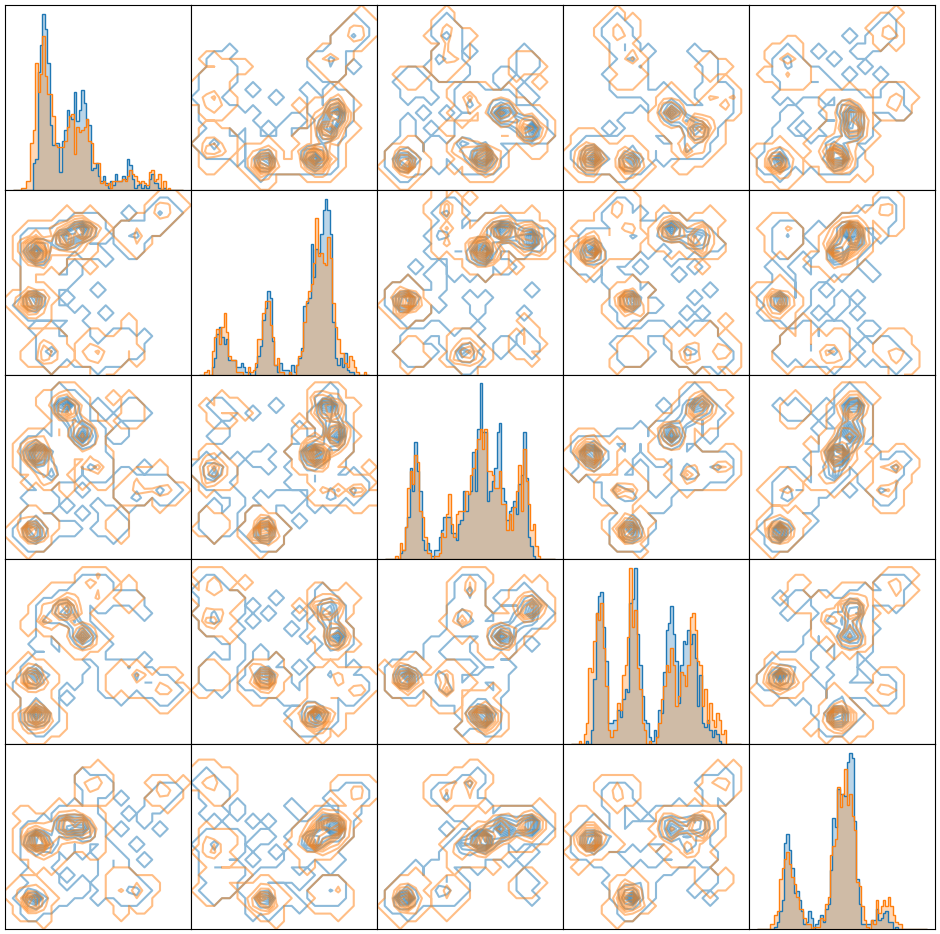}
    \caption{minibatch OT flow matching for $W_{2,100}$}
    \label{fig:enter-label}
    \end{subfigure}
    \begin{subfigure}{0.49\linewidth}
        \includegraphics[width=\linewidth]{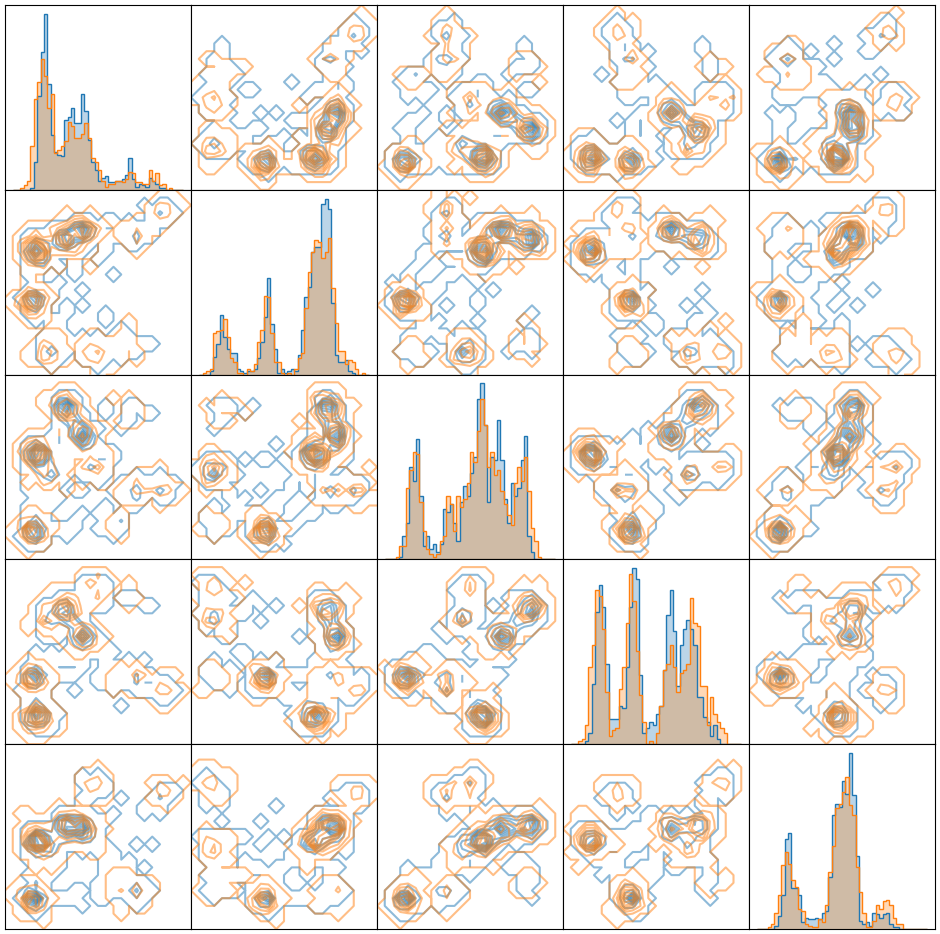}
        \caption{$\alpha$ from Remark \ref{rem:cond_plans}}
    \end{subfigure}
    \caption{Bayesian flow matching for an inverse problem.}
    \label{fig:inverse_problem}
\end{figure}

\section{Continuous Normalizing Flows} \label{sec:CNF}
Flow matching is related to continuous normalizing flows,
which historically were introduced earlier, but appears to be more complicated
when learned with a likelihood loss.
To provide a more circumvent view, we add this section.
A remark on (noncontinuous) normalizing flows is given in 
Appendix \ref{appc}.

Continuous normalizing flows aim to find a vector field $v_t$ such that the solution $\psi:[0,1]\times \R^d\to \R^d$ of 
\begin{equation} \label{eq:ode_pure}
\partial_t\psi(t,x)=v_t(\psi(t,x)), \quad \psi(0,x)=x,
\end{equation}
 satisfies $\psi(1,\cdot)_\sharp \mu_0=\mu_1$. More precisely,
for every $t \in [0,1]$, the solution $\psi(t, \cdot): \R^d \to \R^d$ has to be a diffeomorphism. 
In contrast to flow matching, continuous normalizing flow techniques
work without previously determining the curve via plans, Markov kernels or a
stochastic process, but approximate the vector field by a neural network using an appropriate loss function, namely the log-likelihood one.

Therefore, we first derive  general properties of vector fields and associated curves  to make the process invertible in the next Subsection \ref{subsec:flow}. In other words, we will take care about the time reverse ODE. 
This will facilitate the finding of an appropriate loss function in Subsection \ref{sec:loss}. 
In Subsection \ref{subsec:adj}, we deal with the minimization of the loss
by computing gradients with the adjoint method.

In this section, we will exclusively work with absolutely continuous measures and switch from the notation of measures to those of densities. 
In particular, we will write $p_t\coloneqq \psi(t,\cdot)_\sharp \rho_0$ for the density $p_t$ of $\psi(t,\cdot)_\sharp\mu_0$, if it exists. 

\subsection{Curves of Probability Measures from ODEs} \label{subsec:flow}
We start with classical results from the theory of ODEs,  
for existence and regularity see, e.g., \cite[Corollary 2.6, Theorem 2.10]{teschl2024ordinary}.

\begin{thm}\label{thm:ode_reg}
Assume that $v_t\in C^l([0,1]\times\R^d,\R^d)$,
$l\geq 1$ fulfills a global Lipschitz condition in the second variable, i.e., there exists  $K>0$ such that 
\begin{align}
    \|v_t(x)-v_t(y)\|\leq K\|x-y\| \quad \text{for all }x,y\in\R^d,t\in[0,1].
\end{align}
Then, for any fixed $t_0\in [0,1]$ and $x_0\in\R^d$, 
the initial value problem $f(t_0)=x_0$ and
\begin{equation} \label{ode_init} 
f'(t)=v_t(f(t)),
\end{equation}
admits a unique global solution $f:[0,1]\to\R^d$. 
Furthermore, if we define 
$\boldsymbol{\psi}(t,s,x)\coloneqq f(t)$ for the solution $f$
of \eqref{ode_init} with initial condition $f(s)  = x$, then it holds $\boldsymbol{\psi} \in C^l([0,1]\times[0,1]\times\R^d,\R^d)$. 
\end{thm}

Concerning the reverse ODE, we obtain the following corollary.

\begin{cor}\label{cor:solutions_diffeo}
In the setting of Theorem \ref{thm:ode_reg}, we have that $\boldsymbol{\psi}(t,s,\boldsymbol{\psi}(s,t,x))=x$. 
In particular, for $(t,s) \in [0,1] \times [0,1]$,
the function $\psi^{t,s}: \R^d \to \R^d$ defined by
$\psi^{t,s}(x)\coloneqq \boldsymbol{\psi}(t,s,x)$
is a $C^l(\R^d,\R^d)$ diffeomorphism with inverse $\psi^{s,t}$.
\end{cor}

\begin{proof} Let $f$ and $g$ be a solution of \eqref{ode_init} with 
initializations $f(t) = x$,
and $g(s) = f(s) = \boldsymbol{\psi}(s,t,x)$, respectively.
By Theorem \ref{thm:ode_reg}, we know that $f=g$ on $[0,1]$
and therefore
$$
g(t) = \boldsymbol{\psi}\left( t,s,\boldsymbol{\psi}(s,t,x) \right) = f(t) = x. 
$$
\end{proof}

The next proposition shows that under the above smoothness assumptions
$\psi(t, \cdot): \R^d \to \R^d$ is indeed a $C^2$ diffeomorphism.

\begin{prop} \label{soln}
Assume that $v_t\in C^l([0,1]\times\R^d,\R^d)$,
$l\geq 2$ globally Lipschitz in the second variable.
Then \eqref{eq:ode_pure} admits a unique global solution and 
$\psi\in C^l([0,1]\times \R^d,\R^d)$. 
If $\mu_0=p_0\L$ with $p_0\in C^{l-1}(\R^d)$, then the curve $\mu_t \coloneqq \psi(t,\cdot)_\sharp \mu_0$ admits a density $p_t\in C^{l-1}([0,1]\times\R^d)$.
Furthermore, if there exists $t_0\in[0,1]$ such that 
$p_{t_0}(x)>0$  for all $x\in \R^d$, then $p_t(x)>0$ for all $x\in\R^d$ and all $t\in[0,1]$.
\end{prop}

\begin{proof}
The existence and uniqueness follow directly from Theorem \ref{thm:ode_reg}. Since $\psi(t,\cdot)$ is a $C^l(\R^d,\R^d)$ diffeomorphism by Corollary \ref{cor:solutions_diffeo} we know that $\det (\nabla_x\psi(t,\cdot)(x))\neq 0$. Furthermore using \eqref{push_density_1} we conclude that $\mu_t$ is absolutely continuous with density 
$$p_t(x)=\frac{p_0\circ \psi(t,\cdot)^{-1}(x)}{\det (\nabla_x\psi(t,\cdot)(x))}.$$ 
The latter equation implies the remaining claims.
\end{proof}

\begin{cor}\label{cor:ce}
    Let $v_t \in C^2([0,1]\times\R^d,\R^d)$ be globally Lipschitz and $\mu_0=p_0\L$ with $p_0\in C^{1}(\R^d)$. Let $\psi\in C^l([0,1]\times \R^d,\R^d)$ be the solution of
    \eqref{eq:ode_pure} and $p_t \coloneqq \psi(t,\cdot)_\sharp p_0$. Then $(p_t,v_t)$ fulfills the continuity equation in a strong sense
    $$
    \partial_t p_t + \div(p_t \, v_t) = 0, \quad x \in \R^d, t \in [0,1] .
    $$
\end{cor}

\begin{proof}
By Remark \ref{rem:conversely}, we know that \eqref{eq:ode_pure} implies for all $\varphi \in C_c^\infty((0,1) \times\R^d)$ that
$$
0 =
\int_0^1\int_{\R^d} \langle \nabla_x\varphi,v_t\rangle + \partial_t \varphi \, \dd [\psi(t,\cdot)_\sharp\mu_0]\dd t 
= 
\int_0^1\int_{\R^d} \big( \div (v_t p_t) + \partial_t p_t \big) \varphi \, \dd x \dd t.
$$
By Theorem \ref{soln}, we know that $p_t \in C^{1}(\R^d)$, so that the function in the
inner brackets is continuous, which yields the assertion.
\end{proof}

\subsection{Likelihood Loss for Continuous Normalizing Flow} \label{sec:loss}

Opposite to flow matching, we assume that the initial density of the ODE is the
target density, and the final one approximates the latent density, i.e.,
$$
p_0 \approx p_{\text{data}}, \quad p_1 = p_{\text{latent}}.
$$
We assume that the vector field and thus, also the corresponding ODE solution depend on a parameter vector $\theta \in \R^n$, i.e.,
\begin{align}\label{eq:ode_theta}
\partial_t\psi^\theta(t,x)=v_t^\theta(\psi(t,x)), \quad \psi^\theta(0,x)=x.
\end{align}

By the previous subsection, we know for
$v_t^\theta\in C^2([0,1]\times\R^d,\R^d)$ with Lipschitz continuous second component for all $t\in[0,1]$,
that
$S^\theta \coloneqq  \psi^\theta(1,\cdot)$
is a $C^2$ diffeomorphism and
$T^\theta \coloneqq \left( S^\theta \right)^{-1}$ 
fulfills $p_0:=T^\theta_\sharp p_1$. 
Moreover, we have for $p_t:=\psi^\theta(t,\cdot)_\sharp p_0$ that 
$$
p_1 = S^{\theta}_\sharp p_0
=S^{\theta}_\sharp\, T^\theta_\sharp p_Z=p_Z
\quad \text{and} \quad
T_\sharp^\theta p_1 
=
T_\sharp^\theta \, S_\sharp^\theta p_0=p_0. 
$$ 
As  \emph{loss function}, the log-likelihood function appears to be a reasonable choice
\begin{align}\label{eq:kl_loss2}
\mathcal{L}(\theta)=\E_{x\sim p_X}[-\log T^\theta_\sharp p_1]
=
\E_{x\sim p_X}[-\log p_0].
\end{align}
Alternatively, this may be reformulated using the Kullback-Leibler divergence between $p_X$ and  $T^\theta_\sharp p_1$,
see Remark \ref{logl_kl}.
In order to minimize the loss function using backpropagation,
we must compute the gradient of $\log(T^\theta_{\sharp}p_1) = \log p_0(x)$ with respect to $\theta$. 
To this end, we introduce the function $\ell:[0,1] \times \R^d \to \R ^d$ by
\begin{equation} \label{noch}
\ell(t,x) \coloneqq \log \big(p_t(\psi(t,x) ) \big) - \log \big(p_0(x) \big),
\end{equation}
so that the loss function can be rewritten setting $t=1$ as
\begin{align}\label{loss*}
  \E_{x\sim p_X}[- \log p_0(x)]=\E_{x\sim p_X}\left[\ell(1,x)-\log p_1(\psi(1,x))\right].
\end{align}
Interestingly, $\ell$ is the solution of an ODE which includes the
velocity field $v_t$.

\begin{prop}\label{prop:log_div}
Let $v \in C^2([0,1]\times\R^d,\R^d)$ be Lipschitz continuous in the second component for all $t\in [0,1]$ and let $p_0\in C^1(\R^d)$ be a strictly positive density. 
For the solution $\psi\in C^l([0,1]\times \R^d,\R^d)$ of \eqref{eq:ode_pure}, let $p_t \coloneqq \psi(t,\cdot)_\sharp p_0$. Then 
$\ell:[0,1] \times \R^d \to \R ^d$ in \eqref{noch}
is the solution of the ODE 
\begin{align} \label{sol_l}
\partial_t \ell(t,x) = -(\nabla \cdot  v_t)(\psi(t,x)), \quad  \ell(x,0)=0
\end{align}
\end{prop}

\begin{proof}
By the chain rule, we obtain 
\begin{align}
\partial_t \ell(t,x) &= \frac{\dd }{\dd t}\log \big( p_t(\psi(t,x) ) \big)
=
\frac{1}{p_t(\psi(t,x))}\frac{\dd}{\dd t} p_t(\psi(t,x))\\
&=
\frac{1}{p_t(\psi(t,x))}\Big(\langle\nabla_x p_t(\psi(t,x)),\partial_t \psi(t,x)\rangle+\partial_t p_t(\psi(t,x))\Big).
\end{align}
and further by applying the definition of $\psi$ for the first term and the continuity equation from Corollary \ref{cor:ce}
for the second one, 
\begin{align*}
\partial_t \ell(t,x) &=\frac{1}{p_t(\psi(t,x))}\Big(\langle\nabla_x p_t(\psi(t,x)),v_t(\psi(t,x))\rangle
-\mathrm{div}\big(p_t(\psi(t,x) \big)v_t(\psi(t,x))\Big)\\
&=\frac{1}{p_t(\psi(t,x))}\Big(\langle\nabla_x p_t(\psi(t,x)),v_t(\psi(t,x))\rangle\\
&\quad\qquad-\langle\nabla_x p_t(\psi(t,x)),v_t(\psi(t,x))\rangle-p_t(\psi(t,x)) (\div \, v_t)(\psi(t,x))\Big)\\
&=-\div(v_t(\psi(t,x))).
\end{align*}
\end{proof}

Now we can combine the ODEs for $\psi$ and $\ell$ into the ODE system
\begin{align}
\begin{pmatrix}
\partial_t {\Psi_1}(t,x,y,\theta)\\
\partial_t {\Psi_2}(t,x,y,\theta)\\
\partial_t {\Psi_3}(t,x,y,\theta)
\end{pmatrix}
=
\begin{pmatrix}
v_t\big(\Psi_1(t,x,y,\theta),\theta \big)\\
-(\div_x \, v_t)\big(\Psi_1(t,x,y,\theta), \theta \big)\\
0
\end{pmatrix}
,\quad
\begin{pmatrix}
\Psi_1(0,x,y,\theta)\\
\Psi_2(0,x,y,\theta)\\
\Psi_3(0,x,y,\theta)
\end{pmatrix}
=
\begin{pmatrix}
x\\
0\\
\theta
\end{pmatrix}
\end{align}
where $x\in \R^d,y\in \R$. Note that $\Psi_2$ corresponds to $\ell$ from \eqref{noch} and $\Psi_3$ remains $\theta$. Making the velocity field dependent on a parameter $\theta\in\R^n$, we can rewrite the loss function \eqref{loss*} as
\begin{align}\label{eq:kl_loss}
\mathcal{L}(\theta)&=\E_{x\sim p_X}\big[\Psi_2(1,x,0,\theta)- \log p_1\big(\Psi_1(1,x,0,\theta) \big)\big].
\end{align}
Now the main challenge is to find  $\nabla_\theta \mathcal{L}(\theta)$. 
To this end, consider 
$$F(x,y,\theta) \coloneqq y-\log p_1(x).$$
Then
\[
F \circ(\Psi(1,\cdot,\cdot,\cdot))(x,y,\theta)
=
\Psi_2(1,x,0,\theta)-\log p_1\big(\Psi_1(1,x,0,\theta)\big)
\]
and we can compute by the Leibniz rule for measure spaces
\begin{equation} \label{losss}
\nabla_\theta \mathcal{L}(\theta)
=
\E_{x\sim p_X}[\nabla_\theta (F \circ(\Psi (1,\cdot,\cdot,\cdot))(x,y,\theta)]. 
\end{equation} 
In the next subsection, we show how to compute 
$\nabla_\theta (F \circ \Psi (1,\cdot,\cdot,\cdot))(x,y,\theta)$
by solving a system of ODEs.

\subsection{Computing Gradients with the Adjoint Method} \label{subsec:adj}
This section is based on the original paper \cite{CRBD2018} as well as the blog \cite{schurovadj}.

To start from an arbitrary time $s\in[0,1]$, we use the following notation.
For $v_t\in C^2([0,1]\times\R^d,\R^d)$,
let 
$\boldsymbol{\psi}:[0,1]\times [0,1] \times \R^d\to \R^d$ 
be the solution of 
\begin{align}\label{eq:ode_st}
\partial_t \boldsymbol{\psi}(t,s,x) = v_t(\boldsymbol{\psi}(t,s,x)), \quad \boldsymbol{\psi}(s,s,x)=x.
\end{align}
For fixed $s=0$, we have
$\psi (t,x) = \boldsymbol{\psi}(t,0,x)$.
For an arbitrary $F\in C^2(\R^d)$, we define a function $a: [0,1] \times \R^d \to \R^d$ by
$$a_t(x) \coloneqq \nabla_x \big( F\circ \boldsymbol{\psi}(1,t,\cdot)\big)(\psi(t,x)).$$ 
Note that $a_t(x)$ is a row vector here.

\begin{prop}
Let $v_t\in C^2([0,1]\times \R^d,\R^d)$  be Lipschitz continuous in the second variable, and let
$\psi:[0,1]\times \R^d \to \R^d$ be  the solution of \eqref{eq:ode_pure}.
Then it holds
\begin{align}\label{eq:adjoint_equation}
\partial_t a_t(x)=- a_t(x) \, (\nabla_x v_t)(\psi(t,x)).
\end{align}
\end{prop}

\begin{proof}
Since $v_t\in C^2([0,1]\times\R^d,\R^d)$, 
we have by Theorem \ref{thm:ode_reg}
that  $\boldsymbol{\psi}\in C^2([0,1]\times[0,1]\times\R^d,\R^d)$. 
Noting that 
$$
\boldsymbol{\psi}(1,0,x)
=
\boldsymbol{\psi}(1,t,\cdot)\circ \boldsymbol{\psi}(t,0,x)=\boldsymbol{\psi}(1,t,\cdot)\left( \psi(t,x) \right),$$ 
we obtain 
\begin{align}
a_0(x)
&=
\nabla_x \big( F\circ \boldsymbol{\psi}(1,0,\cdot)\big)(\psi(0,x))
\\
&=
\nabla_x 
\big( F\circ(\boldsymbol{\psi}(1,t,\cdot)\circ \boldsymbol{\psi}(t,0,\cdot)) \big)(\psi(0,x))\\
&=\nabla_x \big( (F\circ\boldsymbol{\psi}(1,t,\cdot))\circ \psi(t,\cdot) \big)(x)
=
a_t(x) \,  \nabla_x \psi(t,x). 
\end{align}
Since the left hand side does not depend on $t$, we obtain after differentiation with respect to $t$ that
\begin{align}\label{ccc}
0=\partial_t(a_t(x)) \,  \nabla_x \psi(t,x) + a_t(x) \,  \partial_t(\nabla_x\psi(t,x)).
\end{align}
Next, we compute
\begin{align} 
\partial_t(\nabla_x\psi(t,x))=\nabla_x(\partial_t\psi(t,x))=\nabla_x(v_t(\psi(t,x)))=(\nabla_x v_t)(\psi(t,x))\nabla_x\psi(t,x).
\end{align}
Then we get in \eqref{ccc} that
\begin{align}
0=\partial_t(a_t(x)) \,  \nabla_x \psi(t,x) + a_t(x) \,  (\nabla_x v_t)(\psi(t,x))\nabla_x\psi(t,x).
\end{align}
By Corollary \ref{cor:solutions_diffeo}, we know that
$\psi(t,x)$ invertible with differentiable inverse, so that 
the matrix $\nabla_x\psi(t,x)$ is invertible and we obtain the assertion \eqref{eq:adjoint_equation}.
\end{proof}

To compute gradients of $v_t^\theta$ with respect  
$\theta\in\R^n$, we extend the ODE \eqref{eq:ode_st}
for $\boldsymbol{\psi}: [0,1] \times[0,1] \times \R ^d \times \R^n \to \R^d \times \R^n$ and 
$v: \times[0,1] \times \R^d \times \R^n \to \R^d$ as
\begin{align}
\partial_t\boldsymbol{\psi}(t,s,x,\theta)
= 
\begin{pmatrix}
v_t(\boldsymbol{\psi}(t,s,x,\theta)\\
0
\end{pmatrix}, 
\quad \boldsymbol{\psi}(s,s,x,\theta) =(x,\theta)
\end{align}
where we use the same symbols $\boldsymbol{\psi}$ and $v_t$ for convenience.
Again, we write $\psi$ for $\boldsymbol{\psi}(\cdot,0,\cdot,\cdot)$. 
Now let $F\in C^2(\R^d\times \R^n)$ be a function \emph{which does not depend on the second component} and define $a: [0,1] \times \R^d \times \R^n \to \R^d \times \R^n$ by
\begin{align}
a_t(x,\theta)&\coloneqq \nabla_{x,\theta}(F\circ \boldsymbol{\psi}(1,t,\cdot,\cdot))(\psi(t,x,\theta))
=\big( a_t^x,a_t^\theta\big).
\end{align}
Note that in the loss function \eqref{losss} we need
\begin{align}\label{eq:a0theta}
a_0^\theta = \nabla_{\theta}(F\circ \psi(1,\cdot,\cdot))(x,\theta).
\end{align}
Then \eqref{eq:adjoint_equation} modifies to
\begin{align}\label{eq:adjoint_equation2}
\partial_t a_t(x)=- a_t(x) \, \nabla_{x,\theta} 
\begin{pmatrix} 
v_t\\0
\end{pmatrix}
\big(\psi(t,x,\theta) \big).
\end{align}
Since 
$$
\nabla_{x,\theta}\begin{pmatrix}
v_t\\
0
\end{pmatrix}=\begin{pmatrix} \nabla_x v_t&\nabla_\theta v_t\\0&0\end{pmatrix},
$$ 
we obtain finally
\begin{align}\label{eq:node}
\begin{pmatrix}
\partial_t\psi(t,x,\theta)\\
\partial_t a_t^x(x,\theta)\\
\partial_t a_t^\theta(x,\theta)
\end{pmatrix}=
\begin{pmatrix}
v_t(\psi(t,x,\theta),\theta)\\
a_t^x(x,\theta)(\nabla_x v_t)(\psi(t,x,\theta))\\ a_t^x(x,\theta)(\nabla_\theta v_t)(\psi(t,x,\theta))
\end{pmatrix}
\end{align}
 As noted above, we need to compute $a_0^\theta$ to get $\nabla_\theta\L(\theta)$ in \eqref{losss}. In \eqref{eq:node} the initial conditions at $t=0$ are implicitly already encoded, since we used for our calculations that $\psi(0,x,\theta)=(x,\theta)$. 
 However, to get $a_0^\theta$, we need to find the appropriate condition for $t=1$ such that the solution of \eqref{eq:node} matches the initial condition at $t=0$. Since $F$ does not depend on $\theta$ and we know that $\boldsymbol{\psi}(1,1,x,\theta)=(x,\theta)$, 
we can conclude $$a_1^\theta=\nabla_{\theta}(F\circ\boldsymbol{\psi}(1,1,\cdot,\cdot))=\nabla_\theta F=0.$$ 
Furthermore, with $x^1\coloneqq\psi(1,x,\theta)$ we have that 
\begin{align}\label{eq:q1x}
a_1^x(x,\theta)=\nabla_x(F\circ \boldsymbol\psi(1,1,\cdot))(x^1)=(\nabla_x F)(x^1).
\end{align}
Hence we can obtain $a_0^\theta= \nabla_{\theta}(F\circ \psi(1,\cdot,\cdot))(x,\theta)$ by solving \eqref{eq:node} with initial conditions
\begin{align}
\psi(1,x,\theta)=(x^1,\theta);\ \ \ \ a_1^x(x,\theta)=(\nabla_x F)(x^1);\ \ \ \ a_1^\theta(x,\theta)=0.
\end{align}
For the special  loss \eqref{losss} the variable $x$ consists of two parts, namely $(x,y)$ and $\psi$ corresponds to $\Psi$ and thus we can evaluate the gradient of the loss function.
\subsection{Numerical Examples of Continuous Normalizing Flows}
By \eqref{losss} and \eqref{eq:a0theta}, training a continuous normalizing flow needs the computation of
\begin{equation}
\nabla_\theta \mathcal{L}(\theta)
=
\E_{x\sim p_X}[a_0^\theta(x)], 
\end{equation} 
where we approximate the expectation by an empirical expectation. In order to compute $a_0^\theta$ we need to solve \eqref{eq:node}. The only ingredients for solving this equation are the vector field $v_t(x,y,\theta)$, the gradients thereof, and the initial conditions. The velocity field $v_t$ and its gradients are available since we parametrize it by a neural network. For the first initial condition we need $\psi(1,x,0,\theta)$ which we can obtain by solving $\partial_tf(t)=v_t(f(t),\theta), f(0)=(x,0)$. For the second initial condition we need $\nabla_{x,y} F$ for $F(x,y)=y-\log p_1(x)$. Thus we need to be given a density where the gradient $\nabla_x\log p_1$ is tractable, which is the case, e.g., for the standard Gaussian distribution. In total we need samples from $p_X$ and a density $p_1$ for which $\nabla_x \log p_1$ is tractable. The computation of $a_0^\theta(x)$ can then be done by solving \eqref{eq:node}, for which we used the library \cite{torchdiffeq}. Figure \ref{fig:cnf_gmm} shows the trajectories for a vector field obtained via continuous normalizing flow training with three different target distributions. For Figure \ref{fig:cnf_gmm8} we chose a GMM with $8$ equally weighted modes. The trajectories differ from the ones obtained by flow matching in Figure \ref{fig:gaussian_mix}. For Figure \ref{fig:cnf_moons} resp. Figure \ref{fig:cnf_spirals} we chose the moons resp. spirals data set from \cite{politorchdyn}.

\begin{figure}[ht]
    \centering
    \begin{subfigure}{0.32\linewidth}
    \includegraphics[width=1\linewidth]{images/cnf_gaussians.png}
    \caption{GMM}\label{fig:cnf_gmm8}
    \end{subfigure}
    \begin{subfigure}{0.32\linewidth}
    \includegraphics[width=1\linewidth]{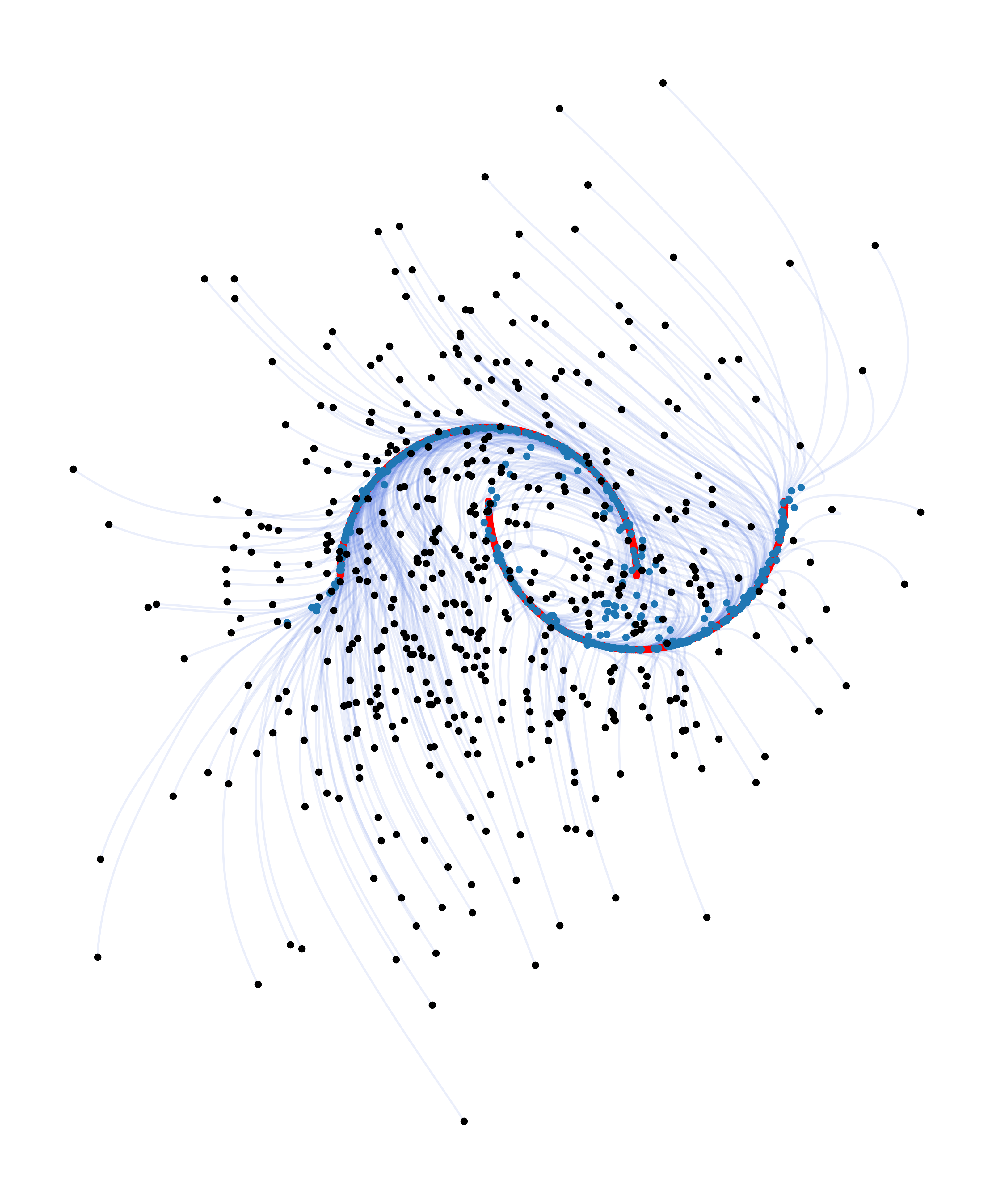}
    \caption{Moons}\label{fig:cnf_moons}
    \end{subfigure}
    \begin{subfigure}{0.32\linewidth}
    \includegraphics[width=1\linewidth]{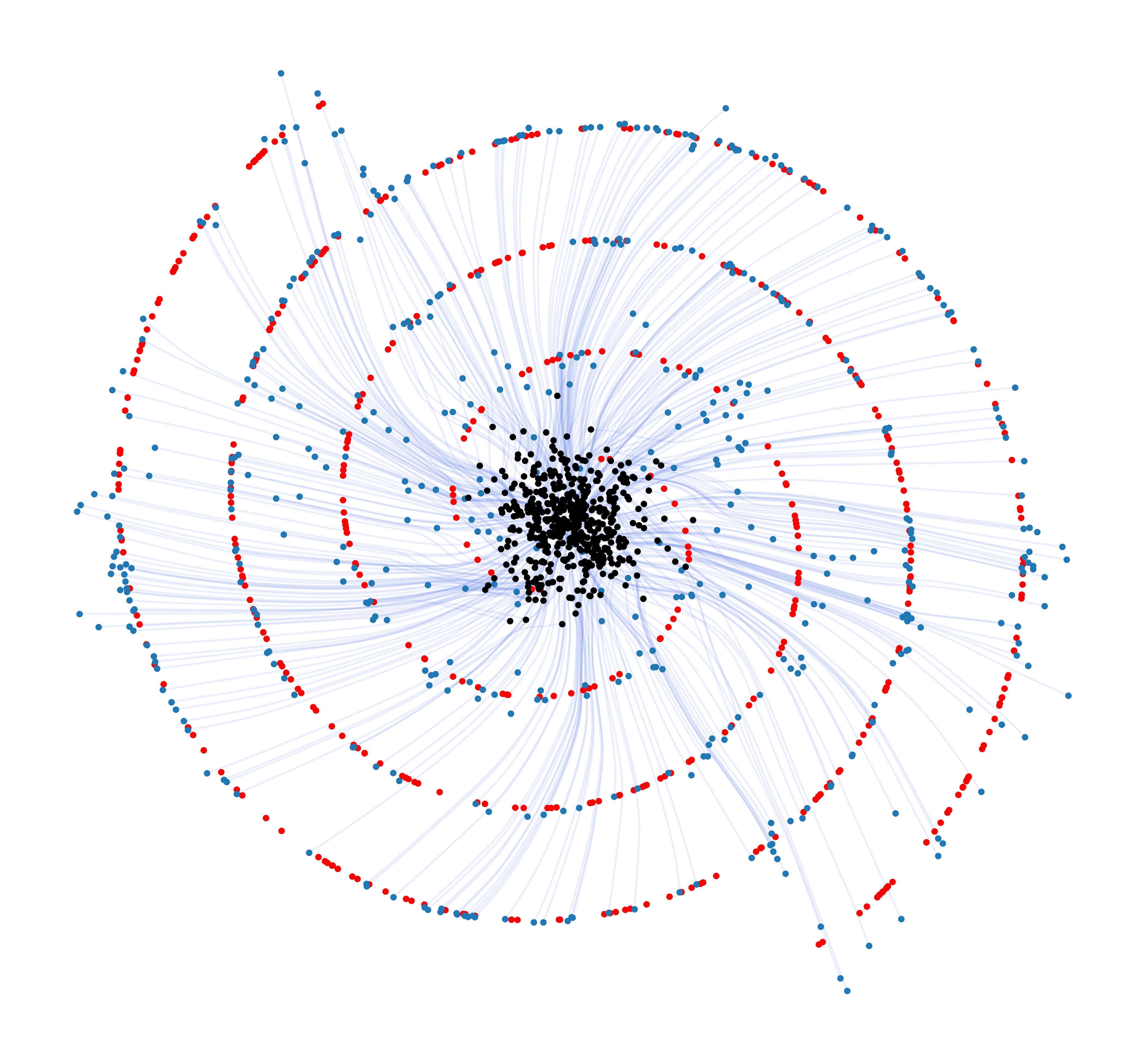}
    \caption{Spirals}\label{fig:cnf_spirals}
    \end{subfigure}    
    \caption{Trajectories of points from a vector field $v_t^\theta$ obtained via a vector field obtained by training a CNF. We chose $\mu_0=\N(0,1)$. Black: points $\{z_i\}_{i=1}^{300}$ drawn from the source distribution $\mu_0=\N(0,1)$. Blue: points sampled via the vector field from $\{z_i\}_{i=1}^{300}$. Red: points sampled from the target distribution. Blue lines: trajectories of the vector field. }
    \label{fig:cnf_gmm}
\end{figure}

\section{A Glimpse at Score-Based Diffusion} \label{sec:diffusion}
Diffusion models go back to  \cite{pmlr-v37-sohl-dickstein15} and \cite{SE2019}.
They can be transformed into CNFs, thus allowing tractable likelihood computation with numerical ODE solvers \cite{song2021maximum}.
However, diffusion models are based on stochastic differential equations (SDEs).
Roughly speaking,
the forward SDE computes $(X_t)_{t \in [0,T]}$ starting with the target distribution by
\begin{equation}  \label{eq}
{\rm d} X_t = f(t,X_t) \, {\rm d} t + g(t) \, {\rm d} W_t, \quad
X_0 \sim P_{{\rm{data}}},
\end{equation}
where $W_t$ is a standard Brownian motion. 
A usual choice is
$$
f(t,X_t) \coloneqq -\frac{1}{2}\beta_t X_t
\quad \text{
and} \quad
g(t) \coloneqq \sqrt{\beta_t}
$$
with a positive, increasing so-called ''time schedule'' 
$\beta$,
e.g., $\beta_t = \beta_{\min} + t(\beta_{\max} - \beta_{\min})$ 
for some 
$0 < \beta_{\min} \ll \beta_{\max}$.
Then the SDE becomes a linear one
$$
{\rm d} X_t = -\frac12 \beta_t X_t\, {\rm d} t + \sqrt{\beta_t} \, {\rm d} W_t, \quad
X_0 \sim P_{{\rm{data}}},
$$
which has the closed form solution  
\begin{equation} \label{closed_form_beta}
    X_t= \sqrt{1-e^{-h(t)}} Z + e^{\frac{-h(t)}{2}} X_0, \quad h(t):= \int_0^t \beta_s \, \dd s,
\end{equation}
where $Z \sim \mathcal{N}(0, I_d)$.
Hence, $X_t$ in \eqref{closed_form_beta} reaches $Z$ for $t \to \infty$.
Then $X_t \sim P_{X_t}$, $t >0$ has a density $p_{X_t}$, and
under certain assumptions, a reverse process becomes
\begin{equation}\label{eq_rev_finite}
{\rm d} Y_{t} = \left(-f(T-t, Y_t) +  g(T-t)^2  \nabla \log p_{X_{T-t}}(Y_t)\right) \, {\rm d} t + g(T-t) \, {\rm d}  W_t, \quad Y_0 \sim X_T.
\end{equation}
However, the reverse SDE depends on the so-called \emph{score} 
$\nabla \log p_{X_t}$.
Score-based models intend to approximate the score by a neural network $s^\theta$ by minimizing, for $T >0$ large enough
$$
\L(\theta)\coloneqq\min_{\theta} \mathbb{E}_{t \sim \mathcal L_{[0,T]}} \mathbb{E}_{x \sim P_{X_t} }\left[ \Vert s_t^\theta(x) - \nabla \log p_{X_t}(x) \Vert^2 \right].
$$
Usually, we do not have access to $\nabla \log p_t$. Fortunately, by \cite[Appendix]{vincent2011connection}, the loss function can be rewritten, up to a constant, as
$$
 \mathcal L (\theta) = \mathbb{E}_{x_0 \sim P_{data}, t \sim \mathcal L_{[0,T]}} \mathbb{E}_{x \sim P_{X_t|X_0 = x_0} }
 \left[ \Vert s^\theta_t(x)
  - \nabla \log p_{X_t|X_0 = x_0}(x) \Vert^2 \right],
$$
which does not involve the score itself, but instead the conditional distribution. 
For \eqref{closed_form_beta}, this is simply a Gaussian
$$P_{X_t|X_0 = x_0} = \mathcal{N} \big(b_t \, x_0, (1- b_t^2) \,  I_d \big)
\quad \text{with} \quad
b_t = \text{e}^{-\frac{h(t)}{2}}
$$
since $X_0$ and $Z$ are independent. This implies
\begin{align}
\nabla \log p_{X_t|X_0=x_0}(x) &=
\nabla \left(-\tfrac{1}{2(1-b_t^2)} \left\|x - b_t x_0 \right\|^2\right)
=-\frac{ x - b_t x_0 }{1-b_t^2}.
\end{align}
Plugging this into the loss function, we get
\begin{align}
\mathcal L (\theta) & = \E_{t \sim \L_{[0,T]}} \mathbb{E}_{x_0 \sim P_{\text{data}}}\mathbb{E}_{x \sim \mathbb{P}_{X_t|X_0 = x_0} }
\Big[ \big\| s^\theta_t(x) + \frac{
x - b_t x_0}{1-b_t^2} \big\|^2 \Big]. 
\end{align}
Once the score is computed, we can use it in the reverse SDE \eqref{eq_rev_finite}, starting with the 
$Z \sim \mathcal N(0,I_d)$ which approximates $X_T$ for $T$ large enough.

For the forward PDE \eqref{eq}, the corresponding densities $p_t= p_{X_t}$ fulfill the \emph{Fokker-Planck equation} 
\begin{align}
\partial_t p_t(x)
&=-\nabla\cdot(f(t,x)p_t(x))+\frac{g(t)^2}{2}\Delta p_t(x)\\
&=-\nabla\cdot \big( \underbrace{\left( f(t,x)-\frac{g(t)^2}{2} \nabla_x \log p_t(x) 
\right)}_{v_t}p_t(x)\big) \label{fp}
\end{align}
which is the continuity equation for $P_{X_t}$. 
This gives the relation to the previous sections. In particular, given the score, we can also use the flow ODE \eqref{eq:flow_ode}
for sampling. 
Moreover, we could compute the log density via Proposition \ref{prop:log_div}. 

For $f(t,x)= -\frac12 \beta_t x$ as above, we obtain $v_t(x) =-\frac12 \beta_t x - g(t)^2\nabla \log p_t(x)$. This is closely related to the vector field associated to the independent coupling which was computed to be $v_t(x)=\frac{x}{t}+\frac{1-t}{t}\nabla\log p_t(x)$ in Proposition \ref{prop:flow_score} and Remark \ref{rem:flow_score}. 

Finally, we like to mention that the above (Fokker-Planck) continuity equation \eqref{fp} arises from a so-called Wasserstein gradient flow
of the KL divergence $KL(p,p_{\text{data}})$, i.e., the velocity field $v_t$ is in  the subdifferential of the above KL function at $p_t$, see e.g.
\cite[Chapter 10.4]{AGS2008}.

\textbf{Acknowledgement.}
Many thanks to R. Beinert and R. Duong for reading the manuscript and  to R. Duong for pointing out Example \ref{ex:v_explode}. C. Wald and G. Steidl gratefully acknowledge funding by the DFG within the SFB “Tomography Across the Scales” (STE 571/19-1, project number: 495365311) and G. Steidl acknowledges funding by the Deutsche Forschungsgemeinschaft under Germany's Excellence Strategy (The Berlin Mathematics Research Center MATH+).

\bibliographystyle{abbrv}
\bibliography{references}

@article{santambrogio2015optimal,
  title={Optimal transport for applied mathematicians},
  author={Santambrogio, Filippo},
  journal={Birkäuser, NY},
  volume={55},
  number={58-63},
  pages={94},
  year={2015},
  publisher={Springer}
}

@inproceedings{albergobuilding,
  title={Building Normalizing Flows with Stochastic Interpolants},
  author={Albergo, Michael Samuel and Vanden-Eijnden, Eric},
  booktitle={The Eleventh International Conference on Learning Representations},
year={2023}
}

@article{HN2021,
	author={P.  Hagemann and S. Neumayer},
	title={Stabilizing Invertible Neural Networks Using Mixture Models},
	journal={Inverse Problems},
volume = {37},
number = {8},
	url={http://iopscience.iop.org/article/10.1088/1361-6420/abe928},
	year={2021}
}

@article{peszek2023heterogeneous,
  title={Heterogeneous gradient flows in the topology of fibered optimal transport},
  author={Peszek, Jan and Poyato, David},
  journal={Calculus of Variations and Partial Differential Equations},
  volume={62},
  number={9},
  pages={258},
  year={2023},
  publisher={Springer}
}

@inproceedings{DSB2017,
  author    = {Laurent Dinh and
               Jascha Sohl{-}Dickstein and
               Samy Bengio},
  title     = {Density estimation using Real {NVP}},
  booktitle = {5th International Conference on Learning Representations, {ICLR} 2017,
               Toulon, France, April 24-26, 2017, Conference Track Proceedings},year={2017}
}

@inproceedings{WKN2020,
  author    = {H. Wu and
               J. K{\"{o}}hler and
               F. No{\'{e}}},
  editor    = {H. Larochelle and
               M. A. Ranzato and
               R. Hadsell and
               M.{-}F. Balcan and
               H.{-}T. Lin},
  title     = {Stochastic Normalizing Flows},
  booktitle = {Advances in Neural Information Processing Systems 2020},
  year      = {2020},
}

@inproceedings{AKRK2019,
  author    = {Lynton Ardizzone and
               Jakob Kruse and
               Carsten Rother and
               Ullrich K{\"{o}}the},
  title     = {Analyzing Inverse Problems with Invertible Neural Networks},
  booktitle = {7th International Conference on Learning Representations, {ICLR} 2019,
               New Orleans, LA, USA, May 6-9, 2019}, 
               year = {2019}}

@misc{sego,
  title = {A Visual Dive into Conditional Flow Matching},
  howpublished = {\url{https://dl.heeere.com/conditional-flow-matching/blog/conditional-flow-matching/}},
author={Quentin Bertrand and Rémi Emonet and Anne Gagneux and Ségolène Martin and Mathurin Massias}
}

@article{goodfellow2014,
author = {I. J. Goodfellow and J. Pouget-Abadie and M. Mirza and B. Xu and D. Warde-Farley and S. Ozair and
A. Courville and Y. Bengio}, 
title = {Generative adversarial nets},
journal = {Advances in Neural
Information Processing Systems}, 
pages = {2672--2680}, 
year       = {2014}
}

@InProceedings{pmlr-v37-sohl-dickstein15,
  title = 	 {Deep Unsupervised Learning using Nonequilibrium Thermodynamics},
  author = 	 {Sohl-Dickstein, Jascha and Weiss, Eric and Maheswaranathan, Niru and Ganguli, Surya},
  booktitle = 	 {Proceedings of the 32nd International Conference on Machine Learning},
  pages = 	 {2256--2265},
  year = 	 {2015},
  editor = 	 {Bach, Francis and Blei, David},
  volume = 	 {37},
  series = 	 {Proceedings of Machine Learning Research},
  address = 	 {Lille, France},
  month = 	 {07--09 Jul},
  publisher =    {PMLR},
  url = 	 {https://proceedings.mlr.press/v37/sohl-dickstein15.html},  
}

@inproceedings{CBDJ2019,
 author = {Chen, R. and Behrmann, J. and Duvenaud, D. K. and Jacobsen, J.-H.},
 booktitle = {Advances in Neural Information Processing Systems},
 pages = {},
 publisher = {Curran Associates, Inc.},
 title = {Residual Flows for Invertible Generative Modeling},
 volume = {32},
 year = {2019}
}

@inproceedings{HZRS2016,
author = {He, K. and Zhang, X. and Ren, S. and Sun, J.}, 
title = {Deep residual learning for image recognition},
booktitle = {Proceedings of the IEEE Conference on Computer Vision and Pattern Recognition},
pages = {770--778}, 
year = {2016}
}

@inproceedings{song2021maximum,
title={Maximum Likelihood Training of Score-Based Diffusion Models},
author={Yang Song and Conor Durkan and Iain Murray and Stefano Ermon},
booktitle={Advances in Neural Information Processing Systems},
editor={A. Beygelzimer and Y. Dauphin and P. Liang and J. Wortman Vaughan},
year={2021},
url={https://openreview.net/forum?id=AklttWFnxS9}
}

@article{KW2013,
author = {D. P. Kingma and M. Welling}, 
title = {Auto-encoding variational Bayes},
journal = {arXiv preprint arXiv:1312.6114}, 
year       = {2013}
}

@article{CRBD2018,
author = {R. Chen and  Y.Rubanova and J. Bettencourt and D.K. 
      Duvenaud},
year = {2018},
title = {Neural ordinary differential equations},
journal = {Advances in Neural
Information Processing Systems}, 
volume = {31}
}

@article{SE2019,
author={Song, Y. and Ermon, St.},
  title={Generative modeling by estimating gradients of the data distribution},
  journal={ArXiv 1907.05600},
  year={2019}
}

@book{PPST2023,
  author     = {G. Plonka and D. Potts and G. Steidl and M. Tasche},
  title      = {Numerical {{Fourier}} Analysis},
edition = {second},
  doi        = {10.1007/978-3-030-04306-3},
  series     = {Applied and Numerical Harmonic Analysis},
  publisher  = {Birkh{\"a}user},
  year       = {2023},
}

@book{AGS2008,
  title={Gradient flows: in metric spaces and in the space of probability measures},
  author={Ambrosio, Luigi and Gigli, Nicola and Savar{\'e}, Giuseppe},
  year={2005},
  publisher={Springer Science \& Business Media}
}

@inproceedings{
lipman2023flow,
title={Flow Matching for Generative Modeling},
author={Yaron Lipman and Ricky T. Q. Chen and Heli Ben-Hamu and Maximilian Nickel and Matthew Le},
booktitle={The Eleventh International Conference on Learning Representations },
year={2023},
url={https://openreview.net/forum?id=PqvMRDCJT9t}
}

@inproceedings{holderrieth2025generator,
title={Generator Matching: Generative modeling with arbitrary Markov processes},
author={Peter Holderrieth and Marton Havasi and Jason Yim and Neta Shaul and Itai Gat and Tommi Jaakkola and Brian Karrer and Ricky T. Q. Chen and Yaron Lipman},
booktitle={ICLR},
year={2025}
}

@article{JCHWS2025,
  title={Trajectory Generator Matching for Time Series},
  author={T Jahn and J Chemseddine and P Hagemann and C Wald and G Steidl},
  journal={arXiv preprint arXiv:2505.23215},
  year={2025}
}

@inproceedings{liu2023flow,
title={Flow Straight and Fast: Learning to Generate and Transfer Data with Rectified Flow},
author={Xingchao Liu and Chengyue Gong and Qiang Liu},
booktitle={The Eleventh International Conference on Learning Representations },
year={2023},
url={https://openreview.net/forum?id=XVjTT1nw5z}
}

@inproceedings{tong2023improving,
  title={Improving and generalizing flow-based generative models with minibatch optimal transport},
  author={Tong, Alexander and Malkin, Nikolay and Huguet, Guillaume and Zhang, Yanlei and Rector-Brooks, Jarrid and Fatras, Kilian and Wolf, Guy and Bengio, Yoshua},
  booktitle={ICML Workshop on New Frontiers in Learning, Control, and Dynamical Systems},
  year={2023}
}

@article{HHS22,
author = {Hagemann, Paul and Hertrich, Johannes and Steidl, Gabriele},
title = {Stochastic Normalizing Flows for Inverse Problems: A {Markov} Chains Viewpoint},
journal = {SIAM/ASA Journal on Uncertainty Quantification},
volume = {10},
number = {3},
pages = {1162-1190},
year = {2022},
doi = {10.1137/21M1450604}
}

@inproceedings{HHG2023,
title = {Generalized Normalizing Flows via {M}arkov Chains},
author = {P. Hagemann and J. Hertrich and G. Steidl},
year = {2022},
editor     = {},
pages = {},
booktitle  = {Non-local Data Interactions: Foundations and Applications},
publisher  = {Cambridge University Press},
}

@article{ 
  gig,
  author = { Gigli, Nicola },
  title = { On the geometry of the space of probability measures endowed with the quadratic {Optimal Transport} distance },
  journal = {PhD Thesis},
  year = { 2008 },
  pages = {  },
  URL = { http://cvgmt.sns.it/paper/491/ },
  note = { cvgmt preprint}
}

@book{ambrosio2021lectures,
  title={Lectures on Optimal Transport},
  author={Ambrosio, L. and Bru{\'e}, E. and Semola, D.},
  isbn={9783030721626},
  series={UNITEXT},
  url={https://books.google.de/books?id=vcI5EAAAQBAJ},
  year={2021},
  publisher={Springer International Publishing}
}

@article{krizhevsky2009learning,
  title={Learning multiple layers of features from tiny images},
  author={Krizhevsky, Alex and Hinton, Geoffrey and others},
  year={2009},
  publisher={Toronto, ON, Canada}
}

@misc{torchdiffeq,
	author={Chen, Ricky T. Q.},
	title={torchdiffeq},
	year={2018},
	url={https://github.com/rtqichen/torchdiffeq},
}

@article{hosseini2024conditional,
      title={Conditional Optimal Transport on Function Spaces}, 
      author={Bamdad Hosseini and Alexander W. Hsu and Amirhossein Taghvaei},
      year={2024},
      journal={arXiv preprint arXiv:2311.05672},
      archivePrefix={arXiv},
      primaryClass={math.OC}
}

@book{bogachev2007measure,
  title={Measure Theory},
  author={Bogachev, Vladimir Igorevich and Ruas, Maria Aparecida Soares},
  volume={1},
  year={2007},
  publisher={Springer}
}

@article{barboni2024understanding,
      title={Understanding the training of infinitely deep and wide ResNets with Conditional Optimal Transport}, 
      author={Raphaël Barboni and Gabriel Peyré and François-Xavier Vialard},
      year={2024},
    journal={arXiv preprint arXiv:2403.12887},
      eprint={2403.12887},
      archivePrefix={arXiv},
      primaryClass={cs.LG}
}

@article{kerrigan2024dynamic,
      title={Dynamic Conditional Optimal Transport through Simulation-Free Flows}, 
      author={Gavin Kerrigan and Giosue Migliorini and Padhraic Smyth},
      year={2024},
      journal={arXiv preprint arXiv:2404.04240},
      archivePrefix={arXiv},
      primaryClass={cs.LG}
}

@article{chemseddine2024conditional,
  title={Conditional {W}asserstein Distances with Applications in {B}ayesian {OT} Flow Matching},
  author={Chemseddine, Jannis and Hagemann, Paul and Wald, Christian and Steidl, Gabriele},
  journal={arXiv preprint arXiv:2403.18705},
  year={2024}
}

@misc{schurovadj,
  title = {Adjoint State Method, Backpropagation and Neural ODEs},
  howpublished = {\url{https://ilya.schurov.com/post/adjoint-method/}},
author={Schurov, Ilya}
}

@article{kloeckner2021extensions,
  title={Extensions with shrinking fibers},
  author={Kloeckner, Benoit R},
  journal={Ergodic Theory and Dynamical Systems},
  volume={41},
  number={6},
  pages={1795--1834},
  year={2021},
  publisher={Cambridge University Press}
}

@article{liu2022rectified,
  title={Rectified flow: A marginal preserving approach to optimal transport},
  author={Liu, Qiang},
  journal={arXiv preprint arXiv:2209.14577},
  year={2022}
}

@book{kallenberg1997foundations,
  title={Foundations of Modern Probability},
  author={Kallenberg, Olav and Kallenberg, Olav},
  volume={2},
  year={1997},
  publisher={Springer}
}

@article{peyre2019computational,
  title={Computational optimal transport: With applications to data science},
  author={Peyr{\'e}, Gabriel and Cuturi, Marco and others},
  journal={Foundations and Trends{\textregistered} in Machine Learning},
  volume={11},
  number={5-6},
  pages={355--607},
  year={2019},
  publisher={Now Publishers, Inc.}
}

@article{gonzalez2024linearization,
  title={Linearization of Monge-Amp$\backslash$ere equations and data science applications},
  author={Gonz{\'a}lez-Sanz, Alberto and Sheng, Shunan},
  journal={arXiv preprint arXiv:2408.06534},
  year={2024}
}

@book{teschl2024ordinary,
  title={Ordinary Differential Equations and Dynamical Systems},
  author={Teschl, Gerald},
  volume={140},
  year={2024},
  publisher={American Mathematical Society}
}

@article{martin2024pnp,
  author = {Martin, Ségolène and Gagneux, Anne and Hagemann, Paul and Steidl, Gabriele},
  title = {{PnP}-Flow: Plug-and-Play Image Restoration with Flow Matching},
  year = {2025},
 journal = {ICLR},
 }

@article{daras2024consistent,
  title={Consistent Diffusion Meets {T}weedie: Training Exact Ambient Diffusion Models with Noisy Data},
  author={Daras, Giannis and Dimakis, Alexandros G and Daskalakis, Constantinos},
  journal={arXiv preprint arXiv:2404.10177},
  year={2024}
}

@misc{politorchdyn,
  title={TorchDyn: Implicit Models and Neural Numerical Methods in PyTorch},
  author={Poli, Michael and Massaroli, Stefano and Yamashita, Atsushi and Asama, Hajime and Park, Jinkyoo and Ermon, Stefano}
}

@article{zhang2024flow,
  title={Flow Priors for Linear Inverse Problems via Iterative Corrupted Trajectory Matching},
  author={Zhang, Yasi and Yu, Peiyu and Zhu, Yaxuan and Chang, Yingshan and Gao, Feng and Wu, Ying Nian and Leong, Oscar},
  journal={arXiv preprint arXiv:2405.18816},
  year={2024}
}

@article{vincent2011connection,
  title={A connection between score matching and denoising autoencoders},
  author={Vincent, Pascal},
  journal={Neural computation},
  volume={23},
  number={7},
  pages={1661--1674},
  year={2011},
  publisher={MIT Press}
}

@article{albergo2023stochastic,
  title={Stochastic interpolants: A unifying framework for flows and diffusions},
  author={Albergo, Michael S and Boffi, Nicholas M and Vanden-Eijnden, Eric},
  journal={arXiv preprint arXiv:2303.08797},
  year={2023}
}

\appendix
\section{Proof of Theorem \ref{m-markov}} \label{app_a}
Recall that a family $\mathcal A$ of subsets of a set $X$ is called \emph{monotone class},
if 
\begin{itemize}
    \item 
$\bigcup_{i=1}^\infty A_n\in \mathcal{A}$ for every increasing sequence $A_i\in \mathcal{A}$, and 
\item $\bigcap_{i=1}^\infty A_i\in\mathcal{A}$ for every decreasing sequence $A_i\in \mathcal A$. 
\end{itemize}
For a family of subsets $B$ of $X$, we call the smallest monotone class containing $B$ the \emph{monotone class generated by} $B$.  
The following theorem can be found, e.g., in \cite[Theorem 1.9.3)]{bogachev2007measure}.

\begin{thm}[Monotone class theorem] 
Let $\mathcal{A}$ be an  algebra of sets, i.e., a collection of sets such that $A \in \mathcal A$ implies $X\setminus A \in \mathcal A$, and
$A,B \in \mathcal A$ implies $A \cup B \in \mathcal A$.
Then the $\sigma$-algebra generated by $\mathcal{A}$ coincides with the monotone class generated by $\mathcal{A}$. In particular, any monotone class containing $\mathcal{A}$ also contains the $\sigma$-algebra generated by $\mathcal A$. 
\end{thm}

\textbf{Theorem \ref{m-markov}}
Let $\mu_t: I\to \mathcal P_2(\R^d)$ be a narrowly continuous curve. Then, for every Borel set $B\subseteq \R^d$, 
we have that $t\mapsto \mu_t(B)$ is measurable, i.e., $\mu_t:I\times\B(\R^d)\to \R$ is a Markov kernel.

\begin{proof}
\underline{Step 1:} Let $\mathcal{A}$ be the set of all measurable sets $B\subseteq \R^d$ such that there exists a series $f_n\in C_b(\R^d), \|f\|_{\infty}\leq 1$ such that $\|f_n-1_B\|_{L^1(\mu_t)}\to 0$ for all $t\in I$. Note that for such sequences with $\|f_n-1_A\|_{L^1(\mu_t)}\to0, \|g_n-1_B\|_{L^1(\mu_t)}\to0$, we have that $\|\min\{f_n,g_n\}-1_{A\cap B}\|_{L^1(\mu_t)}\to 0$ and $\|(1-f_n)-1_{X \setminus A}\|_{L^1(\mu_t)}\to 0$. The former follows from $\min\{a,b\}=\frac{a+b}{2}-\frac{|a-b|}{2}$ and the fact that the $L^1$-limit is linear and interchanges with the absolute value by the reverse triangle inequality. The latter uses $1\in L^1(\mu_t)$, i.e. the finiteness of the measure $\mu_t$. 
Since trivially, $\emptyset \in \mathcal A$, it follows that $\mathcal A$ is an algebra of sets and we checked the first requirement of the monotone class theorem.\\
\underline{Step 2:} We will show that for $A\in \mathcal A$, the map $t\mapsto \mu_t(A)$ is measurable. Let $f_n\in C_b(\R^d), \|f_n\|_\infty\leq1$ be such that $\|f_n-1_A\|_{L^1(\mu_t)}\to 0$ for all $t$. For all $t$ we have by dominated convergence and $\|f_n-1_A\|_{L^1(\mu_t)}\to 0$ that \[t\to \mu_t(A)= \int 1_A \dd \mu_t=\int\lim_{n\to\infty} f_n \dd\mu_t=\lim_{n\to\infty} \int f_n \dd\mu_t.\] 
Consequently, since  $t\to \int f_n \dd\mu_t$ is continuous, we can conclude that $t\mapsto \mu_t(A)$ is measurable as a pointwise limit of continuous functions. \\
Furthermore it is easy to show that any open measurable set is in $\mathcal{A}$ since for an open set $A$ we can approximate $1_A$ pointwise from below by a series of continuous bounded functions. Thus the $\sigma$-algebra generated by $\mathcal{A}$ contains $\B(\R^d)$.\\
\underline{Step 3:} We show that $\mathcal C$, which contains all $C\in \B(\R^d)$ such that $t\mapsto \mu_t(C)$ is measurable, is a monotone class.
Thus we have to show that for an increasing sequence of measurable sets $C_i\in \mathcal C$ and $C=\cup_{i\in \N}C_i$, we have that $C\in \mathcal C$, i.e. $t\mapsto \mu_t(A)$ is measurable and the same for intersections. But this is true since $\mu_t(C)=\lim_{i\to\infty}\mu_t(C_i)$ and pointwise limits of measurable functions are measurable.\\
Finally, since by Step 2 it holds $\mathcal A\subseteq \mathcal C$, we can conclude that $\B(\R^d)\subseteq \mathcal C$ by the monotone class theorem and thus the claim.
\end{proof}

\section{Measurability of $v_t$ in Lemma \ref{prop:vector_dis}}\label{appendix:v_mes}

\begin{prop}\label{rem:v_t_alpha_x_mes}

For $\mu_0,\mu_1\in \P_2(\R^d)$, let $\alpha\in \Gamma(\mu_0,\mu_1)$
and
$\mu_t \coloneqq e_{t,\sharp}\alpha$. Let $\bar{\alpha}=a_\sharp(\mathcal{L}_{(0,1)} \times \alpha)$, where $a(t,x,y)=(t,e_t(x,y),y)$ and denote the disintegration with respect to $\pi^{t,1}(t,x,y)=(t,x)$ by $\bar{\alpha}^{t,x}$. Finally, define 
\begin{align}
    v(t,x)\coloneqq \int_{\R^d}  \frac{y-x}{1-t}\dd\bar{\alpha}^{t,x}(y)
\end{align}
for $(t,x)\in (0,1)\times \R^n$. Then $v$ is the velocity field induced by $\alpha$. For $\alpha_t=(e_t,\pi^2)_\sharp\alpha$, $\alpha_t=\alpha^x_t\times_x\mu_t$ and $v_t(x) = \int_{\R^d} \frac{y-x}{1-t} 
\, \dd \alpha^x_{t}(y)$ we have that
\begin{equation} 
v(t,\cdot) =v_t
\end{equation}
for a.e. $t\in(0,1)$.
\end{prop}
\begin{proof} Note that $\pi^{t,1}_\sharp\bar{\alpha}=(t,e_t)_\sharp\left(\mathcal{L}_{(0,1)} \times \alpha\right)=\mu_t\times_t\mathcal{L}_{(0,1)}$. Hence 
\begin{align}
    \int f(t,x)\int \frac{y-x}{1-t}\dd\bar{\alpha}^{t,x}(y)\dd\left(\mu_t\times_t\mathcal{L}_{(0,1)}\right)&=\int f(t,x)\frac{y-x}{1-t}\dd\bar{\alpha}(t,x,y)\\
    &=\int f(t,e_t(x,y))(y-x)\dd (\alpha\times \mathcal{L}_{(0,1)})\\
    &=\int f(t,x)\dd e_\sharp((y-x)\alpha\times \mathcal{L}_{(0,1)}).
\end{align}
Thus, $v$ is the vector field induced by $\alpha$ since we verified \eqref{eq:v_char}.
As in the proof of Lemma \ref{eq:velo} it follows $v(t,\cdot)=v_t$ for almost all $t\in(0,1)$. More precisely, by carefully checking the definition of disintegration we will show that $\bar{\alpha}^{t,x}=\alpha_t^x$ for a.e. $t\in[0,1]$. Let $\{f_n\}_{n\in\NN}\subset C_b(\R^d\times \R^d)$ be a dense subset. Then for every $h\in C_b([0,1])$ we have that
\begin{align}
\int_0^1 h(t)\int_{\R^d}\int_{\R^d} f_n(x,y)\dd\bar{\alpha}^{t,x}(y)\dd\mu_t(x)\dd t&=\int_0^1 h(t)\int_{\R^d\times\R^d} f_n(e_t(x,y),y)\dd \alpha\dd t\\
    &=\int_0^1 h(t)\int_{\R^d\times\R^d} f_n(x,y)\dd \alpha_t\dd t
\end{align}
and therefore
\begin{align}
\int_{\R^d}\int_{\R^d} f_n(x,y)\dd\bar{\alpha}^{t,x}(y)\dd\mu_t(x)=\int_{\R^d\times\R^d} f_n(x,y)\dd \alpha_t\qquad\text{for a.e. $t\in[0,1]$}.
\end{align}
Thus there exists a zero set $N$ such that 
\begin{align}
\int_{\R^d}\int_{\R^d} f_n(x,y)\dd\bar{\alpha}^{t,x}(y)\dd\mu_t(x)=\int_{\R^d\times\R^d} f_n(x,y)\dd \alpha_t
\end{align}
for all $t\in[0,1]\setminus N$ and all $n\in\NN$.
Using that $\{f_n\}_{n\in\NN}\subset C_b(\R^d\times \R^d)$ is dense we can conclude that $\bar{\alpha}^{t,x}=\alpha_t^x$ as Markov kernels for a.e. $t\in[0,1]$. Hence, $v(t,\cdot)=v_t$ for a.e. $t\in(0,1)$.
\end{proof}

\section{Remark on Normalizing Flows}\label{appc}
Let us have a quick look at normalizing flows
which are invertible neural networks. Other invertible neural
networks are residual networks \cite{CBDJ2019,HZRS2016}, which we will not address in this paper.

A \emph{normalizing flow} between two distributions 
$\mu_i \in \mathcal P(\R^d)$, $i=0,1$
is a
$C^1$ diffeomorphism $T: \R^d \to \R^d$ such that
$$
\mu_1 = T_\sharp \mu_0.
$$
Normalizing flows were approximated by neural networks $T^\theta$
which have a special architecture in order to be invertible
\cite{AKRK2019,DSB2017}.
To this end, a loss function is minimized, e.g. the KL divergence between the push-forward of the latent distribution
$T^\theta_\sharp \mu_0$ and the data distribution $\mu_1$.
Due to the special network architecture, normalizing flows 
are are not suited for in high dimensional problems. Moreover,
they suffer from a limited expressiveness, e.g., 
when trying to map a unimodal (Gaussian) distribution to a multimodal one, their Lipschitz constants explodes \cite{HN2021}.
Stochastic normalizing flows \cite{WKN2020}
circumvent this problem by introducing stochastic layers.
For the definition of stochastic normalizing flows via Markov kernels and an overview, we refer to 
\cite{HHG2023,HHS22}.

In contrast to normalizing flows, recent generative models like flow matching and continuous
normalizing flows as well as score based diffusion models
use, instead of learning a map $T : \R^d \to
\R^d$,
curves in the
space of probability measures, which slowly transform the latent distribution into the
data distribution. In this way, at every time step only a small change of probability measures has to be learned as opposed to learning everything in one step.

\end{document}